\documentclass{article}
\usepackage{etoolbox}       %

\usepackage{graphicx}  %

\usepackage[parfill]{parskip}
\usepackage{amssymb,amsmath,amsthm,bm}
\usepackage{mathtools}  %
\usepackage[dvipsnames,table]{xcolor}         %
\usepackage{thmtools} %
\usepackage{adjustbox}
\usepackage{arydshln}
\usepackage{colortbl}     %
\usepackage{booktabs}     %
\usepackage{multirow}     %

\newtoggle{arxiv}
\toggletrue{arxiv}

  \usepackage[parfill]{parskip}
  \usepackage[citestyle=authoryear-comp,sorting=nyt,maxbibnames=99,backend=biber]{biblatex}
  \newcommand{\citep}{\parencite}
  \newcommand{\citet}{\textcite}
  \newlength{\defbaselineskip}
  \RequirePackage[T1]{fontenc}
  \RequirePackage[tt=false, type1=true]{libertine}
  \RequirePackage[varqu]{zi4}
  \RequirePackage[libertine]{newtxmath}

\usepackage{bm}

\usepackage[inline]{enumitem}
\usepackage{booktabs}
\usepackage{subcaption}
\usepackage{multirow}
\usepackage{wrapfig,lipsum}
\usepackage{algorithm,algorithmicx,algpseudocode}
\usepackage{nicematrix}

\usepackage[colorlinks=true,linkcolor=blue,citecolor=magenta,urlcolor=red]{hyperref}
\usepackage[capitalise,noabbrev]{cleveref}

\usepackage{import}

\usepackage{pifont}%

\usepackage{siunitx}
\usepackage[normalem]{ulem}
\robustify\bfseries
\robustify\uline
\def\Uline#1{#1\llap{\uline{\phantom{#1}}}}
\newrobustcmd\Bf{\DeclareFontSeriesDefault[rm]{bf}{b}\bfseries}

\definecolor{prune}{rgb}{0.44, 0.11, 0.11}
\definecolor{myblue}{rgb}{0, .5, 1}
\definecolor{maroon}{rgb}{0.5450, 0, 0}
\definecolor{darkred}{rgb}{0.5450, 0, 0}
\definecolor{RoyalBlue}{RGB}{0,100,170}
\definecolor{DarkBlue}{RGB}{20,70,200}
\definecolor{peach}{rgb}{1, 0.56, 0.56}
\definecolor{NotionGreen}{RGB}{15,123,108}
\definecolor{NotionOrange}{RGB}{217,115,13}
\definecolor{NotionRed}{RGB}{224,62,62}
\definecolor{red}{RGB}{224,62,62}
\definecolor{midgray}{RGB}{150,150,150}
\definecolor{lavender}{rgb}{0.75, 0.58, 0.89}
\definecolor{Indigo7}{RGB}{66, 99, 235}
\definecolor{Green7}{RGB}{55, 178, 77}
\definecolor{Yellow7}{RGB}{245, 159, 0}
\definecolor{Red7}{RGB}{240, 62, 62}

\definecolor{TorchCol}{gray}{0.35}
\definecolor{TritonCol}{HTML}{1F77B4}
\definecolor{CuTeCol}{HTML}{FF7F0E}
\definecolor{TileLangCol}{HTML}{FF69B4}

\newcommand{\torch}[1]{\textcolor{TorchCol}{#1}}
\newcommand{\triton}[1]{\textcolor{TritonCol}{#1}}
\newcommand{\cute}[1]{\textcolor{CuTeCol}{#1}}
\newcommand{\tilelang}[1]{\textcolor{TileLangCol}{#1}}

\def\eqref#1{equation~(\ref{#1})}
\def\Eqref#1{Equation~(\ref{#1})}
\def\plaineqref#1{(\ref{#1})}

\def\1{\bm{1}}

\def\vg{{\bm{g}}}
\def\vh{{\bm{h}}}

\def\vv{{\bm{v}}}

\def\vx{{\bm{x}}}
\def\vy{{\bm{y}}}
\def\vz{{\bm{z}}}

\def\mB{{\bm{B}}}

\DeclareMathAlphabet{\mathsfit}{\encodingdefault}{\sfdefault}{m}{sl}
\SetMathAlphabet{\mathsfit}{bold}{\encodingdefault}{\sfdefault}{bx}{n}

\def\bA{{\bm{A}}}
\def\bB{{\bm{B}}}
\def\bC{{\bm{C}}}

\def\bH{{\bm{H}}}

\def\bK{{\bm{K}}}
\def\bL{{\bm{L}}}

\def\bO{{\bm{O}}}

\def\bQ{{\bm{Q}}}
\def\bR{{\bm{R}}}

\def\bU{{\bm{U}}}
\def\bV{{\bm{V}}}
\def\bW{{\bm{W}}}
\def\bX{{\bm{X}}}
\def\bY{{\bm{Y}}}
\def\bZ{{\bm{Z}}}

\theoremstyle{plain}

\theoremstyle{definition}

\theoremstyle{remark}
\newtheorem{remark}{Remark}

\newcommand{\albert}[1]{\textcolor{blue}{}}

\newtoggle{comments}
\toggletrue{comments}
\iftoggle{comments}{
    \newcommand\todo[1]{\textcolor{red}{[TODO: #1]}}
    \newcommand\AG[1]{\textcolor{magenta}{[AG: #1]}}
    \newcommand\AGr[1]{}
    \newcommand\TD[1]{\textcolor{prune}{[TD: #1]}} 
    \newcommand\KYL[1]{\textcolor{orange}{[KL: #1]}}
    \newcommand{\AvB}[1]{\textcolor{cyan}{[AB: #1]}}
    \newcommand{\BC}[1]{\textcolor{peach}{[BC: #1]}}
}{
    \newcommand\todo[1]{}
    \newcommand\AG[1]{}
    \newcommand\TD[1]{}
    \newcommand\AGr[1]{}
    \newcommand\KYL[1]{}
    \newcommand{\AvB}[1]{}
    \newcommand{\BC}[1]{}
}

\NewDocumentCommand\CITE{o}{%
  \IfNoValueTF{#1}
    {\textcolor{red}{[CITE]}}
    {\textcolor{red}{[#1]}}%
}

\newcommand*\samethanks[1][\value{footnote}]{\footnotemark[#1]}

\title{Mamba-3: Improved Sequence Modeling using State Space Principles}
\usepackage{authblk}
\author[$^{1}$]{Aakash Lahoti\,\thanks{Equal contribution.}}
\author[$^{1}$]{Kevin Y. Li\,\samethanks[1]}
\author[$^{2}$]{Berlin Chen\,\samethanks[1]}
\author[$^{2}$]{Caitlin Wang\,\samethanks[1]}
\author[$^{1}$]{Aviv Bick}
\author[$^{1}$]{J. Zico Kolter}
\author[$^{23}$]{\authorcr Tri Dao\thanks{Equal advising.}}
\author[$^{14}$]{Albert Gu\samethanks[2]}

\affil[ ]{\normalsize $^1$Carnegie Mellon University \quad $^2$Princeton University \quad $^3$Together AI \quad $^4$Cartesia AI}
\affil[ ]{{\texttt{\{alahoti, kyl2, abick, zkolter, agu\}@cs.cmu.edu} \quad \texttt{\{bc2188, caitlinwang, tridao\}@princeton.edu}}}

\date{}

\begin{document}
\maketitle

\begin{abstract}
\noindent 
Scaling inference-time compute has emerged as an important driver of LLM performance, making inference efficiency a central focus of model design alongside model quality.
While the current Transformer-based models deliver strong model quality, their quadratic compute and linear memory make inference expensive.
This has spurred the development of sub-quadratic models with reduced linear compute and constant memory requirements.
However, many recent linear models trade off model quality and capability for algorithmic efficiency, failing on tasks such as state tracking. Moreover, their theoretically linear inference remains hardware-inefficient in practice.
Guided by an inference-first perspective, we introduce three core methodological improvements inspired by the state space model (SSM) viewpoint of linear models.
We combine: (1) a more expressive recurrence derived from SSM discretization, (2) a complex-valued state update rule that enables richer state tracking, and (3) a multi-input, multi-output (MIMO) formulation for better model performance without increasing decode latency. 
Together with architectural refinements, our \textbf{Mamba-3} model achieves significant gains across retrieval, state-tracking, and downstream language modeling tasks.
At the 1.5B scale, Mamba-3 improves average downstream accuracy by 0.6 percentage points compared to the next best model (Gated DeltaNet), with Mamba-3's MIMO variant further improving accuracy by another 1.2 points for a total 1.8 point gain.
Across state-size experiments, Mamba-3 achieves comparable perplexity to Mamba-2 despite using half of its predecessor's state size.
Our evaluations demonstrate Mamba-3's ability to advance the performance-efficiency Pareto frontier.
\end{abstract}

\section{Introduction}
Test-time compute has emerged as a key driver of progress in LLMs, with techniques like chain-of-thought reasoning and iterative refinement demonstrating that inference-time scaling can unlock new capabilities~\citep{wu2025inferencescalinglawsempirical,snell2024scalingllmtesttimecompute}.
The rapid rise of parallel, agentic workflows has only intensified the need for efficient inference and deployment of such models~\citep{openai_gpt53_codex_2026,anthropic_claude_opus_46_2026}.
This paradigm shift makes inference efficiency~\citep{kwon2023efficientmemorymanagementlarge,li2024llminferenceservingsurvey} paramount, as the practical impact of AI systems now depends critically on their ability to perform large-scale inference during deployment.
Model architecture design plays a fundamental role in determining inference efficiency, as architectural choices directly dictate the computational and memory requirements during generation.
While Transformer-based models~\citep{vaswani2017attention} are the current industry standard, they are fundamentally bottlenecked by linearly increasing memory demands through the KV cache and quadratically increasing compute requirements through the self-attention mechanism.
These drawbacks have motivated recent lines of work on sub-quadratic models, e.g., state space models (SSMs) and linear attention, which retain constant memory and linear compute while attaining comparable or better performance than their Transformer counterparts.
These models have made it into the mainstream, with layers such as Mamba-2~\citep{dao2024transformersssmsgeneralizedmodels} and Gated DeltaNet (GDN)~\citep{schlag2021deltarule,yang2025parallelizinglineartransformersdelta} recently incorporated into large-scale hybrid models that match the performance of pure Transformer alternatives with much higher efficiency~\citep{nvidia2025nemotronhfamilyaccurateefficient,qwen3technicalreport,kimiteam2025kimilinearexpressiveefficient,tencenthunyuanteam2025hunyuanturbosadvancinglargelanguage}.

Despite the success of linear models, significant progress remains in improving their performance, in particular on advancing the Pareto frontier between model quality and inference efficiency.
For example, Mamba-2 was developed to improve training speed and simplicity over Mamba-1~\citep{gu2024mambalineartimesequencemodeling}, by sacrificing some expressivity and thus performing worse for inference-matched models.
In addition, they have been shown to lack certain capabilities, such as poor state-tracking abilities, e.g., simply determining parity of bit sequences~\citep{grazzi2025unlockingstatetrackinglinearrnns,sarrof2024expressivecapacitystatespace}.
Finally, despite these sub-quadratic models being prized for theoretically efficient inference and thus their widespread adoption, their inference algorithms are not hardware efficient.
In particular, because these algorithms were developed from a training perspective, their decoding phase has low arithmetic intensity (the ratio of FLOPs to memory traffic), resulting in large portions of hardware remaining idle. 

To develop more performant models from an inference-first paradigm, we introduce three core methodological changes on top of Mamba-2, influenced by an SSM-centric viewpoint of sub-quadratic models.

\paragraph{Exponential-Trapezoidal Discretization.} 
We provide a simple technique for discretizing time-varying, selective SSMs.
Through our framework, we can derive several new discretization methods.
One of our instantiations, referred to as ``exponential-Euler,'' formalizes Mamba-1 and Mamba-2's heuristic discretization that previously lacked theoretical justification.
Our new ``exponential-trapezoidal'' instantiation is a more expressive generalization of ``exponential-Euler,'' where the recurrence can be expanded to reveal an implicit convolution applied on the SSM input. 
Combined with explicit $B,C$ bias terms,
Mamba-3 can empirically replace the short causal convolution in language model architectures, which was previously hypothesized to be essential for recurrent models.

\paragraph{Complex-valued State Space Model.} By viewing the underlying SSM of Mamba-3 as complex-valued, we enable a more expressive state update than Mamba-2's. This change in update rule, designed to be lightweight for training and inference, overcomes the lack of state-tracking ability common in many current linear models. We show that our complex-valued update rule is equivalent to a data-dependent rotary embedding and can be efficiently computed~\citep{su2023roformerenhancedtransformerrotary}, and empirically demonstrate its ability to solve synthetic tasks outside the capabilities of prior linear models.

\paragraph{Multi-Input, Multi-Output (MIMO) SSM.} To improve FLOP efficiency during decoding, we switch from an outer-product–based state update to a matrix-multiplication–based state update.
From the view of the signal processing foundations of SSMs, such a transition exactly coincides with the generalization from a single-input single-output (SISO) sequence dynamics to a multiple-input multiple-output (MIMO) one. Here, we find that MIMO is particularly suitable for inference, as the extra expressivity enables more computation during the memory-bound state update during decoding, without increasing the state size and compromising speed.

Put together, these improvements form the core of our \textbf{Mamba-3} layer.
Methodologically, we note that these all arise naturally from an SSM-centric perspective but are not immediate from other popular viewpoints of modern linear layers such as linear attention or test-time regression;
we discuss these connections further in~\Cref{sec:related}.
Empirically, we validate our new model's abilities and capabilities on a suite of synthetic state-tracking and language-modeling tasks.

\begin{itemize}[itemsep=0.1pt,topsep=0pt,leftmargin=*]
    \item \textbf{Better Quality.} 
    At 1.5B scale, Mamba-3 (MIMO) improves downstream language modeling accuracy by \textbf{+2.2} over Transformers, \textbf{+1.9 points} over Mamba-2, and \textbf{+1.8} over GDN, while Mamba-3 (SISO) improves over the next best model, GDN, by \textbf{+0.6} points.
    Furthermore, across state size experiments, Mamba-3 (MIMO) with state size 64 matches the perplexity of Mamba-2 with state size 128, effectively achieving the \textbf{same language modeling performance with half the latency}.
    \item \textbf{New Capabilities.} Mamba-3's complexification of the SSM state enables it to \textbf{solve synthetic state-tracking tasks that Mamba-2 cannot}. We empirically demonstrate that the efficient RoPE-like calculation is able to near perfectly solve arithmetic tasks, while Mamba-3 without RoPE and Mamba-2 perform no better than random guessing.
    \item \textbf{Inference Efficiency.} Mamba-3 (MIMO) improves hardware utilization. It increases decoding FLOPs by up to \textbf{4$\times$} relative to Mamba-2 at fixed state size, while maintaining \textbf{similar wall-clock decode latency}, and simultaneously improving perplexity and downstream performance. We release fast training and inference kernels for Mamba-3.\footnote{\url{https://github.com/state-spaces/mamba}.}
\end{itemize}

Mamba-3 (SISO) improves quality and capability over prior linear models, and Mamba-3 (MIMO) further improves performance over Mamba-3 (SISO) and other strong baselines while matching inference speed with Mamba-2.
Both of our Mamba-3 variants advance the performance-latency Pareto frontier through their strong modeling capabilities and hardware-efficient design.

\section{Preliminaries}
\subsection{Notation}
Scalars are denoted by plain-text letters (e.g., $x, y$). 
Tensors, including vectors and matrices, are denoted by bold letters (e.g., $\vh, \bC$). 
The shape of the tensor can be inferred from the context.
We denote the input sequence length as $T$, the model dimension as $D$, and the SSM state size as $N$. %
For time indices, we use subscripts (e.g., $x_t$ for the input at time $t$).
The Hadamard product between two tensors is denoted by $\odot$.
For a vector $\vv \in \mathbb{R}^d$, we denote $\mathrm{Diag}(\vv) \in \mathbb{R}^{d \times d}$ as the diagonal matrix with the vector $\vv$ as the diagonal, and for products of scalars across time steps, we use the notation $\alpha_{t \cdots s} = \alpha^\times_{t:s} = \prod_{i=s}^{t} \alpha_i$.

\subsection{SSM Preliminaries}

State Space Models (SSMs) describe continuous-time linear dynamics via
\begin{align*}
\dot{\vh}(t) &= \bA(t)\,\vh(t) + \bB(t)\,x(t), &
y(t) &= \bC(t)^\top \vh(t),
\end{align*}
where $\vh(t)\!\in\!\mathbb{R}^N$ is the hidden state, $x(t)\!\in\!\mathbb{R}$ the input, and $\bA(t)\!\in\!\mathbb{R}^{N\times N}$, $\bB(t),\bC(t)\!\in\!\mathbb{R}^N$.
We will occasionally refer to $\bA(t)$ as the \emph{state-transition} and $\bB(t)x(t)$ as the \emph{state-input}; this also extends to their discretized counterparts. 
For discrete sequences with step size $\Delta_t$,
Mamba-1 and Mamba-2 \emph{discretized} the system to the following recurrence
\begin{align*}
\vh_t &= e^{\Delta_t \bA_t}\,\vh_{t-1} + \Delta_t\,\bB_t\,x_t, &
y_t &= \bC_t^\top \vh_t .
\end{align*}

\paragraph{Mamba-2's Parameterization.}
The core of the Mamba-2 layer~\citep{dao2024transformersssmsgeneralizedmodels} is a \emph{data-dependent} and hardware-efficient SSM.
Both the state-transition and state-input are made data-dependent through the projection of $\Delta_t \in \mathbb{R}_{>0}$ and $\bB, \bC\in\mathbb{R}^N$ from the current token.
By parameterizing the state-transition $\bA_t$ as a scalar times identity ($\bA_t = A_t \bm{I}_{N\times N}$, where 
$A_t \in \mathbb{R}_{<0}$), the SSM recurrence can be efficiently computed with the matrix multiplication tensor cores of GPUs.
Defining $\alpha_t \coloneqq e^{\Delta_t A_t}\in(0,1)$ and $\gamma_t \coloneqq \Delta_t$, the update becomes
\begin{equation}
\label{eq:mamba2-recurrence-prelim}
\vh_t = \alpha_t\,\vh_{t-1} + \gamma_t\,\bB_t\,x_t, \qquad
y_t = \bC_t^\top \vh_t .
\end{equation}

The data-dependent state-transition $\alpha_t$ controls the memory horizon of each SSM within the layer.
$\Delta_t$ in particular modulates both the state-transition and state-input: a larger $\Delta_t$ forgets faster and up-weights the current token more strongly, while a smaller $\Delta_t$ retains the hidden state with minimal contributions from the current token.

\begin{remark}
In Mamba-2, $A_t$ is data-independent, since the overall discrete transition $\alpha_t \coloneqq e^{\Delta_t A_t}$ is data-dependent through $\Delta_t$.
In Mamba-3, we empirically found that data-dependent $A_t$ has similar performance to data-independent $A_t$, and chose the former as a default for consistency so that all SSM parameters are data-dependent.
\end{remark}

\subsection{Structured Masked Representation and State Space Duality}
\label{sec:prelim-sma}
Mamba-2 showed that a large class of SSMs admit a \emph{matrix} form that vectorizes the time-step recurrence.
Through the state space duality (SSD) framework, recurrent SSMs can be represented within a parallel form that incorporates an element-wise mask to model the state-transition decay.

SSD provides a general framework for a duality between linear recurrence and parallelizable (matrix-multiplication-based) computational forms
\begin{align}
    \bY 
    =
    (   \bL
        \odot 
        \bC {\bB}^{\top}
    )\bX
    \label{eq:ssd}
\end{align}
where $\bL \in \mathbb{R}^{T \times T}$ is a structured mask,
$\bB,\bC \in \mathbb{R}^{T \times N}$, $\bX \in \mathbb{R}^{T \times D}$ are the inputs to the SSM and $\bY \in \mathbb{R}^{T \times D}$ is its output.
Different structures on $\bL$ give rise to various instantiations of SSD.

\Eqref{eq:ssd} also draws a general connection between recurrence and attention, 
by setting $\bQ\coloneqq\bC$, $\bK\coloneqq\bB$, $\bV\coloneqq\bX$ and viewing $\bL$ as a data-dependent mask.
In fact, the simplest case of SSD is (causal) linear attention~\citep{linearattention}, where $\bL$ is the causal triangular mask.

Mamba-2 is a generalization where
\begin{align}
    \bL
    =
        \begin{bmatrix}
            1 \\
            \alpha_1 & 1 \\
            \vdots & & \ddots \\
            \alpha_{T...1} & \cdots & \alpha_T & 1
        \end{bmatrix}
        \cdot
        \operatorname*{Diag}(\gamma)
    \label{eq:mamba2-mask}
\end{align}
composed of terms $\alpha_t, \gamma_t$ from~\eqref{eq:mamba2-recurrence-prelim}.%
\footnote{In the original Mamba-2 paper, $\gamma$ does not appear because it is viewed as folded into the $\bB$ term. In this paper, $\bB_t$ represents the continuous parameter, whereas in Mamba-2, $\bB_t$ represents the discretized parameter which is equivalent to $\gamma_t \bB_t$.}

In \Cref{sec:method:trap:ssd}, we show that Mamba-3 is a generalization of Mamba-2 with a more expressive $\bL$, and hence also an instance of SSD.

\section{Methodology}

We introduce Mamba-3, a state space model with three new innovations:
``exponential-trapezoidal'' discretization for more expressive dynamics (\cref{sec:method:trap}),
complex-valued state spaces for state tracking (\cref{sec:method:complex}),
and multi-input multi-output (MIMO) to improve modeling power and inference-time hardware utilization (\cref{sec:method:mimo}).
These advances address the quality, capability, and efficiency limitations of current sub-quadratic architectures.
We combine these together into an updated Mamba architecture block in \cref{sec:method:arch}.

\subsection{Exponential-Trapezoidal Discretization}
\label{sec:method:trap}
Structured SSMs are naturally defined as continuous-time dynamical systems that map input functions, $x(t) \in \mathbb{R}$,  to output functions, $y(t) \in \mathbb{R}$, for time $t > 0$.
The underlying continuous state space system is defined by a first-order ordinary differential equation (ODE) for the state $\dot{\vh}(t)$ and an algebraic equation for the output $y(t)$.
In sequence modeling, however, the data is only observed at discrete time steps, which then requires applying a \emph{discretization step} to the SSM to transform its continuous-time dynamics into a discrete recurrence.
\begin{figure}
    \vspace{-1cm}
    \includegraphics[width=\linewidth]{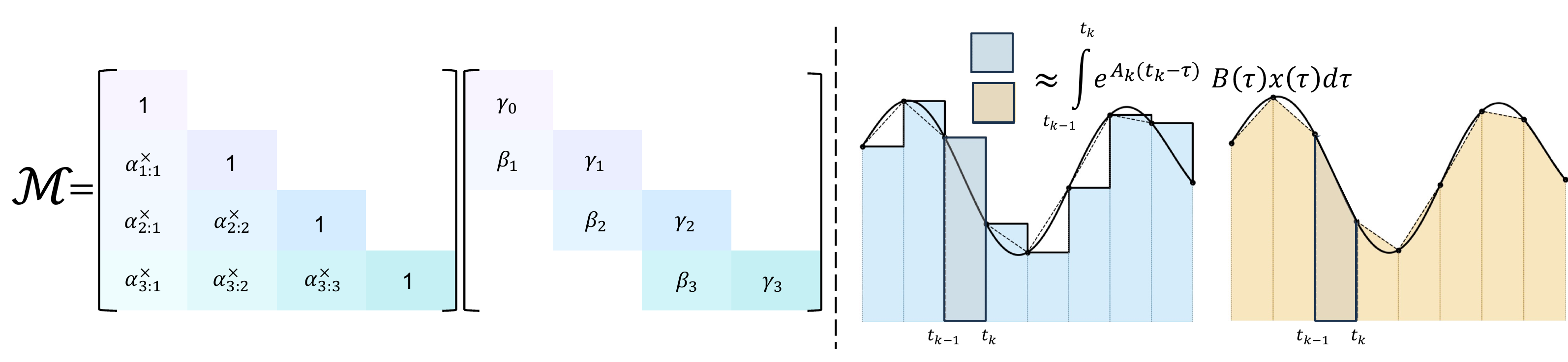}
    \caption{\textbf{Left:} The structured mask induced by the exponential-trapezoidal rule (\Cref{sec:method:trap}) is a product of the decay and two-band convolutional mask. \textbf{Right:} Euler (hold endpoint) versus Trapezoidal (average endpoints) integral approximation.} 
    \label{fig:trap}
\end{figure}

\begin{table*}[t]
  \centering
  \small
  \renewcommand{\arraystretch}{1.2}
  \newcommand{\minitab}[2][l]{\renewcommand{\arraystretch}{1.0}\begin{tabular}[#1]{@{}#1@{}}#2\end{tabular}}
  \caption{
  Table of canonical linear-time invariant discretizations (top) and custom linear-time varying discretizations derived from our exponential-adjusted framework (bottom), along with their appearance in structured SSMs used in deep learning.
  Our theory formalizes the prior Mamba discretization as exponential-Euler and extends it with the more expressive exponential-trapezoidal method. 
  The generalized discretization framework converts a continuous SSM $\dot{\vh(t)} = \bA(t) \vh(t) + \bB(t) x(t)$ into the discrete recurrence $\vh_t = \alpha_t \vh_{t-1} + \beta_t \bB_{t-1} x_{t-1} + \gamma_t \bB_t x_t$, where various discretization methods yield different formulas for $\alpha_t, \beta_t, \gamma_t$.
  }
  \resizebox{\textwidth}{!}{
    \begin{tabular}{@{}lllll@{}}
        \toprule
        \textbf{Discretization Method} &
        $\alpha_t$ &
        $\beta_t$ &
        $\gamma_t$ &
        \textbf{Appearance} \\
        \midrule
        \textbf{Forward Euler} &
        $I + \Delta A$ &
        --- &
        $\Delta$ &
        --- \\
        \textbf{Backward Euler} &
        $(I-\Delta A)^{-1}$ &
        --- &
        $(I-\Delta A)^{-1}\,\Delta$ &
        --- \\
        \textbf{Trapezoidal} &
        $(I-\frac{\Delta}{2}A)^{-1}\,(I+\frac{\Delta}{2}A)$ &
        --- &
        $(I-\frac{\Delta}{2}A)^{-1}\,\Delta$ &
        S4 \\
        \textbf{Zero-Order Hold} &
        $\exp({\Delta A})$ &
        --- &
        $A^{-1}\,\left(\exp({\Delta A})-I\right)$ &
        \minitab{S4D, S5} \\
        \addlinespace[0.6ex]
        \midrule
        \textbf{Zero-Order Hold} &
        $\exp({\Delta_tA_t})$ &
        --- &
        $A_t^{-1}\,\left(\exp({\Delta_tA_t})-I\right)$ & 
        \\
        \textbf{Exponential-Euler} &
        $\exp({\Delta_tA_t})$ &
        --- &
        $\Delta_t$ &
        Mamba-1, -2\footnotemark \\
        \textbf{Exponential-Trapezoidal} &
        $\exp({\Delta_tA_t})$ &
        $(1-\lambda_t)\Delta_t\,\exp({\Delta_tA_t})$ &
        $\lambda_t\,\Delta_t$ &
        Mamba-3 \\
        \bottomrule
      \end{tabular}
    }
  \label{tab:discretizations}
\end{table*}
\footnotetext{While the Mamba-1 paper reports ZOH discretization, the implementation follows \url{https://github.com/state-spaces/mamba/issues/129}.}

Discretization methods are well-studied in classical control theory with several canonical formulas used in earlier SSM works in deep learning~\citep{S4,S4D,S5}.
These mechanisms were traditionally stated and applied to linear-time invariant (LTI) systems, and their derivations do not directly apply to linear-time varying (LTV) systems.
Additionally, while Mamba-1 adapted the zero-order hold (ZOH) method to LTV systems without proof, the complexity associated with selective SSMs prompted the use of an additional heuristic approximation that lacked theoretical justification and did not correspond to any established discretization technique.
In the following subsection, we formalize the previous heuristics used in current LTV SSMs through our discretization framework and utilize it to propose a more expressive discretization scheme.
\subsubsection{Overview of Exponential-Adjusted Discretization}
We introduce a simple derivation that leads to a class of new discretization methods for LTV state space models. The method can be instantiated in various ways; we show that one instantiation results in the heuristic used in Mamba-1/2, thereby theoretically justifying it (exponential-Euler). We also introduce a more powerful discretization (exponential-trapezoidal) used in Mamba-3.

The high-level intuition of our derivation originates from the closed form solution $x(t) = e^{tA}x(0)$ of a simple linear ODE $x'(t) = Ax(t)$, which discretizes to $x_{t+1} = e^{\Delta A}x_t$.
In this example, the exponential dominates the dynamics of the underlying first-order ODE, resulting in imprecise approximations when using low-order methods without significantly constraining $\Delta$.
Thus, we analyze the dynamics of the \emph{exponential-adjusted} system $e^{-At}x(t)$.
The adjusted system yields a discrete recurrent form where the state-transition and the state-input integrals are approximated separately---the state-transition integral is approximated by a right-hand approximation, i.e. $A(s) \coloneqq A(\tau_t)$ for all $s \in [\tau_{t-1},\tau_t]$, yielding,
\begin{align*}
    \vh(\tau_t) &= \underbrace{\exp\left(\int_{\tau_{t-1}}^{\tau_t}A(s)ds\right)\vh(\tau_{t-1})}_\text{via right-hand approximation} + \underbrace{\int_{\tau_{t-1}}^{\tau_t} \exp\left(\int_{\tau}^{\tau_t}A(s)ds\right)\bB(\tau)x(\tau) d\tau}_\text{via different discretization schemes},
    \\
    \vh_t &\approx \exp\left(\Delta_t A_t\right)\vh_{t-1} + \int_{\tau_{t-1}}^{\tau_t} \exp\left((\tau_t -\tau)A_t\right)\bB(\tau)x(\tau) d\tau,
\end{align*}
which serves as the foundation for further discretization techniques for the state-input integral. The full derivation is detailed in \Cref{prop:var-const}.

\paragraph{ZOH.}
The classical zero-order hold discretization method can be derived from the foundation above with a specific approximation of the right-hand side integral.
By treating $A_t, \bB(\tau), x(\tau)$ as constants over the interval $[\tau_{t-1}, \tau_t]$ where the values are fixed to the right endpoint $\tau_t$, the integral results in $A_t^{-1}\,\left(\exp(\Delta_tA_t) - I\right)\bB_tx_t$.

We note that this formally proves that the classical ZOH formula for LTI systems applies to LTV by naively replacing the parameters $A, B, \Delta$ with their time-varying ones.

\paragraph{Exponential-Euler (Mamba-1/-2).}
While Mamba-1 stated to use the time-varying ZOH formula above, Mamba-1 and Mamba-2 actually used an additional approximation in the released implementation.
This discretization can be recovered by approximating the state-input integral with \emph{Euler's rule} \citep{Süli_Mayers_2003} and holding the (right) endpoint constant throughout the interval (Fig. \ref{fig:trap})
\begin{align}
    \vh_t \;&\approx\; e^{\Delta_t A_t}\,\vh_{t-1} 
        \;+\; (\tau_t - \tau_{t-1}) e^{(\tau_t-\tau_t) A_t}\,\bB_t\,x_t \nonumber \\
    &=\; e^{\Delta_t A_t}\,\vh_{t-1} 
        \;+\; \Delta_t\,\bB_t\,x_t.
    \label{eq:mamba2-recurrence}
\end{align}

We call \eqref{eq:mamba2-recurrence}
the \emph{exponential-Euler} discretization method,
stemming from the exponential integration followed by Euler approximation.
This derivation justifies the formulas used in Mamba-1/-2's implementation.

\paragraph{Exponential-Trapezoidal (Mamba-3).}

However, Euler's rule provides only a first-order approximation of the state-input integral and its local truncation error scales as $O(\Delta_t^2)$.
In contrast, we introduce a \emph{generalized trapezoidal rule}, which provides a second-order accurate approximation of the integral, offering improved accuracy over Euler's rule.
Specifically, it approximates the integral with a \emph{data-dependent, convex combination of both interval endpoints}.
This generalization extends the classical trapezoidal rule \citep{Süli_Mayers_2003}, which simply averages the interval endpoints (\Cref{fig:trap}).

\begin{restatable}[Exponential-Trapezoidal Discretization]{proposition}{PropTrap}\label{prop:trap}
Approximating the state-input integral in~\eqref{eq:approx-step} by the general trapezoidal rule yields the recurrence,
\begin{align}
    \vh_t 
    \;&=\; 
    e^{\Delta_t A_t} \vh_{t-1} 
    \;+\; (1-\lambda_t)\Delta_t e^{\Delta_t A_t} \bB_{t-1} x_{t-1} 
    \;+\; \lambda_t \Delta_t \bB_t x_t,
    \\
    &\eqqcolon\; 
    \alpha_t \vh_{t-1} 
    \;+\; \beta_t \bB_{t-1} x_{t-1} 
    \;+\; \gamma_t \bB_t x_t,
    \label{eq:mamba3-recurrence}
\end{align}
where $\lambda_t \in [0,1]$ is a data-dependent scalar, $\alpha_t \coloneqq e^{\Delta_t A_t}$, $\beta_t \coloneqq (1-\lambda_t)\Delta_t e^{\Delta_t A_t}$, $\gamma_t \coloneqq \lambda_t \Delta_t$.

\end{restatable}
\begin{remark}[Expressivity]
    The exponential-trapezoidal rule is a generalization of (a) the classical trapezoid rule, which is recovered when $\lambda_t = \tfrac{1}{2}$, and (b) Mamba-2's Euler's rule, which is recovered when $\lambda_t = 1$.
\end{remark}
\begin{remark} [Error Rate]
    This is a second-order discretization of the state-input integral and its error scales as $O(\Delta_t^3)$ 
    under standard stability assumptions, provided that the trapezoidal parameter satisfies $\lambda_t = \tfrac{1}{2} + O(\Delta_t)$.
    However, our ablations indicate that not enforcing this constraint is better for empirical performance. See Appendix~\ref{app:trap-error-proof} and \ref{app:trap-param} for details.
\end{remark}

Our new discretization framework and the two instantiations, exponential-Euler and exponential-trapezoidal, are, to the best of our knowledge, novel for structured SSMs used in deep learning.
\Cref{tab:discretizations} compares and summarizes canonical and commonly used discretization schemes for state space models.

\subsubsection{Exponential-Trapezoidal Recurrence as an Implicit Convolution}

Our generalized exponential-trapezoidal discretization is equivalent to applying a \emph{data-dependent} convolution of size two on the state-input to the SSM.
In particular, a normal SSM in recurrent form materializes the state-input $\vv_t = \bB_tx_t$, then computes a linear recurrence $\vh_t = \alpha_t \vh_{t-1} + \gamma_t \vv_t$. In \eqref{eq:mamba3-recurrence} we instead first apply a width-2 convolution on $\vv_t$ (weighted by $\beta, \gamma$) before passing it into the linear recurrence.

\begin{remark} [Convolution Differences]
    There is a distinct difference between the ``convolution'' induced by exponential-trapezoidal discretization and the standard short convolutions used by sequence models such as Mamba and GDN.
    Standard short convolutions are independent operations applied on $x_t$ (and often $\bB_t, \bC_t$) \textit{outside} the core recurrence, while our new discretization can be interpreted as a convolution on the \emph{state-input} $\bB_t x_t$ \emph{within} the core recurrence.
\end{remark}

\subsubsection{Parallel Representation of Exponential-Trapezoidal Recurrence}
\label{sec:method:trap:ssd}

Our new recurrence can be instantiated as a case of SSD and has a corresponding parallel form to \eqref{eq:ssd}.
Expanding the state recurrence from $\vh_0 = \gamma_0 \bB_0 x_0$ results in $\vh_T = \alpha_{T\cdots 2}(\gamma_0\alpha_1 + \beta_1)\bB_0x_0 + \cdots + \gamma_T \bB_T x_T$, where the SSM output is $\vy_T = \alpha_{T\cdots 2}(\gamma_0\alpha_1 + \beta_1) \bC_T^\top \bB_0 x_0 + \cdots + \gamma_T \bC_T^\top \bB_T x_T$.
Unrolling these rows shows that the mask induced by the trapezoidal update is no longer a fixed averaging of endpoints (as in the classical trapezoid rule), but a \emph{data-dependent convex combination} of the two interval endpoints. 

Under the SSD framework~\plaineqref{eq:ssd} with parallel form $\bY = (\bL \odot \bC {\bB}^{\top})\bX$,
Mamba-3 corresponds to a mask $\bL$ whose structure is a 1-semiseparable matrix composed with a 2-band matrix:\footnote{Incidentally, this is a special case of a 2-semiseparable matrix.}
\begin{align}
\bL =
    \begin{bmatrix}
    \gamma_0 & & & \\
    (\gamma_0\alpha_1 + \beta_1) & \gamma_1 & & \\
    \alpha_2(\gamma_0\alpha_1 + \beta_1) & (\gamma_1\alpha_2+\beta_2) & \gamma_2 & \\
    \vdots & & & \ddots \\
    \alpha_{T\cdots 2}(\gamma_0\alpha_1 + \beta_1) & & & \cdots & \gamma_T
    \end{bmatrix} =
    \begin{bmatrix}
    1 & & & \\
    \alpha_1 & 1 & & \\
    \alpha_2\alpha_1 & \alpha_2 & 1 & \\
    \vdots & & & \ddots \\
    \alpha_{T\cdots 1} & & & \cdots & 1
    \end{bmatrix}
    \begin{bmatrix}
    \gamma_0 & & & \\
    \beta_1 & \gamma_1 & & \\
    0 & \beta_2 & \gamma_2 & \\
    \vdots & & & \ddots \\
    0 & & & \cdots & \gamma_T
    \end{bmatrix}.
    \label{eq:mamba3-mask}
\end{align}

This parallel formulation enables the hardware-efficient matmul-focused calculation of the SSM output for training. 

We note that the convolutional connection of Mamba-3 can also be seen through this parallel dual form, where multiplication by the 2-band matrix in \eqref{eq:mamba3-mask} represents convolution with weights $\beta, \gamma$.
In \Cref{app:proof-tensor}, we use the SSD tensor contraction machinery to prove that the parallel form is equivalent to a vanilla SSM with a state-input convolution.

\begin{remark}
    The structured mask of Mamba-3 can be viewed as generalizing Mamba-2, which instead of the 2-band matrix has a diagonal matrix with $\gamma_t$ only~\plaineqref{eq:mamba2-mask}.
\end{remark}

\subsection{Complex-Valued SSMs}
\label{sec:method:complex}

Modern SSMs are designed with efficiency as the central goal, motivated by the need to scale to larger models and longer sequences.
For instance, successive architectures have progressively simplified the state-transition matrix: S4~\citep{S4} used complex-valued Normal Plus Low Rank (NPLR) matrices, Mamba~\citep{gu2024mambalineartimesequencemodeling} reduced this to a diagonal of reals, and Mamba-2~\citep{dao2024transformersssmsgeneralizedmodels} further simplified it to a single scaled identity matrix.
Although these simplifications largely maintain language modeling performance, recent works~\citep{merrill2025illusionstatestatespacemodels,sarrof2024expressivecapacitystatespace,grazzi2025unlockingstatetrackinglinearrnns} have shown that the restriction to real, non-negative eigenvalue transitions degrades the capabilities of the model on simple state-tracking tasks---here referring primarily to the solvable-group regime ($\text{TC}^0$) such as parity---which can be solved by a one-layer LSTM.
This limitation, formalized in Theorem 1 of \citep{grazzi2024mambacapableincontextlearning}, arises from restricting the eigenvalues of the transition matrix to real numbers, which cannot represent ``rotational'' hidden state dynamics.
For instance, consider the parity function on binary inputs $\{0,1\}$, defined as $\sum_t x_t \bmod 2$.
This task can be performed using update: $\vh_t = \bR(\pi x_t) \vh_{t-1}$, where $\bR(\cdot)$ is a 2-D rotation matrix. 
Such rotational dynamics cannot be expressed with real eigenvalues. 

\subsubsection{Complex SSM with Exponential-Euler Discretization}
To recover this capability, we begin with complex SSMs~\plaineqref{eq:complex-ssm}, which \emph{are} capable of representing state-tracking dynamics. 
We show that, under discretization (\Cref{prop:var-const}), complex SSMs can be formulated as real SSMs with a \emph{block-diagonal transition matrix composed of $2 \times 2$ rotation matrices} (\Cref{prop:complex-to-real-ssm}). 
We then show that this is equivalent to applying \emph{data-dependent rotary embeddings} on both the input and output projections $\bB, \bC$ respectively.
This result establishes a theoretical connection between complex SSMs and data-dependent RoPE embeddings (\Cref{prop:rope-trick}).
Finally, the ``RoPE trick'' used in \citet{su2023roformerenhancedtransformerrotary} allows for an efficient implementation of complex-valued state-transition matrices with minimal computational overhead compared to real-valued SSMs.

\begin{restatable}[Complex-to-Real SSM Equivalence]{proposition}{PropComplexRealSSM}\label{prop:complex-to-real-ssm}
Consider a complex-valued SSM
\begin{align}
\label{eq:complex-ssm}
    \dot{\vh}(t) &= \mathrm{Diag}\big(A(t) + i \boldsymbol{\theta}(t)\big)\,\vh(t) 
    + \big(\bB(t) + i\hat{\bB}(t)\big)\,x(t), \\ \nonumber
    y(t) &= \mathrm{Re}\!\left(
        \big(\bC(t) + i\hat{\bC}(t)\big)^{\top}\vh(t)
    \right),
\end{align}
where $\vh(t) \in \mathbb{C}^{N/2}$, $\boldsymbol{\theta}(t),\bB(t),\hat{\bB}(t),\bC(t),\hat{\bC}(t)\in\mathbb{R}^{N/2}$, and $x(t),A(t)\in\mathbb{R}$.
Under exponential-Euler discretization, this system is equivalent to a real-valued SSM
\begin{align}
\label{eq:real-ssm}
    \vh_t &= e^{\Delta_t A_t}\,\bR_t\,\vh_{t-1} + \Delta_t \bB_t x_t, \\  \nonumber
    y_t   &= \bC_t^{\top}\vh_t,
\end{align}
with state $\vh_t \in \mathbb{R}^N$, projections 
\[
\bB_t \coloneqq \begin{bmatrix} \bB_t \\ \hat{\bB}_t \end{bmatrix} \in \mathbb{R}^N, 
\qquad 
\bC_t \coloneqq \begin{bmatrix} \bC_t \\ -\hat{\bC}_t \end{bmatrix} \in \mathbb{R}^N, 
\]
and a transition matrix 
\[
\bR_t \;\coloneqq\; \text{Block}\Big(\{R(\Delta_t \boldsymbol{\theta_t}[i])\}_{i=1}^{N/2}\Big)
\in \mathbb{R}^{N \times N}
,
\qquad 
R(\theta) \coloneqq 
\begin{bmatrix}
    \cos(\theta) & -\sin(\theta) \\
    \sin(\theta) & \cos(\theta)
\end{bmatrix}.
\]
\end{restatable}

The proof is given in \cref{app:proof-complex-to-real-ssm}. 

\Cref{prop:complex-to-real-ssm} shows that the discretized complex SSM of state dimension $N/2$ has an equivalent real SSM with doubled state dimension ($N$), and its transition matrix is a scalar decayed block-diagonal matrix of $2\times2$ data-dependent rotation matrices ($e^{\Delta_t A_t} \bR_t$).

\begin{restatable}[Complex SSM, Data-Dependent RoPE Equivalence]{proposition}{PropRoPETrick}\label{prop:rope-trick} 
Under the notation established in \Cref{prop:complex-to-real-ssm}, consider the real SSM defined in \eqref{eq:real-ssm} unrolled for $T$ time-steps.
The output of the above SSM is equivalent to that of a vanilla scalar transition matrix-based SSM~\plaineqref{eq:mamba2-recurrence} with a data-dependent rotary embedding applied on the $\bB,\bC$ components of the SSM, as defined by:
\begin{equation}
    \vh_t = e^{\Delta_tA_t}\vh_{t-1} + \left( \prod_{i=0}^t \bR^\top_i \right) \Delta_t\bB_tx_t, \qquad \qquad
    y_t = \left[\left( \prod_{i=0}^t \bR^\top_i \right)\bC_t\right]^\top\vh_t \qquad
\end{equation}
where the matrix product represents right matrix multiplication, e.g., $\prod_{i=0}^1\bR_i = \bR_0 \bR_1$. 
We refer to the usage of a transformed real-valued SSM to compute the complex SSM as the ``RoPE trick.''
\end{restatable}
The proof is given in \cref{app:proof-rope-trick}.

To observe the connection of complex SSMs to RoPE embeddings, note that in the above proposition, the data-dependent rotations $\bR_i$ are aggregated across time-steps and applied to $\bC,\bB$, which, by the state space duality framework, correspond to the query ($\bQ$) and key ($\bK$) components of attention (\Cref{sec:prelim-sma}). 
Analogously, vanilla RoPE~\citep{su2023roformerenhancedtransformerrotary} applies \emph{data-independent} rotation matrices, where the rotation angles follow a fixed frequency schedule $\boldsymbol{\theta}[i] = 10000^{-2i/N}$. 

\subsubsection{Complex SSM with Exponential-Trapezoidal Discretization}

After deriving the recurrence for complex SSMs with exponential-Euler discretization, the generalization to exponential-trapezoidal discretization is similar.
\Cref{prop:rope-trick-trap} provides the full recurrence with the RoPE trick for Mamba-3.

\begin{restatable}[Rotary Embedding Equivalence with Exponential-Trapezoidal Discretization]{proposition}{PropRoPETrickTrap}\label{prop:rope-trick-trap}

Discretizing a complex SSM with the exponential-trapezoidal rule (\Cref{prop:trap}) yields the recurrence
\begin{align}
    \vh_t 
    &= \alpha_t \vh_{t-1} 
    + \beta_t \left(\prod_{i=0}^{t-1} \bR_i^\top \right)\bB_{t-1} x_{t-1} 
    + \gamma_t \left(\prod_{i=0}^t \bR_i^\top \right)\bB_t x_t, \nonumber \\
    y_t
    &= \left[\left(\prod_{i=0}^t \bR_i^\top\right)\bC_t\right]^\top \vh_t .
\end{align}
Here, $\bR_t$ is the block-diagonal rotation matrix defined in~\Cref{prop:complex-to-real-ssm}.
\end{restatable}

The proof is in \cref{app:proof-rope-trick-trap}.

We empirically validate that our complex SSM, implemented via data-dependent RoPE, is capable of solving state-tracking tasks that real-valued SSMs with and without standard RoPE cannot (\Cref{tab:reasoning}), supporting theoretical claims.

\subsection{Multi-Input, Multi-Output}
\label{sec:method:mimo}

Scaling test-time compute has opened new frontiers in model capability, such as agentic workflows, where inference takes up an increasing share of the overall compute budget.
This has placed a renewed focus on inference efficiency of language models and spurred the adoption of SSMs and sub-quadratic layers which feature fixed-sized hidden states and thus offer lower compute and memory requirements.
Although these new layers have a lower wall-clock time compared to Transformers, their decoding is heavily memory-bound, resulting in low hardware utilization.
In this section, we use the SSM perspective to introduce a methodological refinement to the Mamba-3 recurrence that allows for \emph{increased model FLOPs without increasing decoding wall-clock time, resulting in a better model with the same decoding speed.}

\paragraph{Decoding Arithmetic Intensity.}
To improve hardware efficiency, we need to consider the arithmetic intensity of token generation, 
defined as FLOPs divided by the number of input-output bytes for a given op. 
Since SSM decoding saturates the memory bandwidth with idle compute (i.e., being \textit{memory-bound}), we would like to increase its arithmetic intensity to effectively overlay compute with memory I/O.
More concretely, the arithmetic intensity for a single generation in Mamba is around $2.5$ ops per byte (\Cref{tab:arith_intensity_siso}), while the arithmetic intensity for bfloat16 matmul is about $295$ ops per byte for NVIDIA H100-SXM5~\citep{nvidia_h100_2022}. Consequently, SSM decoding falls far short of a compute-bound regime, and moreover it is not clear how one can adjust the existing parameters in Mamba to mitigate the lack of hardware efficiency. 
We note that this observation applies generally to other sub-quadratic models, such as causal linear attention.

\begin{table}[t]
  \centering

  \caption{
  Arithmetic Intensity for (a) SISO, (b) MIMO.
  The batch and head dimensions cancel out.
  The arithmetic intensity of MIMO increases linearly with rank $R$, enabling better hardware utilization during memory-bound phases like decode.
  Here $N$ is the state size (expansion factor) and $P$ is the head dimension. For Mamba-3, typically $R \ll N, P$.
  }

  \begin{subtable}[b]{0.48\linewidth}
    \small\centering
    \resizebox{\linewidth}{!}{
      \begin{tabular}{
        >{\raggedright\arraybackslash}p{0.20\linewidth}
        >{\raggedright\arraybackslash}p{0.16\linewidth}
        >{\centering\arraybackslash}p{0.20\linewidth}
        >{\raggedright\arraybackslash}p{0.35\linewidth}
      }
        \toprule
        \textbf{Input} & \textbf{Output} & \textbf{FLOPs} & \textbf{Arithmetic Intensity} \\
        \midrule
        \begin{tabular}[t]{@{}l@{}}
          $H_t:(N,P)$ \\[2pt]
          $x_t:(P)$ \\[2pt]
          $a_t:(1)$ \\[2pt]
          $b_t:(N)$ \\[2pt]
          $c_t:(N)$
        \end{tabular}
        &
        $y_t:(P)$
        &
        $5NP - P$
        &
        \begin{tabular}[t]{@{}l@{}}
          $\displaystyle \frac{5NP - P}{2(1+2N+P+NP)}$ \\[6pt]
          $\approx 2.5 = \Theta(1)$
        \end{tabular}
        \\
        \bottomrule
      \end{tabular}
    }
    \caption{SISO (2-byte data).}
    \label{tab:arith_intensity_siso}
  \end{subtable}\hfill
  \begin{subtable}[b]{0.48\linewidth}
    \small\centering
    \resizebox{\linewidth}{!}{
      \begin{tabular}{
        >{\raggedright\arraybackslash}p{0.20\linewidth}
        >{\raggedright\arraybackslash}p{0.16\linewidth}
        >{\centering\arraybackslash}p{0.20\linewidth}
        >{\raggedright\arraybackslash}p{0.35\linewidth}
      }
        \toprule
        \textbf{Input} & \textbf{Output} & \textbf{FLOPs} & \textbf{Arithmetic Intensity} \\
        \midrule
        \begin{tabular}[t]{@{}l@{}}
          $H_t:(N,P)$ \\[2pt]
          $x_t:(P,R)$ \\[2pt]
          $a_t:(1)$ \\[2pt]
          $b_t:(N,R)$ \\[2pt]
          $c_t:(N,R)$
        \end{tabular}
        &
        $y_t:(P,R)$
        &
        $4NPR+NP-PR$
        &
        \begin{tabular}[t]{@{}l@{}}
          $\displaystyle \frac{4NPR+NP-PR}{2(1+2NR+PR+NP)}$ \\[6pt]
          $=\Theta(\min(N,P,R))$ \\
          $=\Theta(R)$, $R \ll N,P$
        \end{tabular}
        \\
        \bottomrule
      \end{tabular}
    }
    \caption{MIMO (2-byte data).}
    \label{tab:arith_intensity_mimo}
  \end{subtable}
  \label{tab:arith_intensity}
\end{table}

\paragraph{From SISO to MIMO.}

Consider a single head of a typical SSM with \emph{head dimension} $P$,
which involves stacking the SISO recurrence $\vh_t \gets \alpha_t\vh_{t-1} + \Delta_t\bB_t x_t$ with $P$ copies sharing the same $\alpha_t, \Delta_t$ and $\bB_t$.
The resulting broadcasted recurrence $\vh_t \gets \alpha_t \vh_{t-1} + \Delta_t\bB_t \vx_t^\top$ takes vector inputs $\vx_t \in \mathbb{R}^P$ and has matrix-valued states $\vh_t \in \mathbb{R}^{N \times P}$.

Note that the memory traffic (input and output size) is dominated by the state $\vh_t$, while the computation mainly comprises the outer product $\bB_t \vx_t^\top$ which has FLOPs proportional to $NP$.
By increasing the dimension of the latter terms, transforming $\bB_t \in \mathbb{R}^{N} \to \bB_t \in \mathbb{R}^{N \times R}$ and $\vx_t \in \mathbb{R}^{P} \to \vx_t \in \mathbb{R}^{P \times R}$,
the memory traffic does not significantly increase (for small $R$) while the FLOPs consumed increase by a factor of $R$ (\Cref{tab:arith_intensity_siso}).
Thus, this transformation increases the arithmetic intensity of the recurrence.
Furthermore, the increase in arithmetic intensity is translated into practical gains, since the outer product $\bB_t \vx_t^\top$ becomes a hardware-efficient matrix-matrix product (matmul), which is computed using fast tensor-cores, incurring only a marginal latency cost.
As a result, the MIMO recurrence is more expressive than the original SISO recurrence, computing $R {\scriptstyle\times}$ more FLOPs while practically preserving the decoding speed.

For similar reasons, the computation of the output from the state, $\vy_t \gets \bC_t^\top \vh_t$ acquires an extra rank $R$ by modifying the output projection as $\bC_t \in \mathbb{R}^{N} \to \bC_t \in \mathbb{R}^{N \times R}$.
Overall, this transformation is equivalent to expanding the original single-input, single-output (SISO) recurrence to multi-input, multi-output (MIMO).

\paragraph{Training MIMO SSMs.}

While the MIMO formulation is motivated by \emph{inference} efficiency, the \emph{training} algorithms for SSMs (including our developments in \cref{sec:method:trap}, \cref{sec:method:complex}) have been typically developed for SISO models.
We begin with the observation that 
MIMO SSMs can be expressed in terms of $R^2$ SISO SSMs, where $R$ SISO SSMs sharing the same recurrence are summed for each of the $R$ MIMO outputs.
In particular,
define $ \bC_t^{(i)} \in \mathbb{R}^N, \bB_t^{(j)} \in \mathbb{R}^N, \vx_t^{(j)} \in \mathbb{R}, \Delta_t \in \mathbb{R}$, where $i,j \in \{0, ..., R-1\}$, then we have,
\begin{align}
    \vh_t^{(j)} &\gets \alpha_t\vh_{t-1}^{(j)} + \Delta_t\bB_t^{(j)} \vx_t^{(j)} \label{eq:mamba3-mimo-siso-equiv-A}\\ 
    \vh_t &= \sum_{j=0}^{R-1} \vh_t^{(j)} \label{eq:mamba3-mimo-siso-equiv-B}\\
    \vy_t^{(i)} &\gets \left(\bC_t^{(i)}\right)^\top \vh_t \label{eq:mamba3-mimo-siso-equiv-C}
\end{align}
Thus, $y_t^{(i)} = \sum_j \mathsf{SSM}\!\left(\alpha,\Delta,\bB^{(j)},\bC^{(i)},\vx^{(j)}\right)_t$, where $\mathsf{SSM}\!\left(\alpha,\Delta,\bB^{(j)},\bC^{(i)},\vx^{(j)}\right)_t := \left(\bC_t^{(i)}\right)^\top \vh_t^{(j)}$ with $\vh_t^{(j)}$ from \plaineqref{eq:mamba3-mimo-siso-equiv-A}.

Furthermore, improvements to standard SISO-based SSM models can be directly applied to MIMO models as the underlying SISO training algorithms can be utilized as a black-box.
This observation allows a MIMO model to be trained by invoking the SISO algorithm $R^2$ times as a black box in parallel.
In contrast, when computed in the recurrent form, \eqref{eq:mamba3-mimo-siso-equiv-A}, \plaineqref{eq:mamba3-mimo-siso-equiv-B}, and \plaineqref{eq:mamba3-mimo-siso-equiv-C} can be performed sequentially, incurring only an $R$-times overhead relative to SISO SSMs (recall the discussion on MIMO decoding FLOPs).

\paragraph{Chunked Algorithm for MIMO SSMs.}
Many modern SISO recurrent models, including Mamba-2, are computed using a \emph{chunked} algorithm, where the sequence is divided into chunks of length $C$. Within each chunk, a parallel (but asymptotically slower) algorithm is applied, while a recurrence is computed across chunks.
Chunked algorithms interpolate between two extremes: a fully parallel and a fully sequential algorithm. By exploiting this structure, we can reduce the training cost of MIMO SSMs to $R$ times that of SISO SSMs.
This idea also appears in the SSD framework---SSD applies a hardware-friendly quadratic algorithm within each chunk, while using the recurrent form across chunks, and shows that when the state and head dimensions are comparable, setting the chunk size to this dimension yields an overall linear-time algorithm.
Specifically, SSD's intra-chunk computation incurs $\left(2C^2 N + 2C^2 P\right)$ FLOPs per chunk, giving a total of $\frac{T}{C}\left(2C^2 N + 2C^2 P \right) = 2TC(N+P)$. The inter-chunk computation incurs $4NPC + 2NP$ FLOPs per chunk, for a total of $\frac{T}{C}(4NPC + 2NP) = 4TNP + \frac{T}{C}2NP$ (ignoring negligible terms). Setting $C=P=N$, the total FLOP count is $8TN^2$, which is linear in $T$.

The chunked algorithm for SSD can be naturally generalized into MIMO SSMs.
In such a case, the FLOP counts of state projection $\bB \vx^\top$ and state emission $\bC^\top \vh$ increase by $R\times$, while the FLOP count of the intrachunk component $\bC^\top\bB$ increases by $R^2\times$. As a result, the intra-chunk computation incurs  $2 \cdot \left(\frac{T}{C} (CR)^2 N + \frac{T}{C} (CR)^2 P\right)$ FLOPs and inter-chunk computation incurs $4\cdot\frac{T}{C}NP(CR) + 2\cdot\frac{T}{C}NP$ FLOPs.
Thus, setting $CR=N=P$ yields a total FLOP count of $8TRN^2$, an $R$-fold increase in FLOP count. Intuitively, setting MIMO chunk size as $\frac{1}{R}$ times the SISO chunk size, i.e., $C_{\text{MIMO}} \gets \frac{1}{R}C_{\text{SISO}}$, maintains the SISO intra-chunk FLOP count while increasing the number of chunks by a factor of $R$, resulting in an overall $R$-times increase in FLOP count instead of an $R^2$-times increase while keeping the algorithm hardware-friendly.

The training speed of algorithms in practice depends on details of the kernel implementation strategy, architectural choices such as how the MIMO parameters are instantiated, and problem dimensions, but should be no more than $R$ times slower.
Our released Triton Mamba-3 SISO kernels are roughly on par with the Triton Mamba-2 kernels, and MIMO kernels only incur a slowdown of $2 {\scriptstyle\times}$ when $R=4$, as compute latency can be parallelized with memory movement. \Cref{tab:benchmark} benchmarks the prefill speed of various kernels which is equivalent to the forward pass of the training kernel.

\paragraph{MIMO Instantiation.}
Among various choices for MIMO parameterizations, Mamba-3's approach achieves a balance that preserves the state size and number of SSMs of its SISO counterpart, while avoiding excessive growth in parameter count.
The naive conversion of a SISO SSM to a rank $R$ MIMO SSM would incur an $R\times$ increase in parameters as all projections that model the inputs to the SSM, $\bB,\bC,\vx$, would increase.
Block-level components, such as the gate $\vz$ (which so far has been ignored for simplicity) and output $\vy$ projection would also be impacted.
This influx in parameter count would be intractable at larger model scales.
To counteract this, we make the following change.
Mamba's multi-value attention (MVA) head structure results in shared $\bB, \bC$ across heads, so these components' projections can be directly converted to incorporate the new MIMO rank $R$ with only a slight increase in parameter count from $DN$ to $DNR$ for the entire layer (recall $D$ as the model dimension).
However, the SSM input $\vx_t$, output $\vy_t$, and gate $\vz_t$ are unique per head and therefore dominate the parameter count.
Here, directly adjusting the projections would increase the parameter count from $DP$ to $DPR$ for \emph{each head}.
Instead, we keep the original SISO projection and element-wise scale each dimension of the projected output to size $R$ with a learnable, data-independent vector, resulting in $DP + PR$ parameters for each head.
This mitigates the multiplicative increase to a more reasonable additive parameter count increase. 
\Cref{app:mimo-for-mamba} details the parameterization, and all MIMO-variants in our paper are parameter-matched to their SISO counterparts by reducing the MLP width.

\begin{remark}
    For simplicity, all discussion in this section was for simpler 2-term recurrences such as that arising from exponential-Euler discretization; the generalization to the 3-term exponential-trapezoidal recurrence is similar.
\end{remark}

\subsection{Mamba-3 Architecture}
\label{sec:method:arch}
\begin{figure}[t!]
    \centering
    \includegraphics[width=\linewidth]{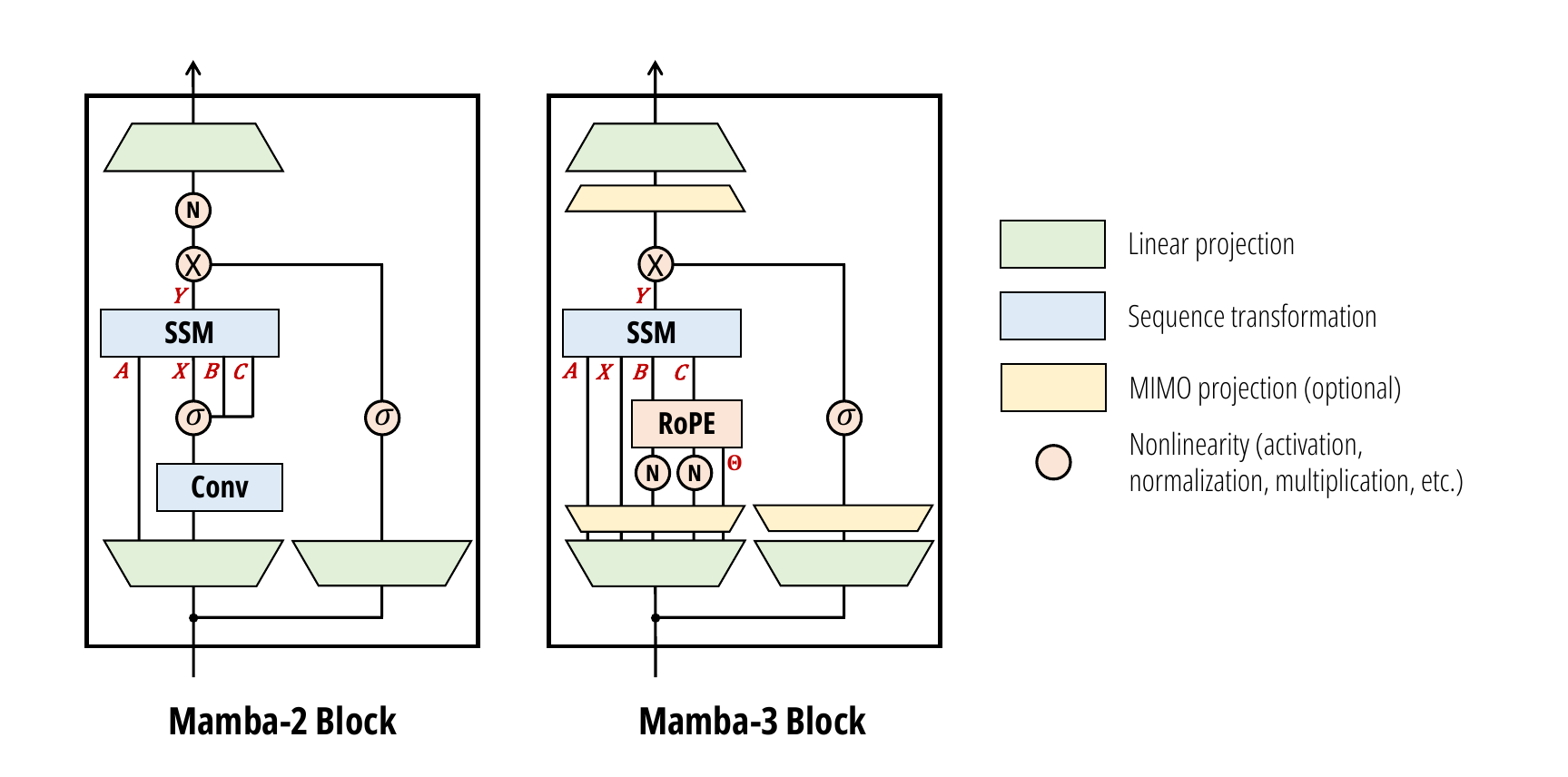}
    \caption{Contrasting Mamba-2 and Mamba-3 Architectures: Key updates include exponential-trapezoidal discretization, data-dependent RoPE embeddings, MIMO projections, QK normalization, and learnable biases.}
    \label{fig:arch_diff}
\end{figure}

The overall architecture follows Llama~\citep{grattafiori2024llama3herdmodels}, alternating Mamba-3 and SwiGLU blocks with pre-norm. 
The Mamba-3 block retains the overall layout of its predecessor, while introducing several key modifications.

\paragraph{Updated SSM Recurrence.}
The SSD layer is replaced with the more expressive complex-valued exponential-trapezoidal SSM defined in Proposition~\ref{prop:rope-trick-trap}.
Mamba-3 employs the SISO SSM by default to enable fair comparisons with other SISO-like models, but its MIMO variant can be trained and deployed as a stronger alternative to baseline Mamba-3 (\Cref{tab:downstream_evaluation}).
Our SSM $\bA$ is complex with both real and imaginary components produced by data-dependent projections. With \Cref{fig:arch_diff}, this is partitioned into the real-valued $A$ and imaginary-valued $\Theta$; the former is passed into the SSD black box as in Mamba-2, while the latter is computed through the RoPE trick.

\paragraph{BC / QK Normalization.}
RMS normalizations are added following the $\bB,\bC$ projection, mirroring the QKNorm commonly used in modern Transformers~\citep{henry2020querykeynormalizationtransformers,wortsman2023smallscaleproxieslargescaletransformer} and other recent linear models~\citep{yang2025gateddeltanetworksimproving, hu2025combaimprovingbilinearrnns}.
We call this either BC normalization (BCNorm) or QK normalization (QKNorm) interchangeably.
We find that BCNorm is also able to stabilize large-scale runs, resulting in the removal of the post-gate RMSNorm layer (introduced in Mamba-2 for stability) in our pure Mamba-3 models.
However, in hybrid models, the removed RMSNorm layer is crucial for long-context extrapolation (\Cref{tab:niah_n_real}).

\paragraph{$\bB, \bC$ Biases.}
Similarly to~\citet{yu2025blockbiasedmamba}, which proved that adding channel-specific biases to $\bB$ in a blockwise variant of Mamba-1 grants universal approximation capabilities, Mamba-3 incorporates learnable, head-specific, channel-wise biases into the $\bB$ and $\bC$ components after the BCNorm.

We hypothesize that these biases also induce a convolution-like behavior in the model. Specifically, adding biases to $\bB$ and $\bC$ introduces data-independent components into SSMs that function more similarly to convolutions.
Ablations on the bias parameterization are located in~\cref{app:arch-ablations}.

The combination of data-independent bias parameters, together with exponential-trapezoidal discretization (which itself induces a convolution on the state-input),
is empirically able to obviate the short causal convolution and its accompanying activation function present in Mamba-2 and most modern recurrent models (\cref{sec:exp:meth-abl}).

\newcommand{\combinedniahretrievaltable}{
\begin{table}[!t]
    \small
    \centering
    \caption{\small Retrieval capabilities measured by a mixture of real-world and synthetic retrieval tasks. Real-world retrieval tasks utilize cloze variants of the original datasets and are truncated to 2K length. Mamba-3 demonstrates strong associative recall, question-answering, and length generalization on needle-in-a-haystack (NIAH), but suffers with information extraction of semi-structured and unstructured data.
    The Transformer baseline uses RoPE which may explain its length generalization issues, and hybrid models utilize NoPE (no positional embeddings).
    We find a pre-gate, grouped RMSNorm can be added to Mamba-3 SISO hybrid models to improve the length generalization of the NIAH tasks at a slight decrease in real-world retrieval performance.
    }
    \resizebox{\textwidth}{!}{
    \sisetup{detect-weight,     %
        table-align-text-post=false,
         table-space-text-post={*} %
         }
    \begin{tabular}{l l*{6}{S[table-format=2.1]}*{9}{S[table-format=2.1]}}
        \toprule
        \multicolumn{2}{l}{\textbf{Model (1.5B)}} & {SWDE} & {SQD.} & {FDA} & {TQA} & {NQ} & {Drop} & \multicolumn{3}{c}{NIAH-Single-1} & \multicolumn{3}{c}{NIAH-Single-2} & \multicolumn{3}{c}{NIAH-Single-3} \\
        \cmidrule(lr){3-8} \cmidrule(lr){9-11} \cmidrule(lr){12-14} \cmidrule(lr){15-17}
        \multicolumn{2}{l}{Context Length} & \multicolumn{6}{c}{2048} & {1024} & {2048} & {4096} & {1024} & {2048} & {4096} & {1024} & {2048} & {4096} \\
        \midrule
        \multirow{5}{*}{\rotatebox[origin=c]{90}{Pure}} & Transformer & 48.9 & 46.6 & 58.4 & 67.5 & 31.7 & 26.4 & 100.0 & 100.0 & 0.0 & 92.2 & 100.0 & 0.0 & 98.6 & 99.4 & 0.0 \\
        \cmidrule(lr){2-17}
        & GDN & \Bf 32.7 & \Uline{40.0} & \Bf 28.3 & 63.5 & \Uline{25.7} & 24.5 & \Bf 100.0 & \Bf 100.0 & \Bf 99.8 & \Bf 100.0 & \Uline{93.8} & \Uline{49.8} & 83.8 & 68.4 & \Bf 34.2 \\
        & Mamba-2 & \Uline{30.7} & 39.1 & \Uline{23.7} & \Uline{64.3} & 25.1 & \Bf 28.5 & \Bf 100.0 & 99.6 & 62.0 & \Bf 100.0 & 53.8 & 11.8 & \Bf 95.8 & \Bf 87.4 & 13.4 \\
        & \textbf{Mamba-3 SISO} & 28.5 & \Bf 40.1 & 23.4 & \Bf 64.5 & \Bf 26.5 & \Uline{27.4} & \Bf 100.0 & \Bf 100.0 & \Uline{88.2} & \Bf 100.0 & \Bf 95.4 & \Bf 50.6 & \Uline{92.4} & \Uline{81.4} & \Bf 34.2 \\
        \addlinespace[0.6ex]
        \cdashline{2-17}
        \addlinespace[0.6ex]
        & \textbf{Mamba-3 MIMO} & \Bf 36.3 & \Bf 41.7 & \Bf 29.3 & \Bf 64.5 & \Uline{26.2} & 26.3 & \Bf 100.0 & \Bf 100.0 & \Uline{93.0} & \Bf 100.0 & 86.0 & 40.4 & \Bf 95.8 & \Uline{84.4} & 25.6\\
        \midrule
        \multirow{4}{*}{\rotatebox[origin=c]{90}{Hybrid}} & GDN & 54.6 & \Bf 48.4 & 58.8 & 64.9 & 32.7 & \Bf 30.0 & \Bf 100.0 & \Bf 100.0 & \Uline{71.4} & 99.6 & \Bf 100.0 & \Uline{60.2} & 70.0 & 96.2 & \Uline{24.0} \\
        & Mamba-2 & 58.2 & 45.6 & \Bf 71.0 & \Bf 66.1 & \Bf 33.4 & 28.1 & \Bf 100.0 & \Bf 100.0 & 3.2 & 99.6 & 98.8 & 0.0 & 98.2 & \Uline{98.0} & 0.0 \\
        & Mamba-3 SISO & \Uline{58.5} & 47.0 & \Uline{65.9} & 64.8 & \Bf 33.4 & 27.0 & \Bf 100.0 & \Bf 100.0 & 36.2 & \Bf 100.0 & \Bf 100.0 & 9.4 & \Bf 99.8 & \Bf 100.0 & 8.8 \\
        & Mamba-3 SISO Norm$^*$ & \Bf 58.6 & \Uline{47.3} & 52.4 & \Uline{65.7} & 33.3 & \Uline{28.5} & \Bf 100.0 & \Bf 100.0 & \Bf 100.0 & \Bf 100.0 & \Bf 100.0 & \Bf 96.0 & \Bf 99.8 & 97.2 & \Bf 56.8 \\
        \bottomrule
    \end{tabular}
    }
    \label{tab:niah_n_real}
\end{table}
}

\newcommand{\ablationstable}{
\begin{table}[!t]
    \small
    \centering
    \caption{\small\textbf{Left}: Ablations on core modeling components of Mamba-3 SISO, results on test split of dataset. \textbf{Right}: Formal language evaluation (scaled accuracy, \%). Higher is better. SISO models are trained on short sequences and evaluated on longer lengths to test length generalization. For GDN we report the variant with eigenvalue range \([-1,1]\).}
    \sisetup{detect-weight, table-align-text-post=false, table-space-text-post={*}}

    \makebox[\textwidth][c]{%
      \begin{subtable}[t]{0.4\textwidth}
        \centering
        \begingroup\setlength{\tabcolsep}{4pt}
        \begin{tabular}{@{}p{0.68\linewidth} S[table-format=2.2]@{}}
            \toprule
            \textbf{Model Variant} & {ppl $\downarrow$} \\
            \midrule
            Mamba-3 $-$ bias $-$ trap & 16.68 \\
            Mamba-3 $-$ bias & 16.49 \\
            Mamba-3 & \Bf 15.72 \\
            Mamba-3 $+$ conv & 15.85 \\
            \bottomrule
        \end{tabular}
        \endgroup
        \caption{Component ablation at 440M scale. A combination of our BC bias and exponential-trapezoidal discretization makes the ubiquitous short convolution optional.}
        \label{tab:abl:meth}
      \end{subtable}%
      \hspace{0.03\textwidth}%
      \begin{subtable}[t]{0.55\textwidth}
        \centering
        \begingroup\setlength{\tabcolsep}{4pt}
        \begin{tabular}{@{}>{\raggedright\arraybackslash}p{0.43\linewidth} S[table-format=3.2] S[table-format=3.2] S[table-format=3.2]@{}}
            \toprule
            \textbf{Model} &
            \multicolumn{1}{l}{\shortstack{Parity $\uparrow$}} &
            \multicolumn{1}{l}{\shortstack{Arith.\ w/o \\brackets $\uparrow$}} &
            \multicolumn{1}{l}{\shortstack{Arith.\ w/ \\brackets $\uparrow$}} \\
            \midrule
            Mamba-3 & 100.00 & 98.51 & 87.75 \\
            Mamba-3 (w/ Std. RoPE) & 1.56 & 20.70 & 2.62 \\
            Mamba-3 (w/o RoPE) & 2.27 & 1.49 & 0.72 \\
            Mamba-2 & 0.90 & 47.81 & 0.88 \\
            GDN [-1,1] & 100.00 & 99.25 & 93.50 \\
            \bottomrule
        \end{tabular}
        \endgroup
        \caption{Performance comparison on formal language tasks. Results show that unlike Mamba-2, Mamba-3 features state-tracking ability stemming from data-dependent RoPE embeddings. 
        }
        \label{tab:reasoning}
      \end{subtable}%
    }%
\end{table}
}

\newcommand{\latencytable}{
\begin{figure}[!t]
  \centering
  \begin{minipage}[b]{0.49\linewidth}
    \centering
    \includegraphics[width=\linewidth]{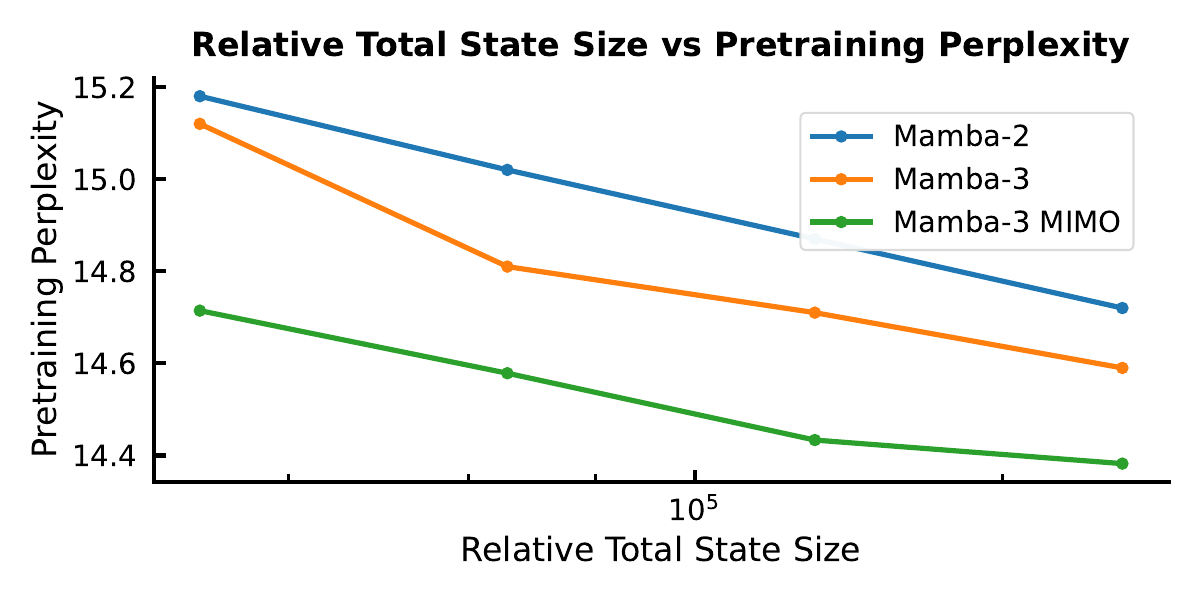}
    \captionof{figure}{Exploration of state size (inference speed proxy) versus pretraining perplexity (performance proxy) across different Mamba variants. Mamba-3 improves the Pareto frontier compared to previous recurrent SISO models, while incorporating MIMO further shifts the frontier through better modeling performance without increasing state size.
    }
    \label{fig:pareto}
  \end{minipage}\hfill
  \begin{minipage}[b]{0.48\linewidth}
    \small\centering
    \resizebox{\linewidth}{!}{
        \begin{tabular}{l*{2}{S[table-format=1.3]}*{2}{S[table-format=1.3]}}
          \toprule
          \multirow{2}{*}{\textbf{Model}}
          & \multicolumn{2}{c}{\textbf{FP32}}
          & \multicolumn{2}{c}{\textbf{BF16}} \\
          \cmidrule(lr){2-3} \cmidrule(lr){4-5}
          & {$d_{\text{state}}=64$} & {$d_{\text{state}}=128$}
          & {$d_{\text{state}}=64$} & {$d_{\text{state}}=128$} \\
          \midrule
          Mamba-2        & 0.295 & 0.409 & 0.127 & 0.203 \\
          GDN & 0.344 & 0.423 & 0.176 & 0.257 \\
          Mamba-3 (SISO) & 0.310 & 0.399 & 0.110 & 0.156 \\
          Mamba-3 (MIMO) & 0.333 & 0.431 & 0.137 & 0.179 \\
          \bottomrule
        \end{tabular}
    }
    \captionof{table}{
    Kernel latency (in milliseconds) comparison across models, precision, and $d_{\text{state}}$ values.
    Mamba-3 introduces minimal overhead compared to Mamba-2 and features highly efficient practical implementations.
    Our Mamba-3 SISO kernels are faster than reference Mamba-2 and GDN kernels at the commonly used bf16, $d_\text{state}=128$ setting.
    Mamba-3 MIMO ($R=4$) incurs little additional cost compared to SISO.
    }
    \label{tab:benchmark}
  \end{minipage}
\end{figure}
}

\section{Empirical Validation}
\label{sec:empirical_validation}
We empirically validate our SSM-centric methodological changes through the Mamba-3 model on a host of synthetic and real-world tasks.
\cref{sec:exp:lm} evaluates Mamba-3 on language modeling and retrieval-based tasks.
\cref{sec:exp:meth-abl} ablates the effect of our new SSM components such as discretization and complex transitions.
\cref{sec:exp:pareto} explores the inference efficiency of the Mamba-3 family and MIMO Mamba-3's benefits over the SISO variant under fixed inference compute, and \cref{sec:exp:speed} benchmarks the performance of our Mamba-3 training and inference kernels.

\subsection{Language Modeling}
\label{sec:exp:lm}

\begin{table}[!t]
    \small
    \centering
    \caption{Downstream language modeling evaluations on models trained with 100B FineWeb-Edu tokens. Best results for each size are \textbf{bolded}, and second best are \underline{underlined}, excluding Mamba-3 MIMO variants. All models are trained with the same procedure. Mamba-3 SISO outperforms Mamba-2 and others at every model scale, and the MIMO variant with rank $R=4$ further improves modeling capabilities.}
    \resizebox{\textwidth}{!}{
    \sisetup{detect-weight,     %
        table-align-text-post=false,
         table-space-text-post={*} %
         }
    \begin{tabular}{l*{10}{S[table-format=2.1]}}
        \toprule
        \textbf{Model} & {FW-Edu} & {LAMB.} & {LAMB.} & {HellaS.} & {PIQA} & {Arc-E} & {Arc-C} & {WinoGr.} & {OBQA} & {Average} \\
         & {ppl $\downarrow$} & {ppl $\downarrow$}  &  {acc $\uparrow$} & {acc\_n $\uparrow$} & {acc $\uparrow$} & {acc $\uparrow$} & {acc\_n $\uparrow$} & {acc $\uparrow$} & {acc $\uparrow$} & {acc $\uparrow$} \\
        \midrule
        Transformer-180M & 16.89 & 45.0 & \Bf 32.5 & 39.0 & \Bf 67.1 & 59.8 & \Uline{27.9} & 51.2 & 21.8 & 42.8 \\
        GDN-180M & \Uline{16.52} & \Bf 40.8 & 31.3 & \Uline{40.2} & 66.3 & \Bf 62.3 & \Bf 28.2 & 51.7 & 22.0 & \Uline{43.2} \\
        Mamba-2-180M & 16.76 & 41.8 & 30.9 & 40.1 & \Uline{66.8} & 60.1 & 27.3 & \Bf 52.0 & \Bf 23.2 & 42.9 \\
        \textbf{Mamba-3-SISO-180M} & \Bf 16.59 & \Bf 37.7 & \Bf 32.5 & \Bf 40.8 & 66.1 & \Uline{61.5} & \Uline{27.9} & \Bf 52.0 & \Uline{22.8} & \Bf 43.4 \\
        \addlinespace[0.6ex]
        \cdashline{1-11}
        \addlinespace[0.6ex]
        \textbf{Mamba-3-MIMO-180M} & \Bf 16.46 & \Bf 32.1 & \Bf 34.0 & \Bf 41.0 & 66.7 & \Uline{60.6} & 27.7 & \Bf 52.9 & 22.0 & \Bf 43.5 \\
        \midrule
        Transformer-440M & 13.03 & 21.2 & \Uline{41.7} & 50.5 & 69.9 & 67.6 & \Uline{34.6} & \Bf 56.7 & \Uline{26.0} & 49.6 \\
        GDN-440M & 13.01 & \Bf 18.0 & \Bf 41.9 & 50.9 & 70.0 & 67.0 & \Uline{34.6} & \Uline{56.1} & \Bf 27.6 & \Uline{49.7} \\
        Mamba-2-440M & \Uline{13.00} & \Uline{19.6} & 40.8 & \Bf 51.7 & \Uline{70.6} & \Uline{68.8} & \Bf 35.0 & 54.1 & \Uline{26.0} & 49.6 \\
        \textbf{Mamba-3-SISO-440M} & \Bf 12.87 & \Uline{19.6} & 40.2 & \Bf 51.7 & \Bf 71.9 & \Bf 68.9 & 34.4 & 55.8 & \Uline{26.0} & \Bf 49.8 \\
        \addlinespace[0.6ex]
        \cdashline{1-11}
        \addlinespace[0.6ex]
        \textbf{Mamba-3-MIMO-440M} & \Bf 12.72 & \Bf 17.1 & \Bf 43.4 & \Bf 52.8 & \Uline{70.8} & \Bf 69.6 & \Bf 35.6 & \Uline{56.3} & \Bf 28.4 & \Bf 51.0 \\
        \midrule
        Transformer-880M & 11.42 & 15.0 & 44.7 & 57.2 & 72.6 & 71.6 & \Uline{39.2} & 57.7 & 26.8 & 52.8 \\
        GDN-880M  & 11.37 & \Bf 12.9 & \Bf 47.6 & 57.3 & \Uline{73.3} & 71.4 & 38.7 & \Bf 58.8 & 28.6 & \Uline{53.7} \\
        Mamba-2-880M & \Uline{11.35} & 13.8 & 45.0 & \Uline{58.1} & 72.5 & \Uline{72.3} & 38.7 & 56.8 & \Bf 30.2 & 53.4 \\
        \textbf{Mamba-3-SISO-880M} & \Bf 11.23 & \Bf 12.9 & \Uline{47.2} & \Bf 58.8 & \Bf 73.6 & \Bf 72.7 & \Bf 40.2 & \Uline{58.4} & \Uline{30.0} & \Bf 54.4 \\
        \addlinespace[0.6ex]
        \cdashline{1-11}
        \addlinespace[0.6ex]
        \textbf{Mamba-3-MIMO-880M} & \Bf 11.11 & \Bf 11.8 & \Bf 49.5 & \Bf 59.2 & \Bf 73.7 & \Bf 74.7 & \Bf 41.2 & \Bf 59.9 & 28.6 & \Bf 55.3 \\
        \midrule
        Transformer-1.5B & 10.51 & 11.1 & \Bf 50.3 & 60.6 & \Uline{73.8} & 74.0 & 40.4 & \Uline{58.7} & 29.6 & 55.4 \\
        GDN-1.5B & \Uline{10.45} & \Bf 10.9 & 49.2 & 61.3 & \Bf 74.3 & \Uline{75.3} & 41.2 & 58.0 & 31.6 & \Uline{55.8} \\
        Mamba-2-1.5B & 10.47 & 12.0 & 47.8 & \Uline{61.4} & 73.6 & \Uline{75.3} & \Uline{41.8} & 57.5 & \Bf 32.6 & 55.7 \\
        \textbf{Mamba-3-SISO-1.5B} & \Bf 10.35 & \Bf 10.9 & \Uline{49.4} & \Bf 61.9 & 73.6 & \Bf 75.9 & \Bf 42.7 & \Bf 59.4 & \Uline{32.0} & \Bf 56.4 \\
        \addlinespace[0.6ex]
        \cdashline{1-11}
        \addlinespace[0.6ex]
        \textbf{Mamba-3-MIMO-1.5B} & \Bf 10.24 & \Bf 10.2 & \Bf 51.7 & \Bf 62.3 & \Bf 75.3 & \Bf 76.5 & \Bf 44.5 & \Bf 60.6 & \Bf 32.6 & \Bf 57.6 \\
        \bottomrule
    \end{tabular}
    }
    \label{tab:downstream_evaluation}
\end{table}

All models are pretrained with 100B tokens of the FineWeb-Edu dataset~\citep{penedo2024finewebdatasetsdecantingweb} with the Llama-3.1 tokenizer~\citep{grattafiori2024llama3herdmodels} at a 2K context length with the same standard training protocol. Training and evaluation details can be found in \cref{app:exp-details}.

Across all four model scales, Mamba-3 outperforms popular baselines at various downstream tasks (\cref{tab:downstream_evaluation}).
We highlight that Mamba-3 does not utilize the
external short convolution that has been empirically identified as an important component in many performant linear models~\citep{gu2024mambalineartimesequencemodeling,yang2025gateddeltanetworksimproving,allen2025canon}.

\subsubsection{MIMO}
We aim to further verify the gain from MIMO by investigating its language-modeling capabilities by training MIMO models with rank $R=4$ under the same settings.
To ensure that the total parameter count is comparable to SISO-based models, we decrease the inner dimension of the MLP layers in MIMO models to compensate for the increase due to the MIMO projections.
In the 1.5B-parameter models, for instance, the MLP inner dimension is reduced by only $6.6\%$, from 4096 to 3824. See~\Cref{app:mimo-for-mamba} for more details.

On both validation perplexity and our suite of language evaluation tasks (\cref{tab:downstream_evaluation}), we see significant gains when moving from SISO to MIMO for our Mamba-3 models. Namely, we achieve a significant perplexity gain of $0.11$ on the 1.5B models, and \Cref{fig:pareto} illustrates the downward shift in our validation loss. On the language evaluation front, we see gains on most tasks when compared to SISO, resulting in an average gain of 1.2 percentage points over SISO.

\subsubsection{Retrieval Capabilities}
\label{sec:exp:retrieval}
Beyond standard language modeling, an important measure for linear models is their retrieval ability---how well they can recall information from earlier in the sequence~\citep{arora2025mechanisticevaluationtransformersstate,arora2025simplelinearattentionlanguage}. Unlike attention models, which can freely revisit past context with the growing KV cache, linear models must compress context into a fixed-size state. This trade-off is reflected in the Transformer baseline's substantially stronger retrieval scores.  
To evaluate Mamba-3 under this lens, \cref{tab:niah_n_real} compares it against baselines on both real-world and synthetic needle-in-a-haystack (NIAH) tasks~\citep{hsieh2024rulerwhatsrealcontext}, using our pretrained 1.5B models from \cref{sec:exp:lm}. We restrict the task sequence length to 2K tokens to match the training setup and adopt the cloze-style format for our real-world tasks to mirror the next-token-prediction objective, following~\citet{arora2025simplelinearattentionlanguage,arora2024justreadtwiceclosing}.

Mamba-3 is competitive on real-world associative recall and question-answering (TQA, SQuAD) but struggles when extracting information from semi-structured or unstructured data (SWDE, FDA). On synthetic NIAH tasks, however, Mamba-3 surpasses or matches baselines on most cases and notably demonstrates markedly better out-of-distribution retrieval abilities than its Mamba-2 predecessor.

\paragraph{Improving Retrieval with Hybrid Models.}

Because of the natural retrieval-based weaknesses of fixed state-size, we predict that linear layers will be predominantly used \emph{in hybrid architectures} that mitigate this downside with quadratic self-attention layers.
To evaluate how Mamba-3 performs within this architectural paradigm, we train our hybrid models at the same scale in an interleaving fashion with a 5:1 ratio of linear layer to NoPE self-attention~\citep{yang2025ropenopeagainnew}.
As seen in prior work~\citep{waleffe2024empiricalstudymambabasedlanguage}, hybrid models outperform the Transformer baseline.
We find that the reintroduction of the pre-output projection RMSNorm (pre-gate, grouped RMSNorm in \Cref{tab:niah_n_real}) to the Mamba-3 layer improves the length generalization retrieval abilities at the slight cost of in-context, real-world retrieval tasks and is highly competitive as a linear sequence mixing backbone when mixed with self-attention.
However, the ideal norm type (grouped vs default) and its placement (pre- vs post-gate) is still unclear due to competing tradeoffs (\Cref{app:additional-results}, \Cref{tab:hybrid_norm_abl}), as we find that hybrid models and their exact characteristics and dynamics are complex and oftentimes unintuitive, a point echoed in 
recent works such as \citet{cabannes2025shortwindowattentionenables}.

\combinedniahretrievaltable

\subsection{SSM Methodology Ablations}
\label{sec:exp:meth-abl}
\ablationstable

\cref{tab:abl:meth} ablates the changes that Mamba-3 introduces to core SSM components, mainly the introduction of BC bias and exponential-trapezoidal discretization. We report the pretraining test perplexity on models at the 440M scale, trained for Chinchilla optimal tokens. We find that the bias and exponential-trapezoidal SSM synergize well and make the short convolution utilized by many current linear models redundant.

We empirically demonstrate that data-dependent RoPE in Mamba-3 enables state tracking.
Following \citet{grazzi2025unlockingstatetrackinglinearrnns}, we evaluate on tasks from the Chomsky hierarchy---Parity, Modular Arithmetic (without brackets), and Modular Arithmetic (with brackets)---and report scaled accuracies in \Cref{tab:reasoning}.
Mamba-3 solves Parity and Modular Arithmetic (without brackets), and nearly closes the accuracy gap on Modular Arithmetic (with brackets).
In contrast, Mamba-3 without RoPE, Mamba-3 with standard RoPE~\citep{su2023roformerenhancedtransformerrotary}, and Mamba-2 fail to learn these tasks.
We use the state-tracking-enabled variant of GDN and observe that Mamba-3 is competitive---matching parity and approaching its performance on both modular-arithmetic tasks. Experimental settings are covered in \cref{app:exp-details}.

\subsection{Inference Efficiency to Performance Tradeoff}
\label{sec:exp:pareto}

\latencytable

As $d_\text{state}$ governs the decode runtime for the sub-quadratic models considered in this paper (\Cref{sec:method:mimo}), we use it as a proxy for inference speed.
By plotting the validation perplexity (a proxy for model performance) as a function of $d_\text{state}$, we aim to formulate a holistic picture about how sub-quadratic models can trade off performance with inference speed. 

\cref{fig:pareto} shows such a Pareto frontier for the Mamba models considered in this paper.
For each data point, we train a 440M parameter model to $2\times$ Chinchilla optimal tokens on the Fineweb-Edu dataset, where the model is configured with a $d_\text{state}$ of $\{16, 32, 64, 128\}$.
As expected, we observe an inverse correlation between validation loss and $d_\text{state}$.
Moreover, there is a general downward shift on the Pareto frontier moving from Mamba-2 to Mamba-3, indicating a stronger model:
in this setting, Mamba-3 with $2\times$ smaller state size achieves better pretraining perplexity than its Mamba-2 counterpart, resulting in a faster model with the same quality or a better model for the same speed.

A further downward shift is observed when moving from the SISO variant of Mamba-3 to the MIMO variant of Mamba-3 (where we set the MIMO rank $R = 4$ and decrease the MLP inner dimension to parameter match the SISO variants).

We expand the comparison to include the GDN baseline in~\cref{app:additional-results}, \cref{fig:pareto_gdn}, which also shows Mamba-3 comparing favorably to GDN. %

\subsection{Fast Mamba-3 Kernels}
\label{sec:exp:speed}

\begin{table}[t!]
\centering
\caption{Prefill and Prefill+Decode latency across sequence lengths. Mamba-3 adds minimal overhead to its forward-pass and retains competitive decode latencies. Details in \Cref{app:inference-details}.
}
\label{tab:full-latency}
\resizebox{\linewidth}{!}{
\begin{tabular}{l 
                c c 
                c c 
                c c
                c c
                c c}
\toprule
\textbf{Model} 
& \multicolumn{2}{c}{\textbf{512 tokens}}
& \multicolumn{2}{c}{\textbf{1024 tokens}}
& \multicolumn{2}{c}{\textbf{2048 tokens}}
& \multicolumn{2}{c}{\textbf{4096 tokens}}
& \multicolumn{2}{c}{\textbf{16384 tokens}} \\
\cmidrule(lr){2-11}
& Prefill & Prefill+Dec
& Prefill & Prefill+Dec
& Prefill & Prefill+Dec
& Prefill & Prefill+Dec
& Prefill & Prefill+Dec \\
\midrule
vLLM (Llama-3.2-1B) 
& \Bf 0.26 & \Uline{4.45}
& \Bf 0.52 & 9.60
& \Bf 1.08 & 20.37
& \Bf 2.08 & 58.64
& \Bf 12.17 & 976.50 \\
Gated DeltaNet 
& \Uline{0.51} & 4.56
& \Uline{1.01} & \Uline{9.11}
& \Uline{2.01} & \Uline{18.22}
& \Uline{4.00} & \Uline{36.41}
& \Uline{16.21} & \Uline{145.87} \\
Mamba-2 
& \Uline{0.51} & 4.66
& 1.02 & 9.32
& 2.02 & 18.62
& 4.02 & 37.22
& 16.22 & 149.02 \\
Mamba-3 (SISO) 
& \Uline{0.51} & \Bf 4.39
& \Uline{1.01} & \Bf 8.78
& 2.02 & \Bf 17.57
& 4.01 & \Bf 35.11
& 16.22 & \Bf 140.61 \\
Mamba-3 (MIMO $R=4$) 
& 0.60 & 4.74
& 1.21 & 9.48
& 2.42 & 18.96
& 4.76 & 37.85
& 19.44 & 151.81 \\
\bottomrule
\end{tabular}}
\end{table}

We complement Mamba-3's methodological advances with optimized kernels that deliver fast inference in practical settings.
We implement a new series of inference kernels for Mamba-3---using Triton for the forward (prefill) path and CuTe DSL for decode---and compare their per-token decode latency against the released Triton kernels for Mamba-2 and GDN in Table~\ref{tab:benchmark}.\footnote{Details on each kernel DSL and the exact kernel fusion structure is provided in Appendix~\ref{app:inference-details}.}
The evaluation measures a single decode step at batch size 128 on a single H100 for both FP32 and BF16 datatypes; models are 1.5B parameters with model dimension $2048$ and state dimension $\in \{64, 128\}$.
Across all configurations, SISO achieves the lowest latency amongst baselines.
MIMO, with its higher arithmetic intensity, increases the decoding FLOPs without significantly increasing decode runtime.
Our benchmarks indicate that our CuTe DSL decode implementation is competitive and that the additional components of Mamba-3 (exponential-trapezoidal update, complex-valued state, and MIMO projections) are lightweight.
This supports our overall inference-first perspective: Mamba-3 admits a \textbf{simple, low-latency implementation} while providing strong empirical performance.

\Cref{tab:full-latency} benchmarks both end-to-end latency across different decoding sequence length and prefill time for the same sequence length.
The decode time is consistent with \cref{tab:benchmark}, where Mamba-3 (SISO) is fastest; Mamba-3 (MIMO) is on par with Mamba-2; and all linear methods are faster than optimized attention as sequence length grows.
We also see that MIMO incurs a moderate overhead for prefill, as discussed in \cref{sec:method:mimo}.
Details of the benchmark are in \cref{app:inference-details}.

\section{Related Work}
\label{sec:related}

\subsection{Linear-Time Sequence Mixers} 
\label{sec:related:linear}

A growing body of work seeks to replace the quadratic softmax-based attention mechanism~\citep{vaswani2017attention,bahdanau2016neuralmachinetranslationjointly} with linear runtime alternatives.
Prominent approaches can be categorized under three broad frameworks: linear attention, test-time training, and state space models. 

Many nascent linear attention (LA) models aimed to approximate softmax attention through kernel feature maps~\citep{linearattention,performer}, while recent models have discarded the feature maps for raw dot-products between queries and keys, modulated by decays or masks~\citep{sun2023retentivenetworksuccessortransformer,yang2024gatedlinearattentiontransformers}.
More recently, fast-weight programmers~\citet{schlag2021deltarule} that modulate the state memory with key-value pairs have also fallen under the umbrella term ``linear attention.''
\citet{yang2025gateddeltanetworksimproving,yang2025parallelizinglineartransformersdelta} originated from this line of work and enhanced traditional linear attention by replacing the additive memory update with a delta-rule recurrence.
This has further spurred on a host of work improving the efficiency and capabilities of linear models built on the delta rule~\citep{kimiteam2025kimilinearexpressiveefficient, hu2025combaimprovingbilinearrnns}.

A parallel line of test-time training (TTT) or test-time regression (TTR) work views sequence modeling as an online learning task during inference.
Here, the recurrent state represents a compressed summary of past inputs, and recurrent steps update the state to memorize new information~\citep{sun2025learninglearntesttime,zhang2025testtimetrainingright,tandon2025endtoendtesttimetraininglong}.
Equivalently, these methods can be viewed as optimization of a global regression objective, and recurrent state updates represent iterative optimization procedures such as variants of gradient descent~\citep{wang2025testtimeregressionunifyingframework}.

Structured state space models (SSMs) are another view of modern recurrent models inspired by classical signal processing and dynamical systems.
Early versions of SSMs such as S4~\citep{S4,S5,DSS} used linear time invariant (LTI) layers with structured state transition matrices, for example diagonal or low-rank plus diagonal, to facilitate efficient computation and stable learning of long-context tasks~\citep{S4,S5,DSS}.
The introduction of time-varying, input-dependent selectivity to SSMs in Mamba-1 \citep{gu2024mambalineartimesequencemodeling} reduced the disparity between self-attention and linear models on information-dense modalities, notably language modeling.
Subsequently, Mamba-2 \citep{dao2024transformersssmsgeneralizedmodels} formalized the connection between SSMs and (linear) attention through the structured state space duality (SSD) that we build on in this work.

\subsection{State Tracking and Complex State Space Models}

\paragraph{Expressivity and State Tracking.}
Recent work characterizes the types of state that recurrent, constant-memory mixers can maintain, revealing algorithmic deficiencies in previous SSM-based models. \citet{merrill2025illusionstatestatespacemodels} show that under finite precision, practical SSMs collapse to $\mathrm{TC}^0$, leading to failures on tasks like permutation composition over $S_5$ unless the primitive is extended.
Similarly, \citet{yu2025blockbiasedmamba} prove that a single-layer Mamba is not a universal approximator. Several modifications have been proposed to improve expressivity. For instance, the same work shows that a block-biased variant regains the universal approximation property with only minor changes, either through block decomposition or a channel-specific bias. Allowing negative eigenvalues or non-triangular transitions enables linear RNNs---including diagonal and Householder/DeltaNet forms---to capture parity and, under mild assumptions, regular languages \citep{grazzi2025unlockingstatetrackinglinearrnns}. Complex-valued parameterizations provide another avenue for enhanced expressivity.

\paragraph{Complex State Space Models.}
Structured SSMs prior to Mamba were frequently complex-valued, rooted in traditional SSM theory.
They also generally excelled in domains such as vision and audio, which have explicit frequency-based information content, rather than language.
While some models such as H3~\citep{fu2023hungryhungryhipposlanguage}, RetNet~\citep{sun2023retentivenetworksuccessortransformer}, and Megalodon~\citep{ma2024megalodonefficientllmpretraining} kept complex-valued SSMs while targeting language modeling, they still noticeably underperformed Transformers.

Additionally, because these models were LTI and were computed using very different algorithms (in particular, convolutions or explicit recurrence) than modern selective SSMs such as Mamba, they generally did not use the RoPE trick to handle the complex part.
An exception is RetNet, which introduced a model in between linear attention and Mamba-2 that used constant scalar decays (as opposed to no decay in LA and data-dependent decay in Mamba-2) with an additional constant complex phase that was implemented through RoPE.

In general, complex numbers have been empirically found to be unhelpful for language modeling, and hence were phased out in Mamba-1 and successors, including parallel lines of work on linear attention and test-time training.
Mamba-3 represents the first modern recurrent model with complex-valued state transitions,
which were introduced for specific purposes of increasing expressivity and state-tracking ability.
By incorporating the RoPE trick, this represents, to the best of our knowledge, the first usage of data-dependent RoPE grounded in theoretical motivations.

\subsection{Multi-Input, Multi-Output} 
S4~\citep{S4} is a single-input, single-output LTI system where each dimension of the input was assigned its own independent SSM.
Such SISO models have a significantly larger recurrent state than classical RNNs, and necessitated more complicated mathematical machinery to compute them efficiently.
Aiming to simplify the model, S5~\citep{S5} and LRU~\citep{LRU} replaced the set of SISO SSMs with a multi-input, multi-output SSM applied directly on the entire vectorized input.
This change reduced the effective state capacity but enabled an alternate computation path by directly computing the recurrence with a parallel scan.
While this trade-off between state capacity and modeling performance was less pronounced in LTI models, Mamba-1 (S6)~\citep{gu2024mambalineartimesequencemodeling} and Mamba-2~\citep{dao2024transformersssmsgeneralizedmodels} returned to the SISO system due to the importance of a large state size in the time-varying setting.
The computational bottleneck associated with the increased state size was addressed with a hardware-aware parallel scan algorithm for Mamba-1 and a matrix multiplication-based algorithm for Mamba-2.

The introduction of MIMO to Mamba-3 significantly diverges from prior work. Unlike previous MIMO models, which aimed to simplify training algorithms at the cost of slightly reduced expressivity,
Mamba-3's MIMO structure is motivated to \emph{increase} modeling power while preserving \emph{inference} efficiency.
Accordingly, its state expansion is kept at Mamba-1/-2 levels to maintain modeling capabilities while trading off additional training compute.

\subsection{The State Space Model Viewpoint}

Although modern recurrent models have several different viewpoints that largely converge (\cref{sec:related:linear}),
each framework has slightly different interpretations and motivations that can lead to different design spaces and extensions.
In particular, linear attention and test-time training are more closely related and can perhaps be lumped together under a framework of \emph{associative memory} that explicitly aims to memorize input data through ``key-value'' stores;
either through approximations to the canonical KV method (i.e., quadratic attention) in LA,
or by minimizing soft optimization objectives in TTT.
On the other hand, state space models have a different lineage, as reflected both in terminology (e.g., $A,B,C,X$ instead of $Q,K,V$) and in their natural extensions.
Notably, the methodological improvements in Mamba-3 are all associated with the SSM viewpoint specifically and are less motivated from associative memory frameworks.

\begin{enumerate}
    \item \textbf{Exponential-Trapezoidal Discretization.}
    The SSM viewpoint entails the discretization of a continuous ODE governing the system;
    our exponential-trapezoidal discretization falls out of an improved discretization method.
    As associative memory methods do not use discretization, it is not obvious how to interpret a 3-term recurrence such as exponential-trapezoidal under alternate viewpoints.
    \item \textbf{Complex-Valued State Transitions.}
    Complex SSMs have long been a staple of dynamical systems, and it is natural to consider complex values as an extension of selective SSMs.
    On the other hand, the associative memory framework interprets the $A$ state transition as a coefficient of an objective function, for example corresponding to the weight of an L2 regularization (or weight-decay) term in the optimization objective~\citep{wang2025testtimeregressionunifyingframework}.
    However, complex values are meaningless as the coefficient of a regression objective;
    hence, Mamba-3 is not obviously interpretable within these frameworks.
    \item \textbf{Multi-Input, Multi-Output.}
    MIMO is a classical concept from the state space model literature and does not naturally appear in associative memory (linear attention or test-time training) frameworks.
    However, we do note that the MIMO formulation introduced in this paper is not directly tied to SSM theory---and instead is motivated from a computational perspective---and our techniques can be adapted to other modern recurrent models as well.
\end{enumerate}

There continues to be vigorous progress in the development of linear-time sequence models, and the discussion here only captures a portion of them.
We anticipate a growing space of unified frameworks, improved understanding, and new generalizations as the development of these models continually evolves.

\section{Conclusion And Future Work}
We introduce Mamba-3, a state space model with several methodological improvements over prior SSMs: a more powerful recurrence via exponential-trapezoidal discretization; improved expressivity through complex-valued state transitions; and higher inference efficiency and modeling abilities with a MIMO formulation.
The base SISO version of Mamba-3 delivers strong language modeling results, both standalone and in interleaved hybrid architectures, and advances the Pareto frontier on the performance-efficiency tradeoff over prior linear sequence models.
The MIMO version trades off slower training for even stronger modeling power, while maintaining competitive inference efficiency compared to Mamba-2.
Put together, the techniques in Mamba-3 show simple and theoretically motivated improvements from the state space model viewpoint, and open up new directions and design principles for efficient sequence models.

\textbf{Acknowledgments.}

We gratefully acknowledge the support of the Schmidt Sciences AI2050 fellowship, the Google ML and Systems Junior Faculty Awards, the Google Research Scholar program, Princeton Language and Intelligence (PLI), Together AI, and Cartesia AI. 
KL is supported by the NSF GRFP under Grant DGE2140739.
We also thank Sukjun Hwang and Gaurav Ghosal for helpful feedback and discussions.

\printbibliography

@misc{gu2024mambalineartimesequencemodeling,
      title={Mamba: Linear-Time Sequence Modeling with Selective State Spaces}, 
      author={Albert Gu and Tri Dao},
      year={2024},
      eprint={2312.00752},
      archivePrefix={arXiv},
      primaryClass={cs.LG},
      url={https://arxiv.org/abs/2312.00752}, 
}

@misc{dao2024transformersssmsgeneralizedmodels,
      title={Transformers are SSMs: Generalized Models and Efficient Algorithms Through Structured State Space Duality}, 
      author={Tri Dao and Albert Gu},
      year={2024},
      eprint={2405.21060},
      archivePrefix={arXiv},
      primaryClass={cs.LG},
      url={https://arxiv.org/abs/2405.21060}, 
}

@misc{wortsman2023smallscaleproxieslargescaletransformer,
      title={Small-scale proxies for large-scale Transformer training instabilities}, 
      author={Mitchell Wortsman and Peter J. Liu and Lechao Xiao and Katie Everett and Alex Alemi and Ben Adlam and John D. Co-Reyes and Izzeddin Gur and Abhishek Kumar and Roman Novak and Jeffrey Pennington and Jascha Sohl-dickstein and Kelvin Xu and Jaehoon Lee and Justin Gilmer and Simon Kornblith},
      year={2023},
      eprint={2309.14322},
      archivePrefix={arXiv},
      primaryClass={cs.LG},
      url={https://arxiv.org/abs/2309.14322}, 
}

@misc{henry2020querykeynormalizationtransformers,
      title={Query-Key Normalization for Transformers}, 
      author={Alex Henry and Prudhvi Raj Dachapally and Shubham Pawar and Yuxuan Chen},
      year={2020},
      eprint={2010.04245},
      archivePrefix={arXiv},
      primaryClass={cs.CL},
      url={https://arxiv.org/abs/2010.04245}, 
}

@misc{yang2025gateddeltanetworksimproving,
      title={Gated Delta Networks: Improving Mamba2 with Delta Rule}, 
      author={Songlin Yang and Jan Kautz and Ali Hatamizadeh},
      year={2025},
      eprint={2412.06464},
      archivePrefix={arXiv},
      primaryClass={cs.CL},
      url={https://arxiv.org/abs/2412.06464}, 
}

@misc{yang2025parallelizinglineartransformersdelta,
      title={Parallelizing Linear Transformers with the Delta Rule over Sequence Length}, 
      author={Songlin Yang and Bailin Wang and Yu Zhang and Yikang Shen and Yoon Kim},
      year={2025},
      eprint={2406.06484},
      archivePrefix={arXiv},
      primaryClass={cs.LG},
      url={https://arxiv.org/abs/2406.06484}, 
}

@article{allen2025canon,
  author = {{Allen-Zhu}, Zeyuan},
  title = {{Physics of Language Models: Part 4.1, Architecture Design and the Magic of Canon Layers}},
  year = {2025},
  month = may,
  journal = {SSRN Electronic Journal},
  note = {\url{https://ssrn.com/abstract=5240330}} 
}

@misc{arora2025simplelinearattentionlanguage,
      title={Simple linear attention language models balance the recall-throughput tradeoff}, 
      author={Simran Arora and Sabri Eyuboglu and Michael Zhang and Aman Timalsina and Silas Alberti and Dylan Zinsley and James Zou and Atri Rudra and Christopher Ré},
      year={2025},
      eprint={2402.18668},
      archivePrefix={arXiv},
      primaryClass={cs.CL},
      url={https://arxiv.org/abs/2402.18668}, 
}

@misc{grazzi2025unlockingstatetrackinglinearrnns,
      title={Unlocking State-Tracking in Linear RNNs Through Negative Eigenvalues}, 
      author={Riccardo Grazzi and Julien Siems and Arber Zela and Jörg K. H. Franke and Frank Hutter and Massimiliano Pontil},
      year={2025},
      eprint={2411.12537},
      archivePrefix={arXiv},
      primaryClass={cs.LG},
      url={https://arxiv.org/abs/2411.12537}, 
}

@misc{sarrof2024expressivecapacitystatespace,
      title={The Expressive Capacity of State Space Models: A Formal Language Perspective}, 
      author={Yash Sarrof and Yana Veitsman and Michael Hahn},
      year={2024},
      eprint={2405.17394},
      archivePrefix={arXiv},
      primaryClass={cs.CL},
      url={https://arxiv.org/abs/2405.17394}, 
}

@misc{merrill2025illusionstatestatespacemodels,
      title={The Illusion of State in State-Space Models}, 
      author={William Merrill and Jackson Petty and Ashish Sabharwal},
      year={2025},
      eprint={2404.08819},
      archivePrefix={arXiv},
      primaryClass={cs.LG},
      url={https://arxiv.org/abs/2404.08819}, 
}

@book{tenenbaum1985ordinary,
  title={Ordinary Differential Equations: An Elementary Textbook for Students of Mathematics, Engineering, and the Sciences},
  author={Tenenbaum, M. and Pollard, H.},
  isbn={9780486649405},
  lccn={lc85012983},
  series={Dover Books on Mathematics},
  url={https://books.google.com/books?id=iU4zDAAAQBAJ},
  year={1985},
  publisher={Dover Publications}
}

@misc{su2023roformerenhancedtransformerrotary,
      title={RoFormer: Enhanced Transformer with Rotary Position Embedding}, 
      author={Jianlin Su and Yu Lu and Shengfeng Pan and Ahmed Murtadha and Bo Wen and Yunfeng Liu},
      year={2023},
      eprint={2104.09864},
      archivePrefix={arXiv},
      primaryClass={cs.CL},
      url={https://arxiv.org/abs/2104.09864}, 
}

@article{S4D,
  title   = {On the Parameterization and Initialization of Diagonal State Space Models},
  author  = {Gu, Albert and Gupta, Ankit and Goel, Karan and R{\'e}, Christopher},
  journal = {arXiv preprint arXiv:2206.11893},
  year    = {2022},
  url     = {https://arxiv.org/abs/2206.11893}
}

@misc{S4,
      title={Efficiently Modeling Long Sequences with Structured State Spaces}, 
      author={Albert Gu and Karan Goel and Christopher Ré},
      year={2022},
      eprint={2111.00396},
      archivePrefix={arXiv},
      primaryClass={cs.LG},
      url={https://arxiv.org/abs/2111.00396}, 
}

@misc{schlag2021deltarule,
      title={Linear Transformers Are Secretly Fast Weight Programmers}, 
      author={Imanol Schlag and Kazuki Irie and Jürgen Schmidhuber},
      year={2021},
      eprint={2102.11174},
      archivePrefix={arXiv},
      primaryClass={cs.LG},
      url={https://arxiv.org/abs/2102.11174}, 
}

@misc{LRU,
      title={Resurrecting Recurrent Neural Networks for Long Sequences}, 
      author={Antonio Orvieto and Samuel L Smith and Albert Gu and Anushan Fernando and Caglar Gulcehre and Razvan Pascanu and Soham De},
      year={2023},
      eprint={2303.06349},
      archivePrefix={arXiv},
      primaryClass={cs.LG},
      url={https://arxiv.org/abs/2303.06349}, 
}

@misc{DSS,
      title={Diagonal State Spaces are as Effective as Structured State Spaces}, 
      author={Ankit Gupta and Albert Gu and Jonathan Berant},
      year={2022},
      eprint={2203.14343},
      archivePrefix={arXiv},
      primaryClass={cs.LG},
      url={https://arxiv.org/abs/2203.14343}, 
}

@misc{S5,
      title={Simplified State Space Layers for Sequence Modeling}, 
      author={Jimmy T. H. Smith and Andrew Warrington and Scott W. Linderman},
      year={2023},
      eprint={2208.04933},
      archivePrefix={arXiv},
      primaryClass={cs.LG},
      url={https://arxiv.org/abs/2208.04933}, 
}

@misc{yu2025blockbiasedmamba,
      title={Block-Biased Mamba for Long-Range Sequence Processing}, 
      author={Annan Yu and N. Benjamin Erichson},
      year={2025},
      eprint={2505.09022},
      archivePrefix={arXiv},
      primaryClass={cs.LG},
      url={https://arxiv.org/abs/2505.09022}, 
}

@misc{arora2025mechanisticevaluationtransformersstate,
      title={Mechanistic evaluation of Transformers and state space models}, 
      author={Aryaman Arora and Neil Rathi and Nikil Roashan Selvam and Róbert Csordás and Dan Jurafsky and Christopher Potts},
      year={2025},
      eprint={2505.15105},
      archivePrefix={arXiv},
      primaryClass={cs.CL},
      url={https://arxiv.org/abs/2505.15105}, 
}

@misc{linearattention,
      title={Transformers are RNNs: Fast Autoregressive Transformers with Linear Attention}, 
      author={Angelos Katharopoulos and Apoorv Vyas and Nikolaos Pappas and François Fleuret},
      year={2020},
      eprint={2006.16236},
      archivePrefix={arXiv},
      primaryClass={cs.LG},
      url={https://arxiv.org/abs/2006.16236}, 
}

@misc{performer,
      title={Rethinking Attention with Performers}, 
      author={Krzysztof Choromanski and Valerii Likhosherstov and David Dohan and Xingyou Song and Andreea Gane and Tamas Sarlos and Peter Hawkins and Jared Davis and Afroz Mohiuddin and Lukasz Kaiser and David Belanger and Lucy Colwell and Adrian Weller},
      year={2022},
      eprint={2009.14794},
      archivePrefix={arXiv},
      primaryClass={cs.LG},
      url={https://arxiv.org/abs/2009.14794}, 
}

@misc{grazzi2024mambacapableincontextlearning,
      title={Is Mamba Capable of In-Context Learning?}, 
      author={Riccardo Grazzi and Julien Siems and Simon Schrodi and Thomas Brox and Frank Hutter},
      year={2024},
      eprint={2402.03170},
      archivePrefix={arXiv},
      primaryClass={cs.LG},
      url={https://arxiv.org/abs/2402.03170}, 
}

@book{
    Süli_Mayers_2003,
    place={Cambridge},
    title={An Introduction to Numerical Analysis},
    publisher={Cambridge University Press},
    author={Süli, Endre and Mayers, David F.},
    year={2003},
}

@misc{grattafiori2024llama3herdmodels,
      title={The Llama 3 Herd of Models}, 
      author={Aaron Grattafiori and Abhimanyu Dubey and Abhinav Jauhri and Abhinav Pandey and Abhishek Kadian and Ahmad Al-Dahle and Aiesha Letman and Akhil Mathur and Alan Schelten and Alex Vaughan and Amy Yang and Angela Fan and Anirudh Goyal and Anthony Hartshorn and Aobo Yang and Archi Mitra and Archie Sravankumar and Artem Korenev and Arthur Hinsvark and Arun Rao and Aston Zhang and Aurelien Rodriguez and Austen Gregerson and Ava Spataru and Baptiste Roziere and Bethany Biron and Binh Tang and Bobbie Chern and Charlotte Caucheteux and Chaya Nayak and Chloe Bi and Chris Marra and Chris McConnell and Christian Keller and Christophe Touret and Chunyang Wu and Corinne Wong and Cristian Canton Ferrer and Cyrus Nikolaidis and Damien Allonsius and Daniel Song and Danielle Pintz and Danny Livshits and Danny Wyatt and David Esiobu and Dhruv Choudhary and Dhruv Mahajan and Diego Garcia-Olano and Diego Perino and Dieuwke Hupkes and Egor Lakomkin and Ehab AlBadawy and Elina Lobanova and Emily Dinan and Eric Michael Smith and Filip Radenovic and Francisco Guzmán and Frank Zhang and Gabriel Synnaeve and Gabrielle Lee and Georgia Lewis Anderson and Govind Thattai and Graeme Nail and Gregoire Mialon and Guan Pang and Guillem Cucurell and Hailey Nguyen and Hannah Korevaar and Hu Xu and Hugo Touvron and Iliyan Zarov and Imanol Arrieta Ibarra and Isabel Kloumann and Ishan Misra and Ivan Evtimov and Jack Zhang and Jade Copet and Jaewon Lee and Jan Geffert and Jana Vranes and Jason Park and Jay Mahadeokar and Jeet Shah and Jelmer van der Linde and Jennifer Billock and Jenny Hong and Jenya Lee and Jeremy Fu and Jianfeng Chi and Jianyu Huang and Jiawen Liu and Jie Wang and Jiecao Yu and Joanna Bitton and Joe Spisak and Jongsoo Park and Joseph Rocca and Joshua Johnstun and Joshua Saxe and Junteng Jia and Kalyan Vasuden Alwala and Karthik Prasad and Kartikeya Upasani and Kate Plawiak and Ke Li and Kenneth Heafield and Kevin Stone and Khalid El-Arini and Krithika Iyer and Kshitiz Malik and Kuenley Chiu and Kunal Bhalla and Kushal Lakhotia and Lauren Rantala-Yeary and Laurens van der Maaten and Lawrence Chen and Liang Tan and Liz Jenkins and Louis Martin and Lovish Madaan and Lubo Malo and Lukas Blecher and Lukas Landzaat and Luke de Oliveira and Madeline Muzzi and Mahesh Pasupuleti and Mannat Singh and Manohar Paluri and Marcin Kardas and Maria Tsimpoukelli and Mathew Oldham and Mathieu Rita and Maya Pavlova and Melanie Kambadur and Mike Lewis and Min Si and Mitesh Kumar Singh and Mona Hassan and Naman Goyal and Narjes Torabi and Nikolay Bashlykov and Nikolay Bogoychev and Niladri Chatterji and Ning Zhang and Olivier Duchenne and Onur Çelebi and Patrick Alrassy and Pengchuan Zhang and Pengwei Li and Petar Vasic and Peter Weng and Prajjwal Bhargava and Pratik Dubal and Praveen Krishnan and Punit Singh Koura and Puxin Xu and Qing He and Qingxiao Dong and Ragavan Srinivasan and Raj Ganapathy and Ramon Calderer and Ricardo Silveira Cabral and Robert Stojnic and Roberta Raileanu and Rohan Maheswari and Rohit Girdhar and Rohit Patel and Romain Sauvestre and Ronnie Polidoro and Roshan Sumbaly and Ross Taylor and Ruan Silva and Rui Hou and Rui Wang and Saghar Hosseini and Sahana Chennabasappa and Sanjay Singh and Sean Bell and Seohyun Sonia Kim and Sergey Edunov and Shaoliang Nie and Sharan Narang and Sharath Raparthy and Sheng Shen and Shengye Wan and Shruti Bhosale and Shun Zhang and Simon Vandenhende and Soumya Batra and Spencer Whitman and Sten Sootla and Stephane Collot and Suchin Gururangan and Sydney Borodinsky and Tamar Herman and Tara Fowler and Tarek Sheasha and Thomas Georgiou and Thomas Scialom and Tobias Speckbacher and Todor Mihaylov and Tong Xiao and Ujjwal Karn and Vedanuj Goswami and Vibhor Gupta and Vignesh Ramanathan and Viktor Kerkez and Vincent Gonguet and Virginie Do and Vish Vogeti and Vítor Albiero and Vladan Petrovic and Weiwei Chu and Wenhan Xiong and Wenyin Fu and Whitney Meers and Xavier Martinet and Xiaodong Wang and Xiaofang Wang and Xiaoqing Ellen Tan and Xide Xia and Xinfeng Xie and Xuchao Jia and Xuewei Wang and Yaelle Goldschlag and Yashesh Gaur and Yasmine Babaei and Yi Wen and Yiwen Song and Yuchen Zhang and Yue Li and Yuning Mao and Zacharie Delpierre Coudert and Zheng Yan and Zhengxing Chen and Zoe Papakipos and Aaditya Singh and Aayushi Srivastava and Abha Jain and Adam Kelsey and Adam Shajnfeld and Adithya Gangidi and Adolfo Victoria and Ahuva Goldstand and Ajay Menon and Ajay Sharma and Alex Boesenberg and Alexei Baevski and Allie Feinstein and Amanda Kallet and Amit Sangani and Amos Teo and Anam Yunus and Andrei Lupu and Andres Alvarado and Andrew Caples and Andrew Gu and Andrew Ho and Andrew Poulton and Andrew Ryan and Ankit Ramchandani and Annie Dong and Annie Franco and Anuj Goyal and Aparajita Saraf and Arkabandhu Chowdhury and Ashley Gabriel and Ashwin Bharambe and Assaf Eisenman and Azadeh Yazdan and Beau James and Ben Maurer and Benjamin Leonhardi and Bernie Huang and Beth Loyd and Beto De Paola and Bhargavi Paranjape and Bing Liu and Bo Wu and Boyu Ni and Braden Hancock and Bram Wasti and Brandon Spence and Brani Stojkovic and Brian Gamido and Britt Montalvo and Carl Parker and Carly Burton and Catalina Mejia and Ce Liu and Changhan Wang and Changkyu Kim and Chao Zhou and Chester Hu and Ching-Hsiang Chu and Chris Cai and Chris Tindal and Christoph Feichtenhofer and Cynthia Gao and Damon Civin and Dana Beaty and Daniel Kreymer and Daniel Li and David Adkins and David Xu and Davide Testuggine and Delia David and Devi Parikh and Diana Liskovich and Didem Foss and Dingkang Wang and Duc Le and Dustin Holland and Edward Dowling and Eissa Jamil and Elaine Montgomery and Eleonora Presani and Emily Hahn and Emily Wood and Eric-Tuan Le and Erik Brinkman and Esteban Arcaute and Evan Dunbar and Evan Smothers and Fei Sun and Felix Kreuk and Feng Tian and Filippos Kokkinos and Firat Ozgenel and Francesco Caggioni and Frank Kanayet and Frank Seide and Gabriela Medina Florez and Gabriella Schwarz and Gada Badeer and Georgia Swee and Gil Halpern and Grant Herman and Grigory Sizov and Guangyi and Zhang and Guna Lakshminarayanan and Hakan Inan and Hamid Shojanazeri and Han Zou and Hannah Wang and Hanwen Zha and Haroun Habeeb and Harrison Rudolph and Helen Suk and Henry Aspegren and Hunter Goldman and Hongyuan Zhan and Ibrahim Damlaj and Igor Molybog and Igor Tufanov and Ilias Leontiadis and Irina-Elena Veliche and Itai Gat and Jake Weissman and James Geboski and James Kohli and Janice Lam and Japhet Asher and Jean-Baptiste Gaya and Jeff Marcus and Jeff Tang and Jennifer Chan and Jenny Zhen and Jeremy Reizenstein and Jeremy Teboul and Jessica Zhong and Jian Jin and Jingyi Yang and Joe Cummings and Jon Carvill and Jon Shepard and Jonathan McPhie and Jonathan Torres and Josh Ginsburg and Junjie Wang and Kai Wu and Kam Hou U and Karan Saxena and Kartikay Khandelwal and Katayoun Zand and Kathy Matosich and Kaushik Veeraraghavan and Kelly Michelena and Keqian Li and Kiran Jagadeesh and Kun Huang and Kunal Chawla and Kyle Huang and Lailin Chen and Lakshya Garg and Lavender A and Leandro Silva and Lee Bell and Lei Zhang and Liangpeng Guo and Licheng Yu and Liron Moshkovich and Luca Wehrstedt and Madian Khabsa and Manav Avalani and Manish Bhatt and Martynas Mankus and Matan Hasson and Matthew Lennie and Matthias Reso and Maxim Groshev and Maxim Naumov and Maya Lathi and Meghan Keneally and Miao Liu and Michael L. Seltzer and Michal Valko and Michelle Restrepo and Mihir Patel and Mik Vyatskov and Mikayel Samvelyan and Mike Clark and Mike Macey and Mike Wang and Miquel Jubert Hermoso and Mo Metanat and Mohammad Rastegari and Munish Bansal and Nandhini Santhanam and Natascha Parks and Natasha White and Navyata Bawa and Nayan Singhal and Nick Egebo and Nicolas Usunier and Nikhil Mehta and Nikolay Pavlovich Laptev and Ning Dong and Norman Cheng and Oleg Chernoguz and Olivia Hart and Omkar Salpekar and Ozlem Kalinli and Parkin Kent and Parth Parekh and Paul Saab and Pavan Balaji and Pedro Rittner and Philip Bontrager and Pierre Roux and Piotr Dollar and Polina Zvyagina and Prashant Ratanchandani and Pritish Yuvraj and Qian Liang and Rachad Alao and Rachel Rodriguez and Rafi Ayub and Raghotham Murthy and Raghu Nayani and Rahul Mitra and Rangaprabhu Parthasarathy and Raymond Li and Rebekkah Hogan and Robin Battey and Rocky Wang and Russ Howes and Ruty Rinott and Sachin Mehta and Sachin Siby and Sai Jayesh Bondu and Samyak Datta and Sara Chugh and Sara Hunt and Sargun Dhillon and Sasha Sidorov and Satadru Pan and Saurabh Mahajan and Saurabh Verma and Seiji Yamamoto and Sharadh Ramaswamy and Shaun Lindsay and Shaun Lindsay and Sheng Feng and Shenghao Lin and Shengxin Cindy Zha and Shishir Patil and Shiva Shankar and Shuqiang Zhang and Shuqiang Zhang and Sinong Wang and Sneha Agarwal and Soji Sajuyigbe and Soumith Chintala and Stephanie Max and Stephen Chen and Steve Kehoe and Steve Satterfield and Sudarshan Govindaprasad and Sumit Gupta and Summer Deng and Sungmin Cho and Sunny Virk and Suraj Subramanian and Sy Choudhury and Sydney Goldman and Tal Remez and Tamar Glaser and Tamara Best and Thilo Koehler and Thomas Robinson and Tianhe Li and Tianjun Zhang and Tim Matthews and Timothy Chou and Tzook Shaked and Varun Vontimitta and Victoria Ajayi and Victoria Montanez and Vijai Mohan and Vinay Satish Kumar and Vishal Mangla and Vlad Ionescu and Vlad Poenaru and Vlad Tiberiu Mihailescu and Vladimir Ivanov and Wei Li and Wenchen Wang and Wenwen Jiang and Wes Bouaziz and Will Constable and Xiaocheng Tang and Xiaojian Wu and Xiaolan Wang and Xilun Wu and Xinbo Gao and Yaniv Kleinman and Yanjun Chen and Ye Hu and Ye Jia and Ye Qi and Yenda Li and Yilin Zhang and Ying Zhang and Yossi Adi and Youngjin Nam and Yu and Wang and Yu Zhao and Yuchen Hao and Yundi Qian and Yunlu Li and Yuzi He and Zach Rait and Zachary DeVito and Zef Rosnbrick and Zhaoduo Wen and Zhenyu Yang and Zhiwei Zhao and Zhiyu Ma},
      year={2024},
      eprint={2407.21783},
      archivePrefix={arXiv},
      primaryClass={cs.AI},
      url={https://arxiv.org/abs/2407.21783}, 
}

@misc{hsieh2024rulerwhatsrealcontext,
      title={RULER: What's the Real Context Size of Your Long-Context Language Models?}, 
      author={Cheng-Ping Hsieh and Simeng Sun and Samuel Kriman and Shantanu Acharya and Dima Rekesh and Fei Jia and Yang Zhang and Boris Ginsburg},
      year={2024},
      eprint={2404.06654},
      archivePrefix={arXiv},
      primaryClass={cs.CL},
      url={https://arxiv.org/abs/2404.06654}, 
}

@misc{arora2024justreadtwiceclosing,
      title={Just read twice: closing the recall gap for recurrent language models}, 
      author={Simran Arora and Aman Timalsina and Aaryan Singhal and Benjamin Spector and Sabri Eyuboglu and Xinyi Zhao and Ashish Rao and Atri Rudra and Christopher Ré},
      year={2024},
      eprint={2407.05483},
      archivePrefix={arXiv},
      primaryClass={cs.CL},
      url={https://arxiv.org/abs/2407.05483}, 
}

@misc{penedo2024finewebdatasetsdecantingweb,
      title={The FineWeb Datasets: Decanting the Web for the Finest Text Data at Scale}, 
      author={Guilherme Penedo and Hynek Kydlíček and Loubna Ben allal and Anton Lozhkov and Margaret Mitchell and Colin Raffel and Leandro Von Werra and Thomas Wolf},
      year={2024},
      eprint={2406.17557},
      archivePrefix={arXiv},
      primaryClass={cs.CL},
      url={https://arxiv.org/abs/2406.17557}, 
}

@misc{li2024llminferenceservingsurvey,
      title={LLM Inference Serving: Survey of Recent Advances and Opportunities}, 
      author={Baolin Li and Yankai Jiang and Vijay Gadepally and Devesh Tiwari},
      year={2024},
      eprint={2407.12391},
      archivePrefix={arXiv},
      primaryClass={cs.DC},
      url={https://arxiv.org/abs/2407.12391}, 
}

@misc{kwon2023efficientmemorymanagementlarge,
      title={Efficient Memory Management for Large Language Model Serving with PagedAttention}, 
      author={Woosuk Kwon and Zhuohan Li and Siyuan Zhuang and Ying Sheng and Lianmin Zheng and Cody Hao Yu and Joseph E. Gonzalez and Hao Zhang and Ion Stoica},
      year={2023},
      eprint={2309.06180},
      archivePrefix={arXiv},
      primaryClass={cs.LG},
      url={https://arxiv.org/abs/2309.06180}, 
}

@misc{snell2024scalingllmtesttimecompute,
      title={Scaling LLM Test-Time Compute Optimally can be More Effective than Scaling Model Parameters}, 
      author={Charlie Snell and Jaehoon Lee and Kelvin Xu and Aviral Kumar},
      year={2024},
      eprint={2408.03314},
      archivePrefix={arXiv},
      primaryClass={cs.LG},
      url={https://arxiv.org/abs/2408.03314}, 
}

@misc{wu2025inferencescalinglawsempirical,
      title={Inference Scaling Laws: An Empirical Analysis of Compute-Optimal Inference for Problem-Solving with Language Models}, 
      author={Yangzhen Wu and Zhiqing Sun and Shanda Li and Sean Welleck and Yiming Yang},
      year={2025},
      eprint={2408.00724},
      archivePrefix={arXiv},
      primaryClass={cs.AI},
      url={https://arxiv.org/abs/2408.00724}, 
}

@misc{sun2023retentivenetworksuccessortransformer,
      title={Retentive Network: A Successor to Transformer for Large Language Models}, 
      author={Yutao Sun and Li Dong and Shaohan Huang and Shuming Ma and Yuqing Xia and Jilong Xue and Jianyong Wang and Furu Wei},
      year={2023},
      eprint={2307.08621},
      archivePrefix={arXiv},
      primaryClass={cs.CL},
      url={https://arxiv.org/abs/2307.08621}, 
}

@misc{eval-harness,
  author       = {Gao, Leo and Tow, Jonathan and Abbasi, Baber and Biderman, Stella and Black, Sid and DiPofi, Anthony and Foster, Charles and Golding, Laurence and Hsu, Jeffrey and Le Noac'h, Alain and Li, Haonan and McDonell, Kyle and Muennighoff, Niklas and Ociepa, Chris and Phang, Jason and Reynolds, Laria and Schoelkopf, Hailey and Skowron, Aviya and Sutawika, Lintang and Tang, Eric and Thite, Anish and Wang, Ben and Wang, Kevin and Zou, Andy},
  title        = {The Language Model Evaluation Harness},
  month        = 07,
  year         = 2024,
  publisher    = {Zenodo},
  version      = {v0.4.3},
  doi          = {10.5281/zenodo.12608602},
  url          = {https://zenodo.org/records/12608602}
}

@misc{paperno2016lambadadatasetwordprediction,
      title={The LAMBADA dataset: Word prediction requiring a broad discourse context}, 
      author={Denis Paperno and Germán Kruszewski and Angeliki Lazaridou and Quan Ngoc Pham and Raffaella Bernardi and Sandro Pezzelle and Marco Baroni and Gemma Boleda and Raquel Fernández},
      year={2016},
      eprint={1606.06031},
      archivePrefix={arXiv},
      primaryClass={cs.CL},
      url={https://arxiv.org/abs/1606.06031}, 
}

@misc{zellers2019hellaswagmachinereallyfinish,
      title={HellaSwag: Can a Machine Really Finish Your Sentence?}, 
      author={Rowan Zellers and Ari Holtzman and Yonatan Bisk and Ali Farhadi and Yejin Choi},
      year={2019},
      eprint={1905.07830},
      archivePrefix={arXiv},
      primaryClass={cs.CL},
      url={https://arxiv.org/abs/1905.07830}, 
}

@misc{bisk2019piqareasoningphysicalcommonsense,
      title={PIQA: Reasoning about Physical Commonsense in Natural Language}, 
      author={Yonatan Bisk and Rowan Zellers and Ronan Le Bras and Jianfeng Gao and Yejin Choi},
      year={2019},
      eprint={1911.11641},
      archivePrefix={arXiv},
      primaryClass={cs.CL},
      url={https://arxiv.org/abs/1911.11641}, 
}

@misc{clark2018thinksolvedquestionanswering,
      title={Think you have Solved Question Answering? Try ARC, the AI2 Reasoning Challenge}, 
      author={Peter Clark and Isaac Cowhey and Oren Etzioni and Tushar Khot and Ashish Sabharwal and Carissa Schoenick and Oyvind Tafjord},
      year={2018},
      eprint={1803.05457},
      archivePrefix={arXiv},
      primaryClass={cs.AI},
      url={https://arxiv.org/abs/1803.05457}, 
}

@misc{sakaguchi2019winograndeadversarialwinogradschema,
      title={WinoGrande: An Adversarial Winograd Schema Challenge at Scale}, 
      author={Keisuke Sakaguchi and Ronan Le Bras and Chandra Bhagavatula and Yejin Choi},
      year={2019},
      eprint={1907.10641},
      archivePrefix={arXiv},
      primaryClass={cs.CL},
      url={https://arxiv.org/abs/1907.10641}, 
}

@misc{mihaylov2018suitarmorconductelectricity,
      title={Can a Suit of Armor Conduct Electricity? A New Dataset for Open Book Question Answering}, 
      author={Todor Mihaylov and Peter Clark and Tushar Khot and Ashish Sabharwal},
      year={2018},
      eprint={1809.02789},
      archivePrefix={arXiv},
      primaryClass={cs.CL},
      url={https://arxiv.org/abs/1809.02789}, 
}

@inproceedings{Rajpurkar2018SQuAD2,
  title={Know What You Don't Know: Unanswerable Questions for SQuAD},
  author={Pranav Rajpurkar and Jian Zhang and Percy Liang},
  booktitle={ACL 2018},
  year={2018}
}

@misc{joshi2017triviaqalargescaledistantly,
      title={TriviaQA: A Large Scale Distantly Supervised Challenge Dataset for Reading Comprehension}, 
      author={Mandar Joshi and Eunsol Choi and Daniel S. Weld and Luke Zettlemoyer},
      year={2017},
      eprint={1705.03551},
      archivePrefix={arXiv},
      primaryClass={cs.CL},
      url={https://arxiv.org/abs/1705.03551}, 
}

@article{kwiatkowski-etal-2019-natural,
    title = "Natural Questions: A Benchmark for Question Answering Research",
    author = "Kwiatkowski, Tom  and
      Palomaki, Jennimaria  and
      Redfield, Olivia  and
      Collins, Michael  and
      Parikh, Ankur  and
      Alberti, Chris  and
      Epstein, Danielle  and
      Polosukhin, Illia  and
      Devlin, Jacob  and
      Lee, Kenton  and
      Toutanova, Kristina  and
      Jones, Llion  and
      Kelcey, Matthew  and
      Chang, Ming-Wei  and
      Dai, Andrew M.  and
      Uszkoreit, Jakob  and
      Le, Quoc  and
      Petrov, Slav",
    editor = "Lee, Lillian  and
      Johnson, Mark  and
      Roark, Brian  and
      Nenkova, Ani",
    journal = "Transactions of the Association for Computational Linguistics",
    volume = "7",
    year = "2019",
    address = "Cambridge, MA",
    publisher = "MIT Press",
    url = "https://aclanthology.org/Q19-1026/",
    doi = "10.1162/tacl_a_00276",
    pages = "452--466",
}

@misc{dua2019dropreadingcomprehensionbenchmark,
      title={DROP: A Reading Comprehension Benchmark Requiring Discrete Reasoning Over Paragraphs}, 
      author={Dheeru Dua and Yizhong Wang and Pradeep Dasigi and Gabriel Stanovsky and Sameer Singh and Matt Gardner},
      year={2019},
      eprint={1903.00161},
      archivePrefix={arXiv},
      primaryClass={cs.CL},
      url={https://arxiv.org/abs/1903.00161}, 
}

@inproceedings{vaswani2017attention,
  author = {Vaswani, Ashish and Shazeer, Noam and Parmar, Niki and Uszkoreit, Jakob and Jones, Llion and Gomez, Aidan N and Kaiser, {\L}ukasz and Polosukhin, Illia},
  booktitle = {Advances in neural information processing systems},
  pages = {5998--6008},
  title = {Attention is all you need},
  url = {http://arxiv.org/abs/1706.03762},
  year = 2017
}

@misc{nvidia2025nemotronhfamilyaccurateefficient,
      title={Nemotron-H: A Family of Accurate and Efficient Hybrid Mamba-Transformer Models}, 
      author={NVIDIA and : and Aaron Blakeman and Aarti Basant and Abhinav Khattar and Adithya Renduchintala and Akhiad Bercovich and Aleksander Ficek and Alexis Bjorlin and Ali Taghibakhshi and Amala Sanjay Deshmukh and Ameya Sunil Mahabaleshwarkar and Andrew Tao and Anna Shors and Ashwath Aithal and Ashwin Poojary and Ayush Dattagupta and Balaram Buddharaju and Bobby Chen and Boris Ginsburg and Boxin Wang and Brandon Norick and Brian Butterfield and Bryan Catanzaro and Carlo del Mundo and Chengyu Dong and Christine Harvey and Christopher Parisien and Dan Su and Daniel Korzekwa and Danny Yin and Daria Gitman and David Mosallanezhad and Deepak Narayanan and Denys Fridman and Dima Rekesh and Ding Ma and Dmytro Pykhtar and Dong Ahn and Duncan Riach and Dusan Stosic and Eileen Long and Elad Segal and Ellie Evans and Eric Chung and Erick Galinkin and Evelina Bakhturina and Ewa Dobrowolska and Fei Jia and Fuxiao Liu and Gargi Prasad and Gerald Shen and Guilin Liu and Guo Chen and Haifeng Qian and Helen Ngo and Hongbin Liu and Hui Li and Igor Gitman and Ilia Karmanov and Ivan Moshkov and Izik Golan and Jan Kautz and Jane Polak Scowcroft and Jared Casper and Jarno Seppanen and Jason Lu and Jason Sewall and Jiaqi Zeng and Jiaxuan You and Jimmy Zhang and Jing Zhang and Jining Huang and Jinze Xue and Jocelyn Huang and Joey Conway and John Kamalu and Jon Barker and Jonathan Cohen and Joseph Jennings and Jupinder Parmar and Karan Sapra and Kari Briski and Kateryna Chumachenko and Katherine Luna and Keshav Santhanam and Kezhi Kong and Kirthi Sivamani and Krzysztof Pawelec and Kumar Anik and Kunlun Li and Lawrence McAfee and Leon Derczynski and Lindsey Pavao and Luis Vega and Lukas Voegtle and Maciej Bala and Maer Rodrigues de Melo and Makesh Narsimhan Sreedhar and Marcin Chochowski and Markus Kliegl and Marta Stepniewska-Dziubinska and Matthieu Le and Matvei Novikov and Mehrzad Samadi and Michael Andersch and Michael Evans and Miguel Martinez and Mike Chrzanowski and Mike Ranzinger and Mikolaj Blaz and Misha Smelyanskiy and Mohamed Fawzy and Mohammad Shoeybi and Mostofa Patwary and Nayeon Lee and Nima Tajbakhsh and Ning Xu and Oleg Rybakov and Oleksii Kuchaiev and Olivier Delalleau and Osvald Nitski and Parth Chadha and Pasha Shamis and Paulius Micikevicius and Pavlo Molchanov and Peter Dykas and Philipp Fischer and Pierre-Yves Aquilanti and Piotr Bialecki and Prasoon Varshney and Pritam Gundecha and Przemek Tredak and Rabeeh Karimi and Rahul Kandu and Ran El-Yaniv and Raviraj Joshi and Roger Waleffe and Ruoxi Zhang and Sabrina Kavanaugh and Sahil Jain and Samuel Kriman and Sangkug Lym and Sanjeev Satheesh and Saurav Muralidharan and Sean Narenthiran and Selvaraj Anandaraj and Seonmyeong Bak and Sergey Kashirsky and Seungju Han and Shantanu Acharya and Shaona Ghosh and Sharath Turuvekere Sreenivas and Sharon Clay and Shelby Thomas and Shrimai Prabhumoye and Shubham Pachori and Shubham Toshniwal and Shyamala Prayaga and Siddhartha Jain and Sirshak Das and Slawek Kierat and Somshubra Majumdar and Song Han and Soumye Singhal and Sriharsha Niverty and Stefania Alborghetti and Suseella Panguluri and Swetha Bhendigeri and Syeda Nahida Akter and Szymon Migacz and Tal Shiri and Terry Kong and Timo Roman and Tomer Ronen and Trisha Saar and Tugrul Konuk and Tuomas Rintamaki and Tyler Poon and Ushnish De and Vahid Noroozi and Varun Singh and Vijay Korthikanti and Vitaly Kurin and Wasi Uddin Ahmad and Wei Du and Wei Ping and Wenliang Dai and Wonmin Byeon and Xiaowei Ren and Yao Xu and Yejin Choi and Yian Zhang and Ying Lin and Yoshi Suhara and Zhiding Yu and Zhiqi Li and Zhiyu Li and Zhongbo Zhu and Zhuolin Yang and Zijia Chen},
      year={2025},
      eprint={2504.03624},
      archivePrefix={arXiv},
      primaryClass={cs.CL},
      url={https://arxiv.org/abs/2504.03624}, 
}

@misc{qwen3technicalreport,
      title={Qwen3 Technical Report}, 
      author={An Yang and Anfeng Li and Baosong Yang and Beichen Zhang and Binyuan Hui and Bo Zheng and Bowen Yu and Chang Gao and Chengen Huang and Chenxu Lv and Chujie Zheng and Dayiheng Liu and Fan Zhou and Fei Huang and Feng Hu and Hao Ge and Haoran Wei and Huan Lin and Jialong Tang and Jian Yang and Jianhong Tu and Jianwei Zhang and Jianxin Yang and Jiaxi Yang and Jing Zhou and Jingren Zhou and Junyang Lin and Kai Dang and Keqin Bao and Kexin Yang and Le Yu and Lianghao Deng and Mei Li and Mingfeng Xue and Mingze Li and Pei Zhang and Peng Wang and Qin Zhu and Rui Men and Ruize Gao and Shixuan Liu and Shuang Luo and Tianhao Li and Tianyi Tang and Wenbiao Yin and Xingzhang Ren and Xinyu Wang and Xinyu Zhang and Xuancheng Ren and Yang Fan and Yang Su and Yichang Zhang and Yinger Zhang and Yu Wan and Yuqiong Liu and Zekun Wang and Zeyu Cui and Zhenru Zhang and Zhipeng Zhou and Zihan Qiu},
      year={2025},
      eprint={2505.09388},
      archivePrefix={arXiv},
      primaryClass={cs.CL},
      url={https://arxiv.org/abs/2505.09388}, 
}

@misc{sun2025learninglearntesttime,
      title={Learning to (Learn at Test Time): RNNs with Expressive Hidden States}, 
      author={Yu Sun and Xinhao Li and Karan Dalal and Jiarui Xu and Arjun Vikram and Genghan Zhang and Yann Dubois and Xinlei Chen and Xiaolong Wang and Sanmi Koyejo and Tatsunori Hashimoto and Carlos Guestrin},
      year={2025},
      eprint={2407.04620},
      archivePrefix={arXiv},
      primaryClass={cs.LG},
      url={https://arxiv.org/abs/2407.04620}, 
}

@misc{wang2025testtimeregressionunifyingframework,
      title={Test-time regression: a unifying framework for designing sequence models with associative memory}, 
      author={Ke Alexander Wang and Jiaxin Shi and Emily B. Fox},
      year={2025},
      eprint={2501.12352},
      archivePrefix={arXiv},
      primaryClass={cs.LG},
      url={https://arxiv.org/abs/2501.12352}, 
}

@misc{kimiteam2025kimilinearexpressiveefficient,
      title={Kimi Linear: An Expressive, Efficient Attention Architecture}, 
      author={{Kimi Team} and Yu Zhang and Zongyu Lin and Xingcheng Yao and Jiaxi Hu and Fanqing Meng and Chengyin Liu and Xin Men and Songlin Yang and Zhiyuan Li and Wentao Li and Enzhe Lu and Weizhou Liu and Yanru Chen and Weixin Xu and Longhui Yu and Yejie Wang and Yu Fan and Longguang Zhong and Enming Yuan and Dehao Zhang and Yizhi Zhang and T. Y. Liu and Haiming Wang and Shengjun Fang and Weiran He and Shaowei Liu and Yiwei Li and Jianlin Su and Jiezhong Qiu and Bo Pang and Junjie Yan and Zhejun Jiang and Weixiao Huang and Bohong Yin and Jiacheng You and Chu Wei and Zhengtao Wang and Chao Hong and Yutian Chen and Guanduo Chen and Yucheng Wang and Huabin Zheng and Feng Wang and Yibo Liu and Mengnan Dong and Zheng Zhang and Siyuan Pan and Wenhao Wu and Yuhao Wu and Longyu Guan and Jiawen Tao and Guohong Fu and Xinran Xu and Yuzhi Wang and Guokun Lai and Yuxin Wu and Xinyu Zhou and Zhilin Yang and Yulun Du},
      year={2025},
      eprint={2510.26692},
      archivePrefix={arXiv},
      primaryClass={cs.CL},
      url={https://arxiv.org/abs/2510.26692}, 
}

@misc{tencenthunyuanteam2025hunyuanturbosadvancinglargelanguage,
      title={Hunyuan-TurboS: Advancing Large Language Models through Mamba-Transformer Synergy and Adaptive Chain-of-Thought}, 
      author={{Tencent Hunyuan Team} and Ao Liu and Botong Zhou and Can Xu and Chayse Zhou and ChenChen Zhang and Chengcheng Xu and Chenhao Wang and Decheng Wu and Dengpeng Wu and Dian Jiao and Dong Du and Dong Wang and Feng Zhang and Fengzong Lian and Guanghui Xu and Guanwei Zhang and Hai Wang and Haipeng Luo and Han Hu and Huilin Xu and Jiajia Wu and Jianchen Zhu and Jianfeng Yan and Jiaqi Zhu and Jihong Zhang and Jinbao Xue and Jun Xia and Junqiang Zheng and Kai Liu and Kai Zhang and Kai Zheng and Kejiao Li and Keyao Wang and Lan Jiang and Lixin Liu and Lulu Wu and Mengyuan Huang and Peijie Yu and Peiqi Wang and Qian Wang and Qianbiao Xiang and Qibin Liu and Qingfeng Sun and Richard Guo and Ruobing Xie and Saiyong Yang and Shaohua Chen and Shihui Hu and Shuai Li and Shuaipeng Li and Shuang Chen and Suncong Zheng and Tao Yang and Tian Zhang and Tinghao Yu and Weidong Han and Weijie Liu and Weijin Zhou and Weikang Wang and Wesleye Chen and Xiao Feng and Xiaoqin Ren and Xingwu Sun and Xiong Kuang and Xuemeng Huang and Xun Cao and Yanfeng Chen and Yang Du and Zhen Yang and Yangyu Tao and Yaping Deng and Yi Shen and Yigeng Hong and Yiqi Chen and Yiqing Huang and Yuchi Deng and Yue Mao and Yulong Wang and Yuyuan Zeng and Zenan Xu and Zhanhui Kang and Zhe Zhao and ZhenXiang Yan and Zheng Fang and Zhichao Hu and Zhongzhi Chen and Zhuoyu Li and Zongwei Li and Alex Yan and Ande Liang and Baitong Liu and Beiping Pan and Bin Xing and Binghong Wu and Bingxin Qu and Bolin Ni and Boyu Wu and Chen Li and Cheng Jiang and Cheng Zhang and Chengjun Liu and Chengxu Yang and Chengzhong Xu and Chiyu Wang and Chong Zha and Daisy Yi and Di Wang and Fanyang Lu and Fei Chen and Feifei Liu and Feng Zheng and Guanghua Yu and Guiyang Li and Guohua Wang and Haisheng Lin and Han Liu and Han Wang and Hao Fei and Hao Lu and Haoqing Jiang and Haoran Sun and Haotian Zhu and Huangjin Dai and Huankui Chen and Huawen Feng and Huihui Cai and Huxin Peng and Jackson Lv and Jiacheng Shi and Jiahao Bu and Jianbo Li and Jianglu Hu and Jiangtao Guan and Jianing Xu and Jianwei Cai and Jiarong Zhang and Jiawei Song and Jie Jiang and Jie Liu and Jieneng Yang and Jihong Zhang and Jin lv and Jing Zhao and Jinjian Li and Jinxing Liu and Jun Zhao and Juntao Guo and Kai Wang and Kan Wu and Lei Fu and Lei He and Lei Wang and Li Liu and Liang Dong and Liya Zhan and Long Cheng and Long Xu and Mao Zheng and Meng Liu and Mengkang Hu and Nanli Chen and Peirui Chen and Peng He and Pengju Pan and Pengzhi Wei and Qi Yang and Qi Yi and Roberts Wang and Rongpeng Chen and Rui Sun and Rui Yang and Ruibin Chen and Ruixu Zhou and Shaofeng Zhang and Sheng Zhang and Shihao Xu and Shuaishuai Chang and Shulin Liu and SiQi Wang and Songjia Feng and Songling Yuan and Tao Zhang and Tianjiao Lang and Tongkai Li and Wei Deng and Wei Li and Weichao Wang and Weigang Zhang and Weixuan Sun and Wen Ouyang and Wenxiang Jiao and Wenzhi Sun and Wenzhuo Jia and Xiang Zhang and Xiangyu He and Xianshun Ren and XiaoYing Zhu and Xiaolong Guo and Xiaoxue Li and Xiaoyu Ma and Xican Lu and Xinhua Feng and Xinting Huang and Xinyu Guan and Xirui Li and Xu Zhang and Xudong Gao and Xun Luo and Xuxiang Qi and Yangkun Chen and Yangyu Tao and Yanling Xiao and Yantao Mai and Yanze Chen and Yao Ding and Yeting Yang and YiFan Song and Yifan Yang and Yijiao Zhu and Yinhe Wu and Yixian Liu and Yong Yang and Yuanjun Cai and Yuanlin Tu and Yue Zhang and Yufei Huang and Yuhang Zhou and Yuhao Jiang and Yuhong Liu and Yuhui Hu and Yujin Lin and Yun Yang and Yunhao Wang and Yusong Zhang and Zekun Wu and Zelong Zhang and Zhan Yu and Zhaoliang Yang and Zhe Zhao and Zheng Li and Zhenyu Huang and Zhiguang Liu and Zhijiang Xu and Zhiqing Kui and Zhiyin Zeng and Zhiyuan Xiong and Zhuo Han and Zifan Wu and Zigang Geng and Zilong Zhao and Ziyan Tang and Ziyuan Zhu and Zonglei Zhu and Zhijiang Xu},
      year={2025},
      eprint={2505.15431},
      archivePrefix={arXiv},
      primaryClass={cs.CL},
      url={https://arxiv.org/abs/2505.15431}, 
}

@misc{bahdanau2016neuralmachinetranslationjointly,
      title={Neural Machine Translation by Jointly Learning to Align and Translate}, 
      author={Dzmitry Bahdanau and Kyunghyun Cho and Yoshua Bengio},
      year={2014},
      eprint={1409.0473},
      archivePrefix={arXiv},
      primaryClass={cs.CL},
      url={https://arxiv.org/abs/1409.0473}, 
}

@misc{hu2025combaimprovingbilinearrnns,
      title={Comba: Improving Bilinear RNNs with Closed-loop Control}, 
      author={Jiaxi Hu and Yongqi Pan and Jusen Du and Disen Lan and Xiaqiang Tang and Qingsong Wen and Yuxuan Liang and Weigao Sun},
      year={2025},
      eprint={2506.02475},
      archivePrefix={arXiv},
      primaryClass={cs.LG},
      url={https://arxiv.org/abs/2506.02475}, 
}

@misc{yang2024gatedlinearattentiontransformers,
      title={Gated Linear Attention Transformers with Hardware-Efficient Training}, 
      author={Songlin Yang and Bailin Wang and Yikang Shen and Rameswar Panda and Yoon Kim},
      year={2024},
      eprint={2312.06635},
      archivePrefix={arXiv},
      primaryClass={cs.LG},
      url={https://arxiv.org/abs/2312.06635}, 
}

@misc{tandon2025endtoendtesttimetraininglong,
      title={End-to-End Test-Time Training for Long Context}, 
      author={Arnuv Tandon and Karan Dalal and Xinhao Li and Daniel Koceja and Marcel Rød and Sam Buchanan and Xiaolong Wang and Jure Leskovec and Sanmi Koyejo and Tatsunori Hashimoto and Carlos Guestrin and Jed McCaleb and Yejin Choi and Yu Sun},
      year={2025},
      eprint={2512.23675},
      archivePrefix={arXiv},
      primaryClass={cs.LG},
      url={https://arxiv.org/abs/2512.23675}, 
}

@misc{zhang2025testtimetrainingright,
      title={Test-Time Training Done Right}, 
      author={Tianyuan Zhang and Sai Bi and Yicong Hong and Kai Zhang and Fujun Luan and Songlin Yang and Kalyan Sunkavalli and William T. Freeman and Hao Tan},
      year={2025},
      eprint={2505.23884},
      archivePrefix={arXiv},
      primaryClass={cs.LG},
      url={https://arxiv.org/abs/2505.23884}, 
}

@misc{waleffe2024empiricalstudymambabasedlanguage,
      title={An Empirical Study of Mamba-based Language Models}, 
      author={Roger Waleffe and Wonmin Byeon and Duncan Riach and Brandon Norick and Vijay Korthikanti and Tri Dao and Albert Gu and Ali Hatamizadeh and Sudhakar Singh and Deepak Narayanan and Garvit Kulshreshtha and Vartika Singh and Jared Casper and Jan Kautz and Mohammad Shoeybi and Bryan Catanzaro},
      year={2024},
      eprint={2406.07887},
      archivePrefix={arXiv},
      primaryClass={cs.LG},
      url={https://arxiv.org/abs/2406.07887}, 
}

@techreport{nvidia_h100_2022,
  author       = {NVIDIA},
  title        = {NVIDIA H100 Tensor Core GPU White Paper},
  year         = {2022},
  institution  = {NVIDIA},
  url          = {https://resources.nvidia.com/en-us-hopper-architecture/nvidia-h100-tensor-c}
}

@misc{yang2025ropenopeagainnew,
      title={Rope to Nope and Back Again: A New Hybrid Attention Strategy}, 
      author={Bowen Yang and Bharat Venkitesh and Dwarak Talupuru and Hangyu Lin and David Cairuz and Phil Blunsom and Acyr Locatelli},
      year={2025},
      eprint={2501.18795},
      archivePrefix={arXiv},
      primaryClass={cs.CL},
      url={https://arxiv.org/abs/2501.18795}, 
}

@misc{cabannes2025shortwindowattentionenables,
      title={Short window attention enables long-term memorization}, 
      author={Loïc Cabannes and Maximilian Beck and Gergely Szilvasy and Matthijs Douze and Maria Lomeli and Jade Copet and Pierre-Emmanuel Mazaré and Gabriel Synnaeve and Hervé Jégou},
      year={2025},
      eprint={2509.24552},
      archivePrefix={arXiv},
      primaryClass={cs.LG},
      url={https://arxiv.org/abs/2509.24552}, 
}

@online{openai_gpt53_codex_2026,
  author={OpenAI},
  title={Introducing GPT-5.3-Codex},
  year={2026},
  month=feb,
  day={5},
  url={https://openai.com/index/introducing-gpt-5-3-codex/},
  urldate= {2026-02-17}
}

@online{anthropic_claude_opus_46_2026,
  author={Anthropic},
  title={Introducing Claude Opus 4.6},
  year={2026},
  month=feb,
  day={5},
  url={https://www.anthropic.com/news/claude-opus-4-6},
  urldate={2026-02-17}
}

@misc{fu2023hungryhungryhipposlanguage,
      title={Hungry Hungry Hippos: Towards Language Modeling with State Space Models}, 
      author={Daniel Y. Fu and Tri Dao and Khaled K. Saab and Armin W. Thomas and Atri Rudra and Christopher Ré},
      year={2023},
      eprint={2212.14052},
      archivePrefix={arXiv},
      primaryClass={cs.LG},
      url={https://arxiv.org/abs/2212.14052}, 
}

@misc{ma2024megalodonefficientllmpretraining,
      title={Megalodon: Efficient LLM Pretraining and Inference with Unlimited Context Length}, 
      author={Xuezhe Ma and Xiaomeng Yang and Wenhan Xiong and Beidi Chen and Lili Yu and Hao Zhang and Jonathan May and Luke Zettlemoyer and Omer Levy and Chunting Zhou},
      year={2024},
      eprint={2404.08801},
      archivePrefix={arXiv},
      primaryClass={cs.LG},
      url={https://arxiv.org/abs/2404.08801}, 
}

\newpage

\appendix

\section{Exponential-Trapezoidal Discretization}
\begin{restatable}[Variation of Constants~\citep{tenenbaum1985ordinary}]{proposition}{PropVarConst}\label{prop:var-const}

Consider the linear SSM
\begin{equation*}
    \dot{\vh}(t) = A(t)\,\vh(t) + \bB(t)\,x(t),
\end{equation*}
where $\vh(t)\in\mathbb{R}^N$, $A(t)\in\mathbb{R}$ is a scalar decay, and $\bB(t)x(t)\in\mathbb{R}^N$.  
For $\Delta_t$ discretized time grid $\tau_t = \tau_{t-1} + \Delta_t$, the hidden state satisfies \eqref{eq:exact-step}, which can then be approximated to \eqref{eq:approx-step} with $O(\Delta_t^2)$ error.
The approximation of the remaining integral on the state-input can have varying error bounds depending on the method used: an example can be found in \Cref{app:trap-error-proof}.

\begin{align}
\label{eq:exact-step}
\vh(\tau_t)
&= \exp\!\left(\int_{\tau_{t-1}}^{\tau_t} A(s)\,ds\right)\vh(\tau_{t-1})
 + \int_{\tau_{t-1}}^{\tau_t}
   \exp\!\left(\int_{\tau}^{\tau_t} A(s)\,ds\right)\mB(\tau)x(\tau)\, d\tau, \\
\label{eq:approx-step}
\vh_t
&\approx e^{\Delta_t A_t}\,\vh_{t-1}
 + \int_{\tau_{t-1}}^{\tau_t} e^{(\tau_t-\tau)A_t}\,\bB(\tau)x(\tau)\,d\tau .
\end{align}

\end{restatable} \begin{proof}[Proof.]
Starting from the initial linear SSM, an integrating factor $z(t) \coloneqq e^{\int^t_0 -A(s)ds}$ is applied to facilitate integration.
\[
    z(t) \dot{\vh}(t) = z(t)A(t)\vh(t) +z(t)\mB(t)x(t)
\]
Considering $z'(t) = -A(t)z(t)$; through rearranging the terms and integrating between the time grid $[\tau_{t-1},\tau_t]$
\[
    \int_{\tau_{t-1}}^{\tau_t} \frac{d}{d\tau}\left( z(\tau) \vh(\tau) \right) d\tau = \int_{\tau_{t-1}}^{\tau_t} z(\tau) \mB(\tau)x(\tau) d\tau
\]
results in
\[
    z(\tau_t)\vh(\tau_t) - z(\tau_{t-1})\vh(\tau_{t-1}) = \int_{\tau_{t-1}}^{\tau_t} z(\tau)\mB(\tau)x(\tau) d\tau,
\]
which can be arranged in a more familiar form
\[
    \vh(\tau_t) = z(\tau_t)^{-1}z(\tau_{t-1})\vh(\tau_{t-1}) + \int_{\tau_{t-1}}^{\tau_t} z(\tau_t)^{-1}z(\tau)\mB(\tau)x(\tau) d\tau.
\]
Substituting the integrating factor $z(\tau)$ corresponds to
\[
    \vh(\tau_t) = \exp\left(\int_{\tau_{t-1}}^{\tau_t}A(s)ds\right)\vh(\tau_{t-1}) + \int_{\tau_{t-1}}^{\tau_t} \exp\left(\int_{\tau}^{\tau_t}A(s)ds\right)\mB(\tau)x(\tau) d\tau.
\]
We approximate the state-transition integral with a right-hand assumption where $\forall s \in [\tau_{t-1},\tau_t], A(s) \coloneqq A(\tau_t)$ which we refer to as $A_t$,
\[
    \vh_t \approx \underbrace{\exp\left(\Delta_t A_t\right)\vh_{t-1}}_\text{right-hand approximation} + \underbrace{\int_{\tau_{t-1}}^{\tau_t} \exp\left((\tau_t -\tau)A_t\right)\mB(\tau)x(\tau) d\tau}_\text{to be approximated}.
\]
incurring a local truncation error of order $O(\Delta_t^2)$.
Thus, we have approximated the exponential dynamics of the adjusted underlying ODE and leave the state-input integral to be approximated with any host of methods.
\end{proof}

\subsection{Exponential-Trapezoidal Discretization's Mask Matrix}
\label{app:proof-tensor}

\begin{proof}
    When viewing the tensor contraction form, let us call $C = (T,N), B=(S,N), L=(T,S), X=(S,P)$ based on the Mamba-2 paper. With this decomposition of our mask, we can view $L = \text{contract}(TZ,ZS\rightarrow TS)(L_1,L_2)$.
    
    The original contraction can be seen as
    $$\text{contract}(TN,SN,TS,SP\rightarrow TP)(C,B,L,X)$$
    We can now view it as
    $$\text{contract}(TN,SN,TJ,JS,SP\rightarrow TP)(C,B,L_1,L_2,X)$$
    This can be broken into the following:
    \begin{align*}
        Z &= \text{contract}(SN,SP\rightarrow SNP)(B,X)\\
        Z' &= \text{contract}(JS,SNP\rightarrow JNP)(L_2,Z)\\
        H &= \text{contract}(TJ,JNP\rightarrow TNP)(L_1,Z') \\
        Y &= \text{contract}(TN,TNP\rightarrow TP)(C,H)
    \end{align*}
    We can view this step: $\text{contract}(ZS,SNP\rightarrow ZNP)(L_2,Z)$ as a convolution of size two applied on the state-input ($B,X$ outer product) prior to the decay with the traditional SSD $L=L_1$ matrix.
\end{proof}
    
\subsection{Exponential-Trapezoidal Discretization Error Rate}
\label{app:trap-error-proof}
\paragraph{Standard assumptions.}
We assume that: $A(t),\bB(t),x(t)$ are bounded and $C^3$ on each timestep, so that $g(\tau)$ has three bounded derivatives; the map $\vh \mapsto A(t)\vh + \bB(t)x(t)$ is Lipschitz in $\vh$ which is true for linear systems; $\lambda_t$ lies in a bounded interval so that the update is zero-stable.

\begin{proof}
Let $\vg(\tau) \coloneqq e^{(t_k-\tau) A_k}\,\bB(\tau)x(\tau)$ denote the 
integrand in the second term of Proposition~\ref{prop:var-const}.  
Since $A(t),\bB(t),x(t)$ are $C^3$ on $[t_{k-1},t_k]$, the function $g$ has three
bounded derivatives.  
A second-order Taylor expansion of $g$ around $t_{k-1}$ gives us,
\begin{align*}
\int_{t_{k-1}}^{t_k} g(\tau)\,d\tau
=
\Delta_t\, g(t_{k-1})
+ \frac{\Delta_t^2}{2}\, g'(t_{k-1})
+ \frac{\Delta_t^3}{6}\, g''(t_{k-1})
+ O(\Delta_t^4).
\end{align*}

Recall that the trapezoidal approximation to this integral is given by,
\begin{align*}
    Q_\lambda
    =
    \Delta_t \Big[(1-\lambda_t)\, g(t_{k-1})
    + \lambda_t\, g(t_k)\Big].
\end{align*}

Expanding $g(t_k)$ using Taylor expansion:
$
g(t_k)
= g(t_{k-1}) + \Delta_t g'(t_{k-1})
+ \frac{\Delta_t^2}{2} g''(t_{k-1}) + O(\Delta_t^3)
$.
Substituting this into $Q_\lambda$,
\begin{align*}
Q_\lambda
&= \Delta_t\Big[(1-\lambda_t) g(t_{k-1})
    + \lambda_t g(t_k)\Big] \\
&= \Delta_t g(t_{k-1})
  + \lambda_t \Delta_t^2 g'(t_{k-1})
  + \lambda_t \frac{\Delta_t^3}{2} g''(t_{k-1})
  + O(\Delta_t^4).
\end{align*}
Hence, the error is given by:
\begin{align*}
\int_{t_{k-1}}^{t_k} g(\tau)\,d\tau - Q_\lambda
&= \Big(\tfrac{1}{2} - \lambda_t\Big)\Delta_t^2 g'(t_{k-1})
 + \Big(\tfrac{1}{6} - \tfrac{\lambda_t}{2}\Big)\Delta_t^3 g''(t_{k-1})
 + O(\Delta_t^4).
\end{align*}
Under the assumption that $\lambda_t = \tfrac{1}{2} + c_t \Delta_t$, where $c_t = O(1)$, then
$\tfrac{1}{2} - \lambda_t = -c_t \Delta_t = O(\Delta_t)$ and thus the
$\Delta_t^2$ term is $O(\Delta_t^3)$. Therefore,
\begin{align*}
    \int_{t_{k-1}}^{t_k} g(\tau)\,d\tau - Q_\lambda = O(\Delta_t^3),
\end{align*}
which yields an $O(\Delta_t^3)$ local truncation error. 
\end{proof}

\subsection{Exponential-Trapezoidal Parameterization}
\label{app:trap-param}
\begin{table}[h]
    \centering
    \small
    \caption{\textbf{Ablations on $\lambda_t$ parameterization in the exponential-trapezoidal update.}} 
    \begin{tabular}{l l c}
        \toprule
        \textbf{Parameterization} & \textbf{Form of $\lambda_t$} & \textbf{ppl $\downarrow$} \\
        \midrule
        \textbf{Default} & $\sigma(u_t)$ & \textbf{15.72} \\[5pt]
        Fixed $1/2$          & $\tfrac{1}{2}$ & 15.76 \\[5pt]
        No trapezoidal (Euler)      & $1$ & 15.81 \\
        \bottomrule
    \end{tabular}
    \label{tab:lambda_param_abl}
\end{table}

\textbf{Setting:} All runs use the Mamba-3 (SISO) 440M model trained at Chinchilla scale, with the other architectural and optimization hyperparameters being the same as in Table \ref{tab:downstream_evaluation}.

The default model uses a data-dependent gate 
$\lambda_t = \sigma(u_t)$, where $u_t$ is a learned projection of the current input token. 
In \cref{tab:lambda_param_abl}, we try different parameterizations for
$\lambda_t$ and find that the default parameterization empirically performs the best.
Hence, we choose the simpler default 
parameterization that does \emph{not} enforce \(\lambda_t = \tfrac{1}{2}+ O(\Delta_t)\).

\section{Complex SSM Proofs}
\subsection{Proof of Proposition \ref{prop:complex-to-real-ssm}}
\label{app:proof-complex-to-real-ssm}
\PropComplexRealSSM*
\begin{proof}[Proof.]
We first present the derivation for $N=2$; the block-diagonal structure for general even $N$ follows by grouping pairs of coordinates.

Let $h_t + i\hat{h}_t$ denote the complexified hidden state, with parameters 
$A(t)+i\theta(t)$ and $B(t)+i\hat{B}(t)$ for the transition and input, respectively.  
By the variation of constants formula (Proposition~\ref{prop:var-const}), applying zero-order hold and Euler's rule over a step $[t_{k-1},t_k]$ gives
\[
h_k + i\hat{h}_k 
= e^{\Delta_t (A_t + i\theta_t)}(h_{k-1}+ i\hat{h}_{k-1})
+ \Delta_t (B_t + i\hat{B}_t)x_t.
\]

Expanding the exponential,
\[
e^{\Delta_t(A_t+i\theta_t)} 
= e^{\Delta_t A_t}\Big(\cos(\Delta_t\theta_t)+i\sin(\Delta_t\theta_t)\Big),
\]
so in real coordinates $\vh_t=\begin{bmatrix}h_t \\ \hat{h}_t\end{bmatrix}\in\mathbb{R}^2$ the recurrence becomes
\[
\vh_t
= e^{\Delta_t A_t}
\underbrace{\begin{bmatrix}
    \cos(\Delta_t\theta_t) & -\sin(\Delta_t\theta_t) \\
    \sin(\Delta_t\theta_t) & \cos(\Delta_t\theta_t)
\end{bmatrix}}_{R(\Delta_t\theta_t)}
\vh_{t-1}
+ \Delta_t
\begin{bmatrix}
B_t \\ \hat{B}_t
\end{bmatrix}x_t.
\]

Stacking across $N/2$ such pairs yields the block-diagonal transition
\[
\vh_t
= e^{\Delta_t A_t}\,\mathrm{Block}\big(\{R(\Delta_t\theta_t[i])\}_{i=1}^{N/2}\big)\vh_{t-1}
+ \Delta_t\begin{bmatrix}\bB_t \\ \hat{\bB}_t\end{bmatrix}x_t.
\]

For the output, 
\[
y_t = \mathrm{Re}\Big((\bC_t+i\hat{\bC}_t)^\top(h_t+i\hat{h}_t)\Big)
= \begin{bmatrix}\bC_t \\ -\hat{\bC}_t\end{bmatrix}^\top \vh_t,
\]
which defines the real projection $\bC_t\in\mathbb{R}^N$ in the proposition.
This proves the equivalence between complex SSM and the real block-diagonal system with rotations.
\end{proof}

\subsection{Proof of Proposition \ref{prop:rope-trick}}
\label{app:proof-rope-trick}
\PropRoPETrick*
\begin{proof}[Proof.]
Consider the SSM
\begin{equation}
    \vh_t \;=\; e^{\Delta_t A_t}\,\bR_t\,\vh_{t-1} \;+\; \Delta_t\bB_t x_t, 
    \qquad 
    y_t \;=\; \bC_t^\top \vh_t,
\end{equation}
where (as in Proposition~\ref{prop:rope-trick}) \(A_t\in\mathbb{R}\) is a scalar (so that \(e^{\Delta_t A_t}\) is a scalar and commutes with rotations), and \(\bR_t\) is block-diagonal orthogonal/unitary, hence \(\bR_t^{-1}=\bR_t^\top\) and the matrices $\bR_i, \bR_j$ commute, i.e. $\bR_i\bR_j=\bR_j\bR_i$.

Unrolling the recurrence with the convention that an empty product is the identity,
\begin{equation}
    \vh_t \;=\; \sum_{i=0}^{t} 
    \bigg(\prod_{s=i+1}^{t} e^{\Delta_s A_s}\bR_s\bigg)\,\Delta_i\bB_i x_i.
\end{equation}
Thus
\begin{align}
    y_t 
    \;=\; \bC_t^\top \vh_t
    \;=\; \sum_{i=0}^{t} \bC_t^\top 
    \bigg(\prod_{s=i+1}^{t} e^{\Delta_s A_s}\bR_s\bigg)\Delta_i\bB_i x_i.
\end{align}
Using its unitary property,
\begin{align*}
    \prod_{s=i+1}^{t}\bR_s 
    = \left(\prod_{s=0}^{t}\bR_s\right)\left(\prod_{s=0}^{i}\bR_s\right)^{-1}
    = \left(\prod_{s=0}^{t}\bR_s\right)\left(\prod_{s=0}^{i}\bR_s^\top\right).
\end{align*}
Since \(e^{\Delta_s A_s}\) are scalars, they commute with rotations; hence
\begin{align}
    y_t 
    &= \sum_{i=0}^{t} 
       \bC_t^\top \left(\prod_{s=0}^{t}\bR_s\right)
       \left(\prod_{s=i+1}^{t} e^{\Delta_s A_s}\right)
       \left(\prod_{s=0}^{i}\bR_s^\top\right) \Delta_i \bB_i x_i \\
    &= \left[\left(\prod_{s=0}^{t}\bR_s^\top\right)\bC_t\right]^\top 
       \sum_{i=0}^{t} 
       \left(\prod_{s=i+1}^{t} e^{\Delta_s A_s}\right)
       \left(\prod_{s=0}^{i}\bR_s^\top\right)\Delta_i \bB_i x_i.
\end{align}
Define the rotated parameters
\(
\overline{\bC}_t := \left(\prod_{s=0}^{t}\bR_s^\top\right)\bC_t
\)
and
\(
\overline{\bB}_i := \left(\prod_{s=0}^{i}\bR_s^\top\right)\bB_i.
\)
Then,
\begin{equation}
    y_t \;=\; \overline{\bC}_t^\top 
    \sum_{i=0}^{t} \left(\prod_{s=i+1}^{t} e^{\Delta_s A_s}\right)\Delta_i \overline{\bB}_i x_i.
\end{equation}
Equivalently, introducing the rotated state \(\tilde{\vh}_t := \left(\prod_{s=0}^{t}\bR_s^\top\right)\vh_t\),
\begin{equation}
    \tilde{\vh}_t \;=\; e^{\Delta_t A_t}\,\tilde{\vh}_{t-1} \;+\; \Delta_t \overline{\bB}_t x_t,
    \qquad
    y_t \;=\; \overline{\bC}_t^\top \tilde{\vh}_t,
\end{equation}
\end{proof}

\subsection{Proof of Proposition \ref{prop:rope-trick-trap}}
\label{app:proof-rope-trick-trap}
\PropRoPETrickTrap*
\begin{proof}[Proof.]
We begin from the complex SSM (as in Prop.~\ref{prop:complex-to-real-ssm})
\begin{align*}
    \dot{\vh}(t)
    &= \mathrm{Diag}\,\left(A(t)+i\boldsymbol{\theta}(t)\right)\,\vh(t)
      \;+\;\left(\bB(t)+i\hat{\bB}(t)\right)\,x(t),\\
    y(t)
    &= \mathrm{Re}\,\left((\bC(t)+i\hat{\bC}(t))^\top \vh(t)\right),
\end{align*}

where $A(t)\in\mathbb{R}$ is a scalar and $\boldsymbol{\theta}(t),\bB(t),\hat{\bB}(t),\bC(t),\hat{\bC}(t)\in\mathbb{R}^{N/2}$.

\medskip
Recall from Prop.~\ref{prop:var-const},
\begin{align*}
\vh_t
&\approx e^{\Delta_t(A_t+i\boldsymbol{\theta}_t)}\,\vh_{t-1}
\;+\; \int_{\tau_{t-1}}^{\tau_t}
      e^{(\tau_t-\tau)(A_t+i\boldsymbol{\theta}_t)}
      \left(\bB(\tau)+i\hat{\bB}(\tau)\right)\,x(\tau)\,d\tau.
\end{align*}
Applying Prop.~\ref{prop:trap} to the above integral, we get 
\begin{align}
\label{eq:trap-complex}
\vh_t
&= e^{\Delta_t(A_t+i\boldsymbol{\theta}_t)}\,\vh_{t-1}
  \;+\; \beta_t\,e^{i\Delta_t\boldsymbol{\theta}_t}\left(\bB_{t-1}+i\hat{\bB}_{t-1}\right)x_{t-1}
  \;+\; \gamma_t\left(\bB_t+i\hat{\bB}_t\right)x_t,
\end{align}
where
\[
\alpha_t:=e^{\Delta_t A_t},\qquad
\beta_t:=(1-\lambda_t)\Delta_t e^{\Delta_t A_t},\qquad
\gamma_t:=\lambda_t\Delta_t,
\]

\medskip
Since $e^{\Delta_t(A_t+i\boldsymbol{\theta}_t)}=\alpha_t\,e^{i\Delta_t\boldsymbol{\theta}_t}$ and as shown in Prop.~\ref{prop:complex-to-real-ssm}, multiplication by $e^{i\Delta_t\boldsymbol{\theta}_t}$ is a block-diagonal rotation in real coordinates, we get the real $N$-dimensional recurrence
\begin{align}
\label{eq:trap-real-rot}
\vh_t
&= \alpha_t\,\bR_t\,\vh_{t-1}
  \;+\; \beta_t\,\bR_t\,\bB_{t-1}\,x_{t-1}
  \;+\; \gamma_t\,\bB_t\,x_t,\\[0.25em]
y_t
&= \bC_t^\top \vh_t,\nonumber
\end{align}
where $\bR_t\coloneqq\mathrm{Block}\left(\{R(\Delta_t\boldsymbol{\theta}_t[i])\}_{i=1}^{N/2}\right)$ where
\(
R(\theta)\coloneqq\begin{bmatrix}\cos\theta&-\sin\theta\\ \sin\theta&\cos\theta\end{bmatrix},
\)
and projections \\
\(
\bB_t\coloneqq\begin{bmatrix}\bB_t\\ \hat{\bB}_t\end{bmatrix},\;
\bC_t\coloneqq\begin{bmatrix}\bC_t\\ -\hat{\bC}_t\end{bmatrix}.
\)
Note that $\bR_t$ is orthogonal, so $\bR_t^{-1}=\bR_t^\top$, and that $\bR_i, \bR_j$ commute, i.e., $\bR_i\bR_j=\bR_j\bR_i$.

\medskip
We define the following.
\begin{align*}
\tilde{\vh}_t
&:= \left(\prod_{s=0}^{t}\bR_s^\top\right)\vh_t,\qquad
\overline{\bB}_{t} := \left(\prod_{s=0}^{t}\bR_s^\top\right)\bB_{t},\qquad
\overline{\bC}_{t} := \left(\prod_{s=0}^{t}\bR_s^\top\right)\bC_{t}.
\end{align*}
Left-multiplying \eqref{eq:trap-real-rot} by $\prod_{s=0}^{t}\bR_s^\top$ and using $\bR_t^\top\bR_t=I$,
\begin{align*}
\tilde{\vh}_t
&= \alpha_t\,\tilde{\vh}_{t-1}
  \;+\; \beta_t\,\overline{\bB}_{t-1}\,x_{t-1}
  \;+\; \gamma_t\,\overline{\bB}_{t}\,x_t,\\
y_t
&= \overline{\bC}_t^\top\,\tilde{\vh}_t.
\end{align*}
This is a vanilla scalar-transition SSM with data-dependent rotary embeddings absorbed into $\bB,\bC$ via cumulative products of $\bR_s^\top$.
\end{proof}

\section{MIMO for Mamba-3}
\label{app:mimo-for-mamba}

\paragraph{Mamba with MIMO.}
With a given batch, head, and sequence position $t$, consider the input $\bU_t \in \mathbb{R}^{D}$. Also denote $P, R\in\mathbb{N}$ as the head dimension and MIMO rank, respectively. We first obtain SSM parameters via a set of projections defined in terms of tensor contraction notation as follows:

\begin{align*}
    \bB_t &= \text{contract}(DNR,D\rightarrow NR)(\bW_B,\bU_t)
    && \bC_t = \text{contract}(DNR,D\rightarrow NR)(\bW_C,\bU_t),\\
    \bX'_t &= \text{contract}(PD,D\rightarrow P)(\bW_{X'},\bU_t)
    && \bX_t = \text{contract}(PR,P\rightarrow PR)(\bW_X,\bX'_t),\\
\end{align*}
where $\bW_B, \bW_C, \bW_{X'}, \bW_X$ are model parameters. Additionally, we obtain the residual gate term $\bZ_t$ in the same manner as $\bX_t$ with weights $\bW_{Z'}$ and $\bW_Z$.
This parameterization is used to prevent the parameter count from increasing by a factor of $R$.

The state update and the SSM output are then computed via the following MIMO SSM:
\[
\bH_t \;=\; a_t\,\bH_{t-1} \;+\; \bB_t \bX_t^\top \in \mathbb{R}^{N\times P},
\qquad
\bY_t \;=\; \bH_t ^\top \bC_{t} \in \mathbb{R}^{P\times R} .
\]
Intermediate output $\bY'_t$ is obtained by the residual function $\phi$, $\bY'_t \gets \phi(\bY_t, \bZ_t)$, where $\phi(\bY_t, \bZ_t):=\bY_t \odot \text{SiLU}(\bZ_t)$ in our case. 
Finally, the layer output $\bO_t \in \mathbb{R}^D$ is computed via the following down projections:
\begin{align*}
    \bO'_t &= \text{contract}(PR,PR\rightarrow P)(\bW_{O'},\bY'_t)
    && \bO_t = \text{contract}(PD,P\rightarrow D)(\bW_{O},\bO'_t).
\end{align*}

This formulation enhances the existing Mamba-3 architecture by providing a lightweight parameterization that transforms the set of independent SISO SSMs within each head into a set of MIMO SSMs.

\paragraph{MIMO Parameter Matching.}
The MIMO variant of Mamba-3 incurs additional parameters compared to its SISO counterpart. We therefore reduce the hidden dimension of the MLP layers to parameter match the SISO variants as follows:

\begin{table}[h]
    \centering
    \begin{tabular}{lcccc}
        \toprule
        \textbf{Model}& $180\text{M}$ & $440\text{M}$ & $880\text{M}$ & $1.5\text{B}$ \\
        \midrule
        SISO MLP dim & 1,500 & 2,048 & 3,072 & 4,096 \\
        MIMO MLP dim ($R=4$) & 1,264 & 1,792 & 2,800 & 3,824 \\
        \bottomrule
    \end{tabular}
\end{table}

\section{Experimental Details}
\label{app:exp-details}

\textbf{Language Modeling.}
Our pretraining procedures follow those of \citet{dao2024transformersssmsgeneralizedmodels}'s section D.2. All models at each scale follow the same procedure and were trained with bfloat16. The Mamba family of models were trained using the standard expand factor of $2$ and a $d_\text{state}$ of 128 and head dimension of 64. The Transformer baselines follow~\citet{dao2024transformersssmsgeneralizedmodels}, and the GDN baselines follow~\citep{yang2025gateddeltanetworksimproving} where $q,k_\text{dim}=128,v_\text{dim}=256$. We utilize the Llama-3.1 tokenizer~\citep{grattafiori2024llama3herdmodels} for all models.

We utilize LM Evaluation Harness~\citep{eval-harness} to test the zero-shot language modeling capabilities of our pretrained model on LAMBADA (OpenAI version)~\citep{paperno2016lambadadatasetwordprediction}, HellaSwag~\citep{zellers2019hellaswagmachinereallyfinish}, PIQA~\citep{bisk2019piqareasoningphysicalcommonsense}, Arc-Easy/Arc-Challenge~\citep{clark2018thinksolvedquestionanswering}, WinoGrande~\citep{sakaguchi2019winograndeadversarialwinogradschema}, and OpenBookQA~\citep{mihaylov2018suitarmorconductelectricity}.

\textbf{Real-World and Synthetic Retrieval.}
For our real-world retrieval tasks, we evaluate on the common suite consisting of SWDE~\citep{arora2025simplelinearattentionlanguage}, SQuAD~\citep{Rajpurkar2018SQuAD2}, FDA~\citep{arora2025simplelinearattentionlanguage}, TriviaQA~\citep{joshi2017triviaqalargescaledistantly}, NQ~\citep{kwiatkowski-etal-2019-natural}, and DROP~\citep{dua2019dropreadingcomprehensionbenchmark}. We utilize the cloze-formatted version of the aforementioned tasks provided by ~\citet{arora2025simplelinearattentionlanguage,arora2024justreadtwiceclosing}, as the original datasets are in a question-answering format, making it challenging for solely pretrained models. All tasks were truncated to match the training context length. The synthetic NIAH tasks~\citep{hsieh2024rulerwhatsrealcontext} were also run with LM Evaluation Harness. 

\textbf{State-Tracking Synthetics.}
Training follows a sequence length curriculum that sets the minimum length to $3$ and progresses the maximum length from $40$ to $160$.
Final models are evaluated at $256$ length.
Each curriculum runs for $10^{4}$ steps with batch size $256$. 
We use one-layer models for Parity and three-layer models for Modular-arithmetic tasks. 
The state size is chosen to be $64$,
and we sweep $d_{\text{model}}\in\{32,64\}$ and $8$ learning rates logarithmically spaced between $10^{-4}$ and $10^{-2}$, reporting the best validation accuracy.

\section{Additional Experimental Results}
\label{app:additional-results}

\begin{figure}[H]
    \centering
    \includegraphics[width=0.75\linewidth]{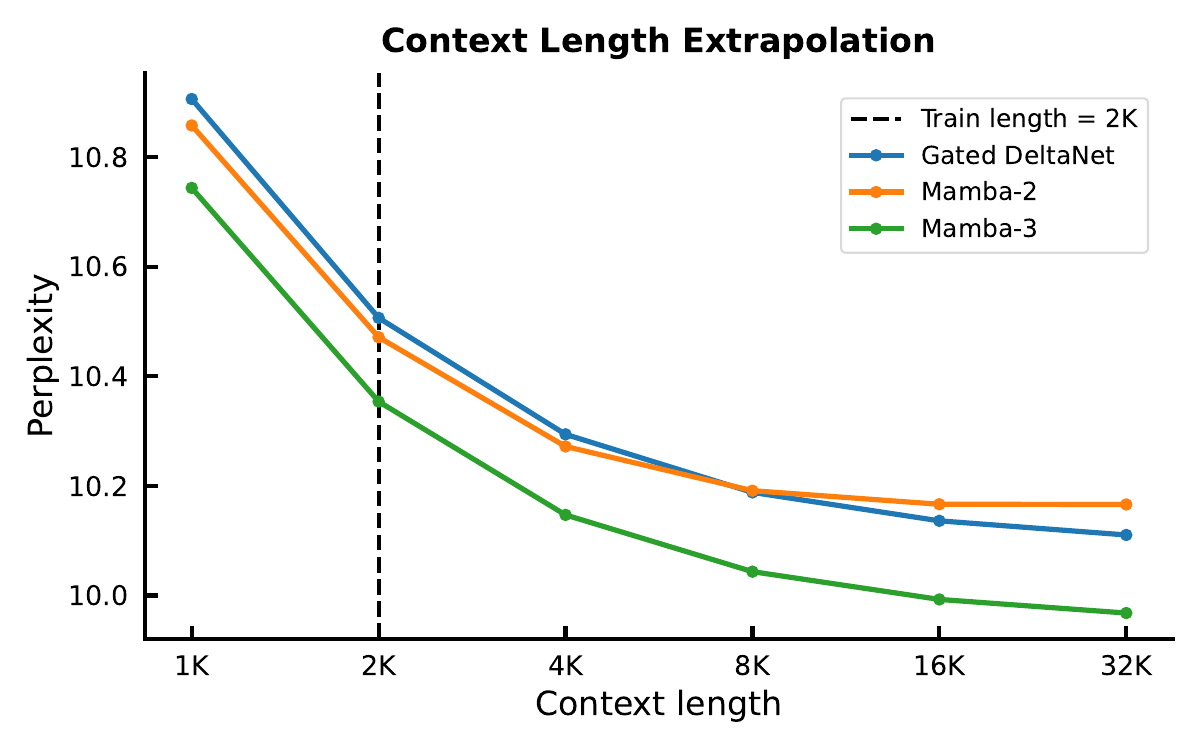}
    \caption{Pretrained 1.5B models' performance on the held-out FineWeb-Edu test set at varying context lengths. Mamba-3 exhibits strong length extrapolation while Mamba-2 falters at longer contexts.}
    \label{fig:len_gen}
\end{figure}

\begin{table}[!t]
    \small
    \centering
    \caption{
    Ablations of optional norm type (grouped vs default) and placement (pre- vs post-gate) on pretrained hybrid Mamba-3 SISO models at the 1.5B scale. All models have BCNorm. No additional norm demonstrates the strongest in-context retrieval performance on average, while pre-gate, grouped RMS results in the best performance on synthetic retrieval, especially on lengths longer than its training context.
    }
    \resizebox{\textwidth}{!}{
    \sisetup{detect-weight,     %
        table-align-text-post=false,
         table-space-text-post={*} %
         }
    \begin{tabular}{l*{7}{S[table-format=2.1]}*{9}{S[table-format=2.1]}}
        \toprule
        \textbf{Mamba-3 Norm Type} & {LM Avg.} & {SWDE} & {SQD.} & {FDA} & {TQA} & {NQ} & {Drop} & \multicolumn{3}{c}{NIAH-Single-1} & \multicolumn{3}{c}{NIAH-Single-2} & \multicolumn{3}{c}{NIAH-Single-3} \\
        \cmidrule(lr){2-2} \cmidrule(lr){3-8} \cmidrule(lr){9-11} \cmidrule(lr){12-14} \cmidrule(lr){15-17}
        Context Length & {---} & \multicolumn{6}{c}{2048} & {1024} & {2048} & {4096} & {1024} & {2048} & {4096} & {1024} & {2048} & {4096} \\
        \midrule
        No Norm & \Uline{56.4} & \Uline{58.5} & \Uline{47.0} & \Uline{65.9} & 64.8 & \Bf 33.4 & 27.0 & \Bf 100.0 & \Bf 100.0 & 36.2 & \Bf 100.0 & \Bf 100.0 & 9.4 & \Bf 99.8 & \Bf 100.0 & 8.8 \\
        Post-Gate Default RMS & \Bf 56.5 & 54.5 & 46.6 & 61.9 & \Uline{65.4} & 31.9 & \Bf 29.2 & \Bf 100.0 & \Bf 100.0 & \Bf 100.0 & \Bf 100.0 & 99.8 & 49.2 & 87.6 & 94.0 & \Uline{62.0} \\
        Pre-Gate Default RMS & 55.9 & 55.4 & 46.9 & \Bf 67.3 & \Uline{65.4} & 33.0 & 28.1 & \Bf 100.0 & \Bf 100.0 & 86.2 & \Bf 100.0 & \Bf 100.0 & \Bf 97.8 & 99.2 & \Uline{97.8} & \Bf 90.2 \\
        Post-Gate Grouped RMS & 56.2 & 51.4 & 46.7 & 56.8 & 64.2 & 30.4 & 27.6 & \Bf 100.0 & \Bf 100.0 & 79.4 & \Bf 100.0 & \Bf 100.0 & 65.8 & 93.8 & 97.0 & 9.6 \\
        Pre-Gate Grouped RMS & 56.1 & \Bf 58.6 & \Bf 47.3 & 52.4 & \Bf 65.7 & \Uline{33.3} & \Uline{28.5} & \Bf 100.0 & \Bf 100.0 & \Bf 100.0 & \Bf 100.0 & \Bf 100.0 & \Uline{96.0} & \Bf 99.8 & 97.2 & 56.8 \\
        \bottomrule
    \end{tabular}
    }
    \label{tab:hybrid_norm_abl}
\end{table}

\begin{figure}[H]
    \centering
    \includegraphics[width=0.75\linewidth]{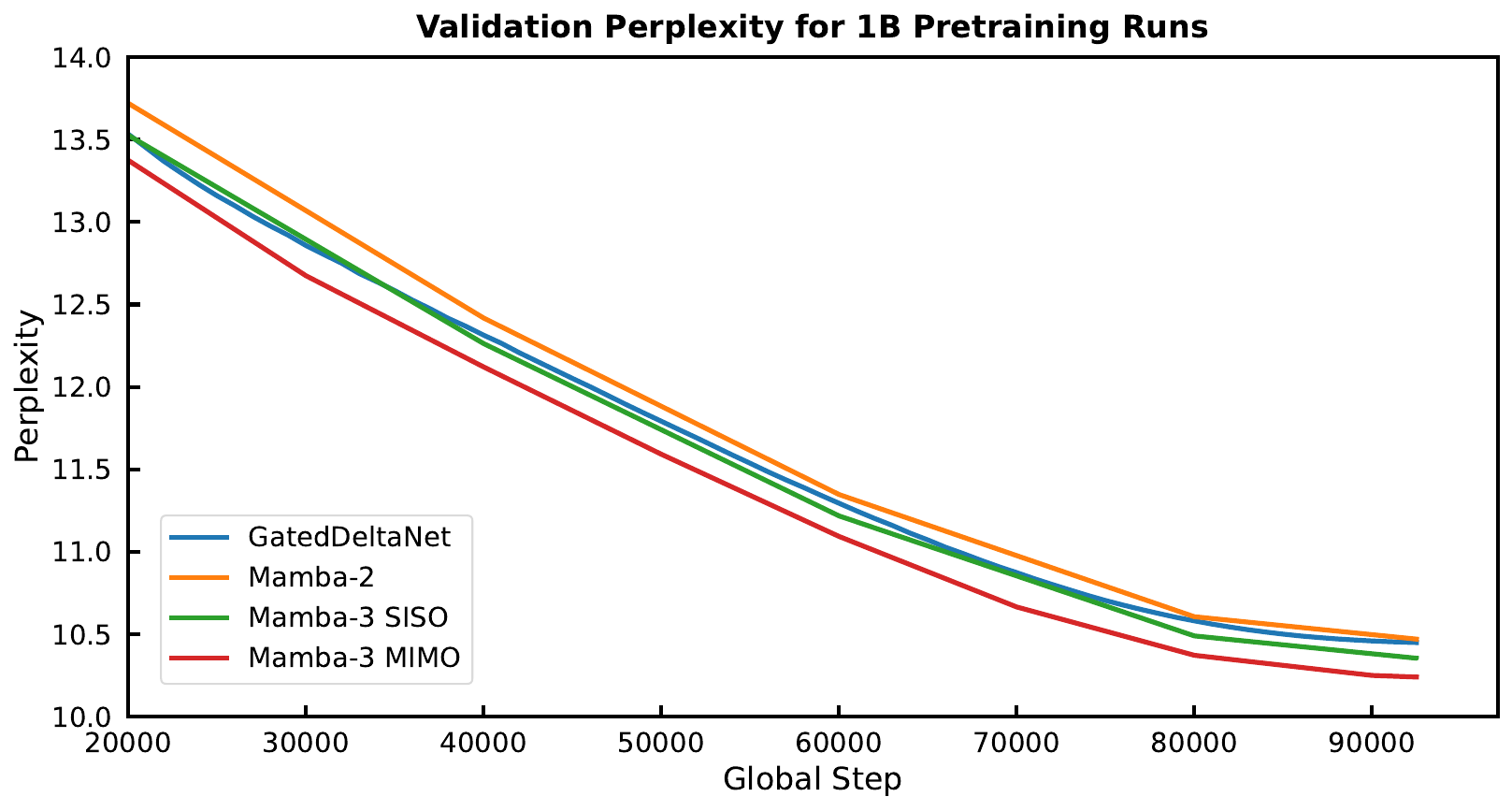}
    \caption{Mamba-3 demonstrates better pretraining performance compared to strong baselines like Mamba-2 and Gated DeltaNet. These are the validation perplexity on FineWeb-Edu of our fully pretrained 1.5B models.}
    \label{fig:pretrain_ppl}
\end{figure}

We also compare the effectiveness of state size usage of Mamba variants to a Gated DeltaNet baseline in \cref{fig:pareto_gdn}. We highlight the difficulty of directly comparing GDN versus Mamba-style models due to the differing head structure (multi-head for GDN compared to multi-value for Mamba). Our experiments hold GDN's $v_{expand}$ to 2 and decrease the head dimension accordingly to vary the relative total state size. Similar to \Cref{fig:pareto}, we train 440M models to $2\times$ Chinchilla tokens ($40\times$ token-to-parameter ratio) and sweep across $d_\text{state}=\{32, 64, 128\}$ for the Mamba models and $d_\text{head dim}=\{32, 64, 128\}$ for GDN. We parameter match all models.
\begin{figure}[H]
    \centering
    \includegraphics[width=0.75\linewidth]{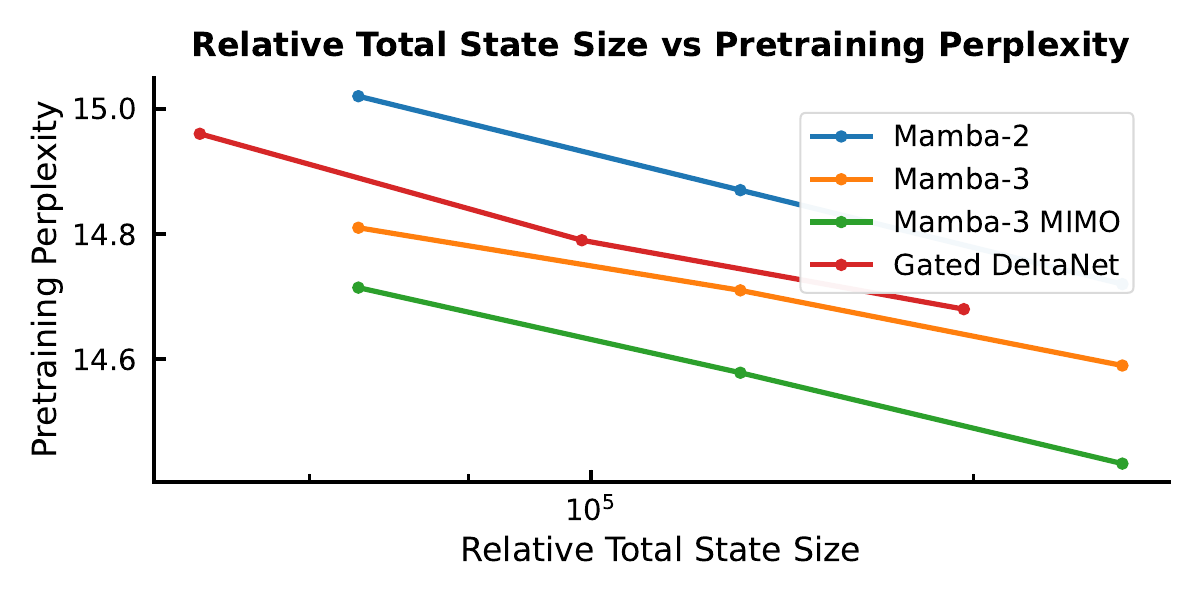}
    \caption{Exploration of state size (inference speed proxy) versus pretraining perplexity (performance proxy). Mamba-3 and Mamba-3 MIMO continue to set the Pareto frontier.}
    \label{fig:pareto_gdn}
\end{figure}

\section{Architecture Ablations}
\label{app:arch-ablations}
We explore our model architecture ablations in this section. All models are trained at the 440M scale to Chinchilla optimal number of tokens ($20\times$ tokens to parameters) with the same experimental procedures as our pretrained models as covered in~\cref{app:exp-details} unless otherwise stated.

\paragraph{$\bB,\bC$ Bias Parameterization.}
The Mamba-3 model's separate $B$ and $C$ biases are head-specific and channel-wise and added to both $\bB$ and $\bC$ after the QK-Norm. While the biases in the final Mamba-3 model are trainable, data-independent parameters and initialized to all ones, we explore various bias parameterizations in~\cref{tab:bc_init_abl}. We find our models are not very sensitive to the initialization of the biases as long as they are positive. We choose the all-ones initialization due to its simplicity.

We also explore the impact of removing the $B$ or $C$ bias on performance in~\cref{tab:bc_bias_presence} (bias is initialized with our default parameterization when utilized). Unlike in ~\citet{yu2025blockbiasedmamba}, which finds that $B$ bias by itself is able to improve performance on Mamba-1, our experiments find that only having $B$ bias hurts performance slightly and that $B$ and $C$ biases have synergistic properties.

\begin{table}[h]
    \centering
    \small
    \makebox[\textwidth][c]{%
        \begin{subtable}[t]{0.45\textwidth}
            \centering
            \begin{tabular}{l c c}
                \toprule
                \textbf{Bias Init.} & \textbf{Trainable} & \textbf{ppl $\downarrow$} \\
                \midrule
                1.0                 & $\checkmark$ & 15.72 \\
                0.0                 & $\checkmark$ & 16.57 \\
                1.0                 & $\times$     & 15.80 \\
                $\mathcal{U}(0,1)$  & $\checkmark$ & 15.76 \\
                $\mathcal{U}(-1,1)$ & $\checkmark$ & 16.07 \\
                \bottomrule
            \end{tabular}
            \caption{Effect of parameterization of the $B$ and $C$ bias on model performance, measured by pretraining perplexity. We find our default initialization of all-ones (first row) provides the best performance, but performance is not sensitive as long as biases are positive.}
            \label{tab:bc_init_abl}
        \end{subtable}%
        \hspace{0.03\textwidth}%
        \begin{subtable}[t]{0.45\textwidth}
            \centering
            \begin{tabular}{c c c}
                \toprule
                $\bB$ \textbf{Bias} & $\bC$ \textbf{Bias} & \textbf{ppl $\downarrow$} \\
                \midrule
                $\times$      & $\times$      & 16.52 \\
                $\checkmark$  & $\times$      & 16.68 \\
                $\times$      & $\checkmark$  & 15.98 \\
                $\checkmark$  & $\checkmark$  & 15.69 \\
                \bottomrule
            \end{tabular}
            \caption{Applying a bias to both $B$ and $C$ leads to the best performance. Only applying $B$ bias (Block-Biased~\citep{yu2025blockbiasedmamba} Mamba-3 variant) does not provide significant gains over the no-bias baseline.}
            \label{tab:bc_bias_presence}
        \end{subtable}%
    }
    \caption{Ablations on $B,C$ bias initialization (left) and presence (right) for Mamba-3.}
    \label{tab:bc_combined}
\end{table}

\section{Inference Kernel Latency Analysis}
\label{app:inference-details}

\subsection{Kernel Implementations and Fusion Structure}
In Table~\ref{tab:benchmark}, we detail the DSL (Triton, TileLang, CuTe, PyTorch) and the fusion level of the kernels used in our latency analysis.
For Mamba-2 and Gated DeltaNet (GDN), we directly use the publicly released Triton kernels from the respective authors. 
For Mamba-3, we implement new inference kernels with a comparable fusion structure: the forward SISO uses a Triton kernel fused with rotary position embeddings and the forward MIMO uses a TileLang kernel with the same fusion level while the decode path uses a CuTe kernel fused with gating and MIMO projection. 

In Tables~\ref{tab:kernel-fusion-forward} and~\ref{tab:kernel-fusion-decode}, 
we abbreviate IP = input projection, Conv = 1D convolution, Gate = gating, 
OP = output projection. Colors indicate implementation backend
(\torch{Torch}, \triton{Triton}, \tilelang{TileLang}, \cute{CuTe}).

\begin{table}[h!]
\centering
\caption{Kernel DSL and fusion structure for \textbf{forward} (prefill) kernels.}
\label{tab:kernel-fusion-forward}
\begin{tabular}{l l p{7.5cm}}
\toprule
\textbf{Model (Forward)} & \textbf{Kernel DSL} & \textbf{Fusion Level} \\
\midrule
Mamba-2        
& Triton 
& \torch{IP}, \triton{Conv}, \triton{SSM}, \triton{Gate}, \torch{OP} \\[3pt]

Gated DeltaNet 
& Triton 
& \torch{IP}, \triton{Conv}, \triton{Chunked Delta}, \triton{Gate}, \torch{OP} \\[3pt]

Mamba-3 (SISO) 
& Triton 
& \torch{IP}, \triton{SSM+Rotary+Gate}, \torch{OP} \\[3pt]

Mamba-3 (MIMO) 
& TileLang 
& \torch{IP}, \tilelang{SSM+Rotary+Gate}, \torch{OP} \\
\bottomrule
\end{tabular}
\end{table}

\begin{table}[h!]
\centering
\caption{Kernel DSL and fusion structure for \textbf{decode} kernels.}
\label{tab:kernel-fusion-decode}
\begin{tabular}{l l p{7.5cm}}
\toprule
\textbf{Model (Decode)} & \textbf{Kernel DSL} & \textbf{Fusion Level} \\
\midrule
Mamba-2        
& Triton        
& \torch{IP}, \triton{Conv}, \triton{SSM}, \triton{Gate}, \torch{OP} \\[3pt]

Gated DeltaNet 
& Triton        
& \torch{IP}, \triton{Conv}, \triton{Recurrent Delta}, \triton{Gate}, \torch{OP} \\[3pt]

Mamba-3 (SISO) 
& CuTe + Triton 
& \torch{IP}, \triton{Rotary}, \cute{SSM+Gate}, \torch{OP} \\[3pt]

Mamba-3 (MIMO) 
& CuTe + Triton 
& \torch{IP}, \triton{Rotary}, \cute{SSM+Gate}, \torch{OP} \\
\bottomrule
\end{tabular}
\end{table}

\subsection{Extended Prefill and Prefill+Decode Latency Measurements}

\paragraph{Models.}
We benchmark Mamba-3 1.5B (SISO), Mamba-2 1.5B, Gated DeltaNet 1.5B, and a strong Transformer baseline implemented via the vLLM engine (v0.11.0) with Llama-3.2~1B.\footnote{\url{https://huggingface.co/meta-llama/Llama-3.2-1B}.}  
All recurrent models are trained at the 1.5B scale with $d_{\text{model}} = 2048$ and $24$ layers. For Mamba variants we set state size as $128$ and head dimension $64$; for GDN we use QK head dimension as $128$.  

\paragraph{Setting.}

Sequence lengths were swept over $L \in \{512, 1024, 2048, 4096, 16384\}$ for prefill, with an equal number of tokens decoded.
For all sequence lengths, we use a batch size of 128. To report vLLM numbers at sequence length $16384$, we measure performance at the same sequence length with batch size 16.
We then scale the result by a factor of 8 to approximate performance at batch size 128 since direct measurement at this setting exceeds GPU memory.
This provides a reasonable estimate because each batch is processed independently by each SM on the GPU, so we expect performance of Transformer models to scale linearly with batch size.
For recurrent models, when the size of input and output tensors exceeds GPU memory at sequence length $16384$, we utilize a state passing approach that processes the sequence in two halves while propagating the hidden state between segments to avoid materializing the entire sequence at once.
We use a single H100-SXM 80GB GPU and report wall-clock times (in seconds) over three repetitions.

We observe that (i) Mamba-3 adds minimal forward-pass cost, showing that the exponential-trapezoidal update, complex state tracking, and MIMO parameterization remain lightweight; (ii) decode latency is competitive across recurrent models; and (iii) recurrent mixers scale more gently with context length than vLLM Llama-3.2-1B, which grows much faster with $L$ due to KV-cache overhead.

\end{document}